\documentclass[]{fairmeta}

\usepackage{latexsym}

\usepackage[T1]{fontenc}

\usepackage[utf8]{inputenc}

\usepackage{microtype}

\usepackage{inconsolata}

\usepackage{graphicx}

\usepackage{wrapfig} 

\usepackage{amsmath,amsfonts,bm}

\def\eqref#1{equation~\ref{#1}}

\def\1{\bm{1}}

\DeclareMathAlphabet{\mathsfit}{\encodingdefault}{\sfdefault}{m}{sl}
\SetMathAlphabet{\mathsfit}{bold}{\encodingdefault}{\sfdefault}{bx}{n}

\usepackage[utf8]{inputenc} 
\usepackage[T1]{fontenc}    
\usepackage{hyperref}       
\usepackage{url}            
\usepackage{booktabs}       
\usepackage{amsfonts}       
\usepackage{nicefrac}       
\usepackage{microtype}      
\usepackage{xcolor}         

\usepackage{amsmath}
\usepackage{amssymb}
\usepackage{mathtools}
\usepackage{amsthm}
\usepackage{algorithm}
\usepackage{setspace} 
\usepackage{enumitem}
\usepackage{algorithmic}

\usepackage{multirow}
\usepackage{indentfirst}
\usepackage{subcaption}
\usepackage{xcolor}
\usepackage{thmtools}
\usepackage{multirow}

\usepackage{siunitx}
\usepackage{booktabs}
\usepackage[textsize=tiny]{todonotes}

\usepackage{hyperref}
\hypersetup{
    colorlinks,
    linkcolor={red!50!black},
    citecolor={blue!50!black},
    urlcolor={blue!80!black}
}

\usepackage[most]{tcolorbox}

\usepackage{xcolor}
\newcommand{\blue}[1]{\textcolor{blue}{#1}}
\newcommand{\red}[1]{\textcolor{red}{#1}}

\newtheorem{theorem}{Theorem}

\numberwithin{mytheorem}{section} 
\newtheorem{proposition}{Proposition}
\newtheorem{lemma}{Lemma}

\newtheorem{remark}{Remark}

\newtheorem*{theorem*}{Theorem}
\newtheorem*{lemma*}{Lemma}
\newtheorem*{proposition*}{Proposition}

\usepackage[xindy, acronyms, nowarn]{glossaries}   %
\glsdisablehyper
\makeglossaries

\newglossaryentry{LLM}
{
  name={LLM},
  plural={LLMs},
  description={large language model},
  first={large language model (\glsentrytext{LLM})},
  descriptionplural={large language models},
  firstplural={large language models (\glsentryplural{LLM})}
}

\newglossaryentry{FM}
{
  name={FM},
  plural={FMs},
  description={foundation model},
  first={foundation model (\glsentrytext{FM})},
  descriptionplural={foundation models},
  firstplural={foundation models (\glsentryplural{FM})}
}

\newglossaryentry{AI}
{
  name={AI},
  plural={AIs},
  description={artificial intelligence},
  first={artificial intelligence (\glsentrytext{AI})},
  descriptionplural={artificial intelligence systems},
  firstplural={artificial intelligence systems (\glsentryplural{AI})}
}

\newglossaryentry{RLHF}
{
  name={RLHF},
  description={Reinforcement Learning from Human Feedback},
  first={Reinforcement Learning from Human Feedback (\glsentrytext{RLHF})}
}

\newglossaryentry{RLAIF}
{
  name={RLAIF},
  description={Reinforcement Learning from AI Feedback (RLAIF)},
  first={Reinforcement Learning from AI Feedback (\glsentrytext{RLAIF})}
}

\newglossaryentry{RLHS}
{
  name={RLHS},
  description={Reinforcement Learning from Hindsight Simulation (RLHS)},
  first={Reinforcement Learning from Hindsight Simulation (\glsentrytext{RLHS})}
}

\newglossaryentry{DPO}
{
  name={DPO},
  description={direct preference optimization (DPO)},
  first={direct preference optimization (\glsentrytext{DPO})}
}

\newglossaryentry{PPO}
{
  name={PPO},
  description={proximal policy optimization (PPO)},
  first={proximal policy optimization (\glsentrytext{PPO})}
}

\newglossaryentry{MDP}
{
  name={MDP},
  description={Markov decision process (MDP)},
  first={Markov decision process (\glsentrytext{MDP})}
}

\newglossaryentry{POMDP}
{
  name={POMDP},
  description={partially observable Markov decision process (POMDP)},
  first={partially observable Markov decision process (\glsentrytext{POMDP})}
}

\newglossaryentry{RL}
{
  name={RL},
  description={reinforcement learning (RL)},
  first={reinforcement learning (\glsentrytext{RL})}
}

\usepackage{ifthen}
\newboolean{include-notes}
\newboolean{include-new}
\newboolean{include-remove}
\setboolean{include-notes}{true}
\setboolean{include-new}{true}
\setboolean{include-remove}{false}

\usepackage{xcolor}
\usepackage[normalem]{ulem}
\newcommand{\haimin}[1]{\ifthenelse{\boolean{include-notes}}{\textcolor{teal}{\textbf{Haimin:} #1}}{}}
\newcommand{\remove}[1]{\ifthenelse{\boolean{include-remove}}{\textcolor{red}{\sout{#1}}}{}}

\definecolor{grey}{HTML}{8C8C8C}
\definecolor{porange}{HTML}{E77500}
\definecolor{purple}{HTML}{9437FF}
\definecolor{trired}{HTML}{D04236}
\definecolor{magenta}{HTML}{FF40FF}
\definecolor{appleblue}{HTML}{3478F6}

\title{\resizebox{\textwidth}{!}{%
  Learning Personalized Agents from Human Feedback%
}}

\author[1,2,*]{Kaiqu Liang}
\author[1]{Julia Kruk}
\author[1]{Shengyi Qian}
\author[1]{Xianjun Yang}
\author[1]{Shengjie Bi}
\author[1]{Yuanshun Yao}
\author[1]{Shaoliang Nie}
\author[1]{Mingyang Zhang}
\author[1]{Lijuan Liu}
\author[2]{Jaime Fernández Fisac}
\author[1,3,*]{Shuyan Zhou}
\author[1]{Saghar Hosseini}

\affiliation[1]{Meta Superintelligence Labs}
\affiliation[2]{Princeton University}
\affiliation[3]{Duke University}

\contribution[*]{Work done at Meta}

\abstract{Modern AI agents are powerful but often fail to align with the idiosyncratic, evolving preferences of individual users. Prior approaches typically rely on static datasets, either training implicit preference models on interaction history or encoding user profiles in external memory. However, these approaches struggle with new users and with preferences that change over time. We introduce Personalized Agents from Human Feedback (PAHF), a framework for continual personalization in which agents learn online from live interaction using explicit per-user memory. PAHF operationalizes a three-step loop: (1) seeking pre-action clarification to resolve ambiguity, (2) grounding actions in preferences retrieved from memory, and (3) integrating post-action feedback to update memory when preferences drift.
To evaluate this capability, we develop a four-phase protocol and two benchmarks in embodied manipulation and online shopping. These benchmarks quantify an agent's ability to learn initial preferences from scratch and subsequently adapt to persona shifts. Our theoretical analysis and empirical results show that integrating explicit memory with dual feedback channels is critical: PAHF learns substantially faster and consistently outperforms both no-memory and single-channel baselines, reducing initial personalization error and enabling rapid adaptation to preference shifts.}

\date{\today}
\correspondence{Kaiqu Liang at \email{kl2471@princeton.edu}}

\metadata[Project page]{\url{personalized-ai.github.io}}

\metadata[Code]{\url{https://github.com/facebookresearch/PAHF}}

\begin{document}

\maketitle

\section{Introduction}

Modern AI agents, powered by large language models (LLMs), are increasingly capable of complex reasoning and task execution. This success has fueled their application in user-facing roles, from digital assistants \citep{zhou2023webarena,xie2024osworld,xu2024theagentcompany} to embodied robots \citep{ahn2022can,huang2022inner,liang2024introspective}. However, it remains a significant challenge for these agents to take actions that align with the complex, idiosyncratic, and often dynamic preferences of individual humans.

Prior works on building personalized assistants have generally relied on pre-existing, static data—either training models to infer implicit preferences from historical interaction logs~\citep{qiu2025bayesian,chen2024pad,tan2024democratizing,li2024personalized} or populating agent memory with pre-defined user profiles and past interaction history~\citep{chhikara2025mem0,xu2025mem}.
However, this reliance on static data poses a fundamental limitation, as the agent cannot properly adapt to real-world human behavior. This failure manifests in three core challenges that most existing systems are not designed to handle. (1) The agent cannot immediately adapt to new users, as it has no profile or interaction history to learn from. (2) Static models are not designed to learn from real-time corrective feedback, when the agent makes an error, it cannot easily update its behavior to avoid that same mistake in the future. (3) The framework fails to handle non-stationary user preferences, as it cannot adapt when a user's persona or preferences evolve over time.

To address these challenges, we introduce the Personalized Agents from Human Feedback (PAHF) framework. PAHF moves away from the static, pre-existing data assumption and treats interaction itself as the primary learning signal.
We simulate long-horizon sequential decision-making, where each user corresponds to a sequence of tasks and each new interaction depends on the agent’s accumulated preference memory.
Within this setting, PAHF uses explicit per-user memory updated from pre- and post-action feedback to continually refine the agent’s behavior. This design is intended to be robust to the realities of human interaction: it enables the agent to build a user’s preference profile from scratch, learn directly from its mistakes, and continually adapt that profile as preferences drift over time.

Our framework operationalizes this continual learning via a novel, three-step interactive loop. The first step is Pre-Action Interaction. When given an ambiguous instruction, the agent first queries its memory. If no relevant preference is found (e.g., for a new user), it proactively asks the user for clarification (e.g., ``Which drink do you prefer?''). This step is designed to resolve ``known uncertainty'' and allows the agent to learn new preferences before making a costly, erroneous action. The second step is Action Execution, where the agent grounds its decision in its current belief state, synthesizing the user's instruction and the preferences retrieved from its memory to execute an action. The final and most critical step is Post-Action Feedback Integration. If the agent's action is incorrect—either from a wrong guess or, more importantly, from applying an outdated preference—the user provides corrective feedback (e.g., ``Actually, I like Sprite most now''). PAHF is explicitly designed to integrate this feedback, using it to revise and update the explicit preferences stored in its memory. This reactive loop is the essential mechanism for correcting miscalibrated beliefs and adapting to non-stationarity, which pre-action queries alone cannot solve.

Our primary contributions are threefold. 
(1) We introduce PAHF, a continual agent personalization framework that couples explicit per-user memory with proactive pre-action clarification to resolve ambiguity and reactive post-action feedback to handle preference drift.
(2) We develop an evaluation suite for continual personalization, comprising two new large-scale benchmarks (embodied manipulation and online shopping) spanning physical and digital agents, and a four-phase protocol that separates initial preference learning from adaptation under persona shifts. 
(3) We provide theoretical and empirical analysis showing that PAHF consistently reduces personalization error over no-memory and single-channel baselines, and that explicit memory combined with post-action feedback is particularly important for robust personalization without pre-existing user data.

\section{Related Work}

\begin{figure*}[!t]
  \centering
   \includegraphics[width = \linewidth]{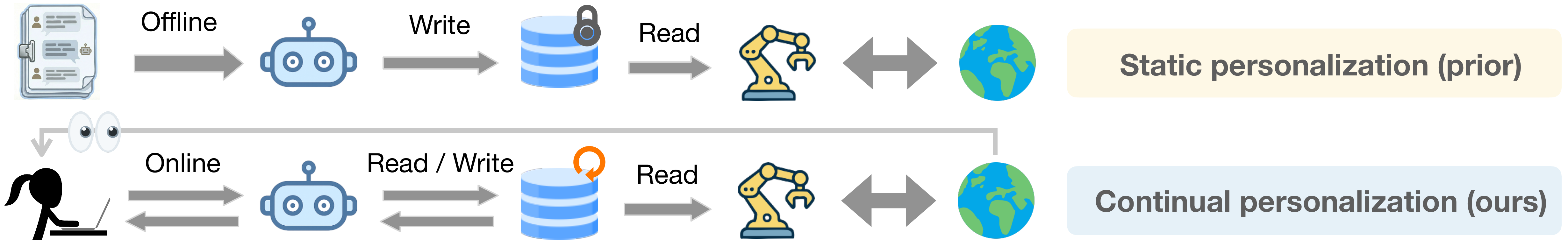}
   \caption{\textbf{Static vs. continual personalization.} Top: Static personalization, where offline human–AI logs are processed once to populate user profiles, and at deployment the agent only reads from this static memory. Bottom: Our continual personalization framework keeps the user in the loop during online interaction: the agent consults and updates explicit memory in the pre-action stage, and the user observes the outcome in the environment after the agent acts and provides post-action feedback, which is written back to memory to correct errors and adapt as preferences change.
   \label{fig:teaser}}
\end{figure*}

\begin{figure*}[!t]
  \centering
   \includegraphics[width = \linewidth]{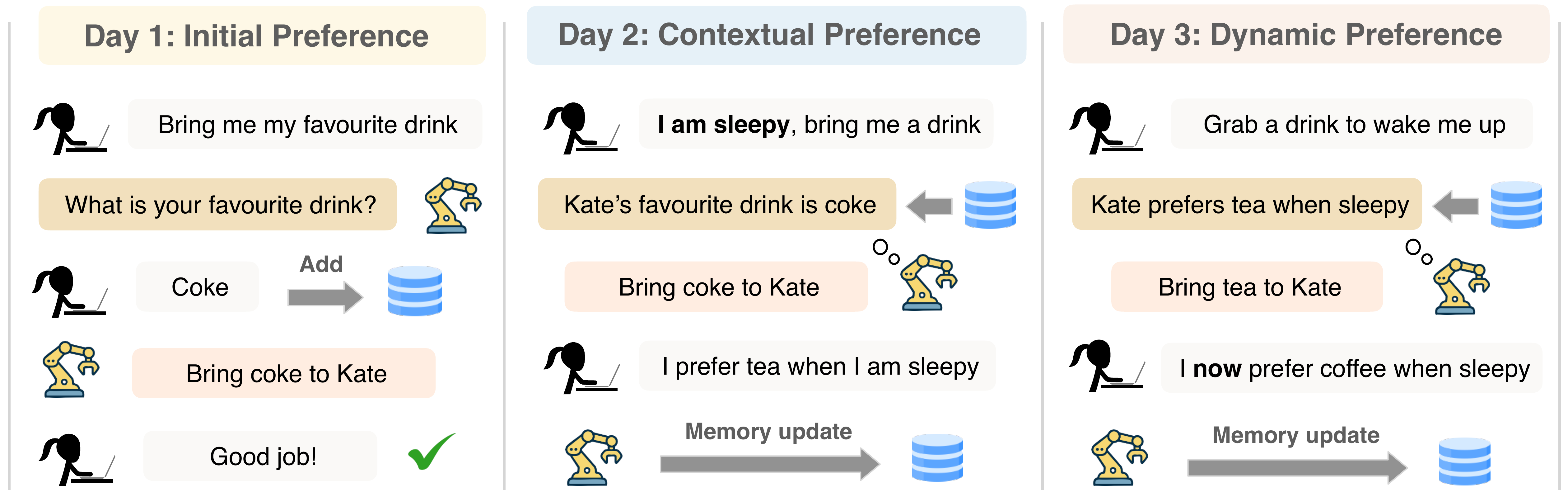}
   \caption{\textbf{Continual agent personalization using pre- and post-action feedback channels.} On Day 1 (Initial Preference), the agent uses a pre-action interaction to learn Kate’s favorite drink (``Coke'') before acting. On Day 2 (Contextual Preference), the agent overgeneralizes this memory when Kate is sleepy; after serving Coke, it is corrected via post-action feedback (``I prefer tea when I am sleepy''). On Day 3 (Dynamic Preference), the agent acts consistently with its memory (serving tea), but Kate’s underlying preference has shifted to coffee. This error triggers another post-action update, illustrating how the framework supports continual personalization under preference drift.
   \label{fig:teaser}}
\end{figure*}

\textbf{Personalized Alignment of LLMs.}
Early approaches to aligning large language models with human preferences rely on reinforcement learning from human feedback (RLHF)~\citep{christiano2017deep,ziegler2019fine,ouyang2022training}. Subsequent work has shown that much of RLHF’s benefit can be recovered via supervised fine-tuning on preference data without explicit RL~\citep{rafailov2024direct,zhao2023slic,yuan2023rrhf,hong2024orpo,ethayarajh2024kto}. Recently, work has highlighted emergent untruthful behaviors induced by RLHF~\citep{liang2025machine} and proposed hindsight simulation as a mitigation~\citep{liang2025rlhs}. This hindsight perspective directly motivates our post-action feedback channel, where the human provides feedback after the agent acts, and the agent uses it to correct and update its memory.
While these methods align to aggregate preferences, recent work explores \emph{personalized} alignment, adapting models to individual users via multi-dimensional objectives or personalized RLHF~\citep{chen2024pad,tan2024democratizing,li2024personalized}. However, these approaches typically rely on substantial pre-collected user data and one-off fine-tuning, making it hard to handle users with no history data and to track evolving preferences.

\textbf{Personalized LLM Agents and Memory.}
A major line of work personalizes LLM outputs through retrieval-augmented generation (RAG)~\citep{lewis2020retrieval,borgeaud2022improving,gao2023retrieval,yu2024chain,wang2023learning}, where user histories are retrieved and injected into context at query time~\citep{salemi2024optimization,salemi2024lamp,richardson2023integrating}. Building on this, recent methods add “thinking” procedures that reason about preferences or select operations such as reasoning, personalization, and clarification during interaction~\citep{Li_2024,salemi2025pathwaysthoughtsmultidirectionalthinking}, or define static persona profiles that enable role-playing but encode coarse, fixed traits rather than fine-grained, evolving preferences~\citep{shao2023character,wang2024rolellm,hu2024quantifying}. These techniques are also used in life-long personalization of LLMs, which updates high-level persona profiles for conversational assistants from accumulated dialogue history~\citep{wang2024ai}. In contrast, we study continual personalization for environment-acting agents in long-horizon sequential tasks.

Recent work has proposed increasingly sophisticated memory architectures for LLM agents~\citep{mei2024aios,chhikara2025mem0,zhong2024memorybank,packer2023memgpt,wang2023enhancing,xu2025mem}, emphasizing long-term state tracking, summarization, and retrieval. In contrast, our memory design is deliberately simple: we use a lightweight agentic memory to isolate the effect of \textit{feedback channels} on continual personalization, rather than introducing new architectural complexity. These advanced memory systems are complementary to our framework and could be integrated into it to further improve scalability and capacity.

\textbf{Personalized Agents in Interactive Environments.}
Personalization for embodied agents has been extensively studied in spatial arrangement~\citep{kapelyukh2022my,wu2023tidybot}, table settings~\citep{puig2020watch}, and object navigation~\citep{dai2024think,barsellotti2024personalized}, as well as inferring user preferences for planning~\citep{xu2025learning}. In flight recommendation~\citep{lin2022inferring}, recent work leverages Bayesian cognitive models~\citep{griffiths2007topics,griffiths2024bayesian,tenenbaum2006theory,tenenbaum2011grow} to teach LLM agents to approximate probabilistic reasoning for personalized recommendations~\citep{qiu2025bayesian}. However, these methods largely rely on implicit preference modeling and offline fine-tuning, which requires extensive historical user data and can be costly to update when user preferences change. Although recent agent memory works~\citep{chhikara2025mem0} introduce explicit preference modeling, they typically assume pre-populated user profiles and do not address the challenge of learning online from live interactions for continual personalization.
Finally, while PREFDISCO~\citep{li2025personalizedreasoningjustintimepersonalization} benchmarks interactive preference discovery, it remains limited to static personas within short-horizon dialogues.

\section{Personalized Agents from Human Feedback}
\label{sec:pahf_framework}

The core objective of the Personalized Agents from Human Feedback (PAHF) framework is to enable an agent to continually align with a user's unique and evolving preferences. Unlike static personalization approaches that learn from fixed datasets, PAHF learns \emph{online from interaction}, continuously updating its internal model of the user. This framework is explicitly designed to operationalize a continual learning loop in which the agent can \texttt{add} new preferences (e.g., preferring energy-efficient appliances), \texttt{revise} outdated ones (e.g., a changed favorite drink), and \texttt{exploit} this knowledge to improve its actions over time.

\subsection{Formalizing Continual Personalization}
\label{sec:formalizing}
We formalize continual personalization as an online learning process. At each interaction $t$, the user has a latent preference state $M_t^*$, and issues an instruction $I_t$ while the agent observes $O_t$ (e.g., a scene or product catalogue) conditional on $M_t^*$. The agent maintains an explicit preference memory $\hat M_t$ (a persistent estimate of the user's preferences) and chooses an action $a_t$ in response to $(I_t,O_t,\hat M_t)$. Let $a_t^*$ denote an optimal action under the true preferences $M_t^*$.
The agent's objective is to learn a personalized policy $\pi(a_t \mid I_t, O_t, \hat M_t)$ that minimizes the cumulative \emph{personalization error} $\sum_{t=1}^T L_t$, where
$
L_t \;=\; \mathbf{1}[a_t \neq a_t^*]
$
is the instantaneous $0$--$1$ loss at round $t$ (details in Appendix \ref{sec:formal_proof}). This problem is challenging because the agent's preference memory $\hat M_t$ is prone to two distinct types of error:
\begin{itemize}[topsep=0.2em, parsep=0em, itemsep=5pt, leftmargin=1em]
    \item \textbf{Partial Observability:} The true state $M_t^*$ is hidden. The agent's memory $\hat M_t$ may be incomplete (e.g., $\hat M_t = \emptyset$ for a new task), leading to errors from \emph{known uncertainty}.

    \item \textbf{Non-Stationarity:} The user's preferences are non-stationary, meaning $M_t^*$ can evolve into $M_{t+1}^*$ at any time (i.e., ``preference drift''). This leads to errors from \emph{miscalibration}, where the agent is \emph{confidently wrong} (e.g., it believes $\hat M_t = \{\text{likes Coke}\}$ when $M_t^* = \{\text{likes Sprite}\}$).
\end{itemize}

We call a round \emph{ambiguous} if, under the agent's current information state, the Bayes-optimal error is at least some fixed constant $\varepsilon_0>0$; let $\gamma\in[0,1]$ denote the fraction of such ambiguous rounds. In this setting, a static policy is insufficient. The agent must execute a dynamic process that \emph{jointly} (1) executes tasks based on its current estimate $\hat M_t$, and (2) ingests human feedback to update $\hat M_t \rightarrow \hat M_{t+1}$, continually improving its approximation of $M_{t+1}^*$.

\subsection{PAHF Framework}
\label{sec:pahf_loop}

The PAHF framework is our practical algorithm for solving this online personalization problem. It operationalizes the estimate update $\hat M_t \rightarrow \hat M_{t+1}$ through an interactive three-step loop designed to mitigate errors from both partial observability and non-stationarity.

\textbf{Pre-Action Interaction.} This step is designed to efficiently mitigate errors from \textit{partial observability}. The loop begins when the agent receives $I_t$ and observes $O_t$. It first queries its memory $\hat M_t$ for relevant preferences, $m_t = \operatorname{Retrieve}(\hat M_t, I_t, O_t)$. If the agent identifies ambiguity (e.g., the instruction is ambiguous and no relevant information is found in memory), it then proactively generates a clarification query $q_t$ to the user, receiving pre-action feedback $f_t^{\text{pre}}$. Crucially, this feedback is parsed and written to memory before acting:
\begin{equation}
\label{eq:pre_update}
\hat M'_t = \mathcal{F}_{\operatorname{update}}^{\text{pre}}(\hat M_t, I_t, O_t, m_t, q_t, f_t^{\text{pre}}).
\end{equation}
This allows the agent to resolve ``known uncertainty'' using $(m_t, q_t, f_t^{\text{pre}})$ when choosing its action, while also storing the information in $\hat M'_t$ for future rounds.

\textbf{Action Execution.}
The agent's action policy $\pi_{\operatorname{act}}$ formulates the final executable action $a_t$. Its role is to synthesize all available information, conditioning jointly on the original instruction $I_t$, the observation $O_t$, the retrieved preferences $m_t$, and the pre-action interaction context $(q_t, f_t^{\text{pre}})$:
\begin{equation}
a_t = \pi_{\operatorname{act}}(I_t, O_t, m_t, q_t, f_t^{\text{pre}}).
\end{equation}
In the case where no clarification is needed, $q_t$ and $f_t^{\text{pre}}$ are null and $\hat M'_t = \hat M_t$. The policy must intelligently integrate these inputs to form a coherent plan. For example, given the task $I_t = \text{``Bring my favorite drink''}$ and observation $O_t$ showing Coke and Sprite, the policy's behavior adapts: if $m_t$ provides a relevant preference (e.g., $\{\text{favorite drink is Coke}\}$), the policy uses this directly to formulate $a_t$ (e.g., pick up the Coke). However, if $m_t$ is empty, the agent must rely entirely on the interaction (e.g., $q_t = \text{``Which drink do you prefer?''}$, $f_t^{\text{pre}} = \text{``I'd like a Coke''}$) to resolve the ambiguity and formulate the same action. The agent then executes $a_t$ in the environment, which transitions the world into a new observable state $O_{t+1}$.

\textbf{Post-Action Feedback Integration.}
This step is the agent's \emph{essential} mechanism for learning from errors. Post-action feedback is triggered \emph{only if} the action $a_t$ (taken in the previous step) results in a non-optimal outcome, leading to a new state $O_{t+1}$ that the user deems incorrect. While this reactive channel can also address \emph{partial observability} (e.g., learning an unknown preference after a wrong guess), its unique power is in correcting \emph{miscalibration}---the ``confidently wrong'' state that pre-action queries cannot detect.

After observing the outcome, the user provides post-action feedback $f_t^{\text{post}}$. We then prompt an LLM as a \emph{feedback detector} to determine whether $f_t^{\text{post}}$ contains salient, personalized information that should be stored. If such information is found (e.g., ``Actually, I like Sprite most now''), this feedback drives the post-action update function. This function integrates the new feedback with the current memory state $\hat M'_t$, producing the next memory state $\hat M_{t+1}$:
\begin{equation}
\hat M_{t+1} = \mathcal{F}_{\operatorname{update}}^{\text{post}}(\hat M'_t, I_t, m_t, q_t, f_t^{\text{pre}}, a_t, f_t^{\text{post}}).
\end{equation}
If no post-action feedback is given, or if the detector classifies the feedback as non-informational (e.g., ``Thank you''), the memory state simply carries over: $\hat M_{t+1} = \hat M'_t$. This post-action update allows the agent to \texttt{add}, \texttt{revise}, and \texttt{refine} its knowledge, such as revising $m_{\text{coke}} \rightarrow m_{\text{sprite}}$ after $f_t^{\text{post}} = \text{``I now like Sprite most.''}$

In the next section, we prove the necessity of both pre-action and post-action feedback channels. Pre-action feedback prevents initial errors caused by partial observability, whereas post-action feedback is essential for adapting to preference drift and correcting confidently wrong miscalibration. Their combination minimizes cumulative personalization error.

\subsection{Theoretical Justification for PAHF}
\label{sec:theory-necessity}

\paragraph{Setup.}
At round $t\in[T]$, the user has latent preference state $M_t^*$ and the agent holds an estimate $\hat M_t$.
Let $a_t$ be the agent’s action and $a_t^*$ an optimal action under $M_t^*$.
Define the $0$–$1$ loss $L_t=\mathbf{1}[a_t\neq a_t^*]$.
Preferences are piecewise stationary with at most $K\!\ge\!1$ switches over $T$ rounds.
A round is \emph{ambiguous} if, under the agent’s information state (posterior conditioned on its transcript up to $t$), the Bayes-optimal error probability is at least some fixed constant $\varepsilon_0>0$; let $\gamma\in[0,1]$ denote the fraction of ambiguous rounds.
The agent may ask up to $k$ balanced $m$-ary \emph{pre-action} questions on ambiguous rounds, where $m\!\ge\!2$ is fixed, and may apply a \emph{post-action} update after a mistake. 

\begin{proposition}[Post-action feedback is necessary under preference drift]
\label{prop:post-necessary}
Suppose rounds are unambiguous under the agent's information state (i.e., $\gamma=0$) and preferences switch at most $K\!\ge\!1$ times.
Any policy that never uses post-action feedback incurs
$\mathbb{E}\!\left[\sum_{t=1}^T L_t\right]=\Omega(T)$.
Conversely, any policy that performs an update upon the first error after each switch makes at most $O(K)$ expected mistakes.
\end{proposition}

\begin{proof}[Sketch]
Without post-action feedback, pre-action histories are indistinguishable across switch times, so the agent cannot condition on when the change occurs. This failure to detect drift implies $\Omega(T)$ expected mistakes, as the agent persists with stale preferences for a constant fraction of rounds. Conversely, with post-action feedback revealing the correct action, the agent updates its explicit memory after the first error following each switch, yielding at most one mistake per switch ($\le K$).
\end{proof}

\begin{remark}[Context-dependent preferences]
Preferences may depend on context $C_t$ (e.g., time, location, state). 
When $C_t$ changes, using a context-agnostic ``global'' note can yield a confidently wrong action (without triggering pre–action queries); by Proposition~\ref{prop:post-necessary}, a post–action signal is needed to correct miscalibration.
\end{remark}

\begin{proposition}[Pre-action feedback is necessary under partial observability]
\label{prop:pre-necessary}
Assume $M_t^*$ is stationary and a fraction $\gamma>0$ of rounds are ambiguous.
Any policy that never requests pre-action clarification suffers
$\mathbb{E}\!\left[\sum_{t=1}^T L_t\right]=\Omega(\gamma T)$.
If the agent asks up to $k$ balanced $m$-ary pre-action questions \emph{on each ambiguous round}, the expected number of errors on those rounds is $O(\gamma T\cdot m^{-k})$.
\end{proposition}
\begin{proof}[Sketch]
Without querying, each ambiguous round has Bayes-optimal error probability at least $\varepsilon_0>0$, so the expected number of mistakes is $\Omega(\gamma T)$.
Under balanced $m$-ary queries, each question contracts the posterior mass of hypotheses that induce a wrong action by a factor of at most $1/m$; after $k$ questions, the residual error probability is at most $m^{-k}$.
Summed over $\gamma T$ ambiguous rounds, this yields $O(\gamma T\cdot m^{-k})$ errors.
\end{proof}

\begin{theorem}[Complementarity of pre- and post-action feedback]
\label{thm:complementarity}
Let $\pi^*$ be an oracle policy that knows $M_t^*$ at every round, and define the dynamic regret
$\mathcal{R}_T=\sum_{t=1}^T\big(L_t - \mathbf{1}[a_t^{\pi^*}\neq a_t^*]\big)$.
With at most $K$ switches and ambiguity rate $\gamma$, any PAHF policy that
(i) asks up to $k$ balanced $m$-ary pre-action questions on ambiguous rounds and
(ii) performs an immediate post-action update after the first error per switch
satisfies
$\mathbb{E}[\mathcal{R}_T] = O\!\big(K + \gamma T m^{-k}\big)$.
In particular, choosing $k=\Theta(\log_{m} T)$ yields $\mathbb{E}[\mathcal{R}_T]=O(K+\gamma)$.
\end{theorem}
\begin{proof}[Sketch]
Since the oracle $\pi^*$ knows $M_t^*$, it never errs, so $\mathcal{R}_T = \sum_{t=1}^T L_t$ just counts mistakes made by PAHF. On unambiguous rounds, each preference switch causes at most one error before the post-action update corrects the estimate (Proposition~\ref{prop:post-necessary}), contributing $O(K)$ mistakes. On ambiguous rounds, up to $k$ balanced $m$-ary pre-queries shrink the error probability to at most $m^{-k}$ (Proposition~\ref{prop:pre-necessary}), so across at most $\gamma T$ such rounds this adds $O(\gamma T m^{-k})$ mistakes. Summing the two contributions yields the stated bound; choosing $k = \lceil \log_m T \rceil$ makes the second term $O(\gamma)$.
\end{proof}

\subsection{Implementation}
\label{sec:implementation_details}

To empirically validate PAHF and the theory in Section~\ref{sec:theory-necessity}, we constructed datasets, implemented the agent with two different memory backends, defined baselines, and designed evaluation metrics.

\textbf{Baselines.} 
Our agent uses a large language model (LLM) and follows the ReAct framework \citep{yao2022react} to interleave reasoning and acting. Unless otherwise noted, we instantiate the agent with GPT-4o. In Appendix~\ref{sec:add_results}, we report ablation studies that vary both the agent model (GPT-4.1) and the human simulator model (GPT-4.1, GPT-4o, and Llama-4-Scout) to examine how the quality of simulated human feedback affects agent performance.
We compare four settings designed to isolate the roles of memory and feedback. Our primary baseline is (i) \emph{No Memory} (no persistent store). We then evaluate three memory-enabled agents: (ii) \emph{Pre-action Only} (clarification allowed, but no post-action updates), (iii) \emph{Post-action Only} (no clarifications, learns only from corrective feedback), and (iv) \emph{PAHF} (both channels enabled; Section~\ref{sec:pahf_loop}). The \emph{Pre-action Only} baseline targets partial observability (cf.\ $\gamma>0$, $K{=}0$); \emph{Post-action Only} targets preference drift (cf.\ $\gamma{=}0$, $K{\ge}1$). PAHF's hybrid approach is motivated by the reality of tasks exhibiting both error sources.

\textbf{Memory design.} Our goal in this work is not to propose a new memory architecture, but to integrate human feedback into memory for continual personalization. We therefore adopt standard dense-retrieval memory backends—a portable SQLite note store and a FAISS-based vector index—behind a shared API to keep the implementation simple and easily reproducible. Each memory entry is a short natural-language note paired with an embedding, and all notes are strictly isolated per user. The API exposes only the functionality the agent needs: add a note, retrieve a small top-$k$ set of relevant notes for a query, detect near-duplicates via similarity to decide “update vs.\ add,” update an existing note in place, and enumerate or address notes by id. The SQLite backend persists notes in a simple on-disk table and computes similarities on demand (favoring simplicity and reproducibility), while the FAISS backend maintains an in-memory nearest-neighbor index with optional save/load (favoring fast retrieval at scale). Both backends provide identical retrieval semantics so the rest of the system can treat memory as a black box; more sophisticated memory modules (e.g., hierarchical or structured memory) are complementary to PAHF and can be integrated as drop-in replacements in future work. Unless stated otherwise, we report all main results with SQLite; FAISS results are in the Appendix \ref{sec:add_results}.

\textbf{Memory interaction (agentic memory).}
The agent's memory interaction $M_t$ consists of two core components: reading (retrieval) and writing (updates). 
Reading follows a standard Retrieval-Augmented Generation (RAG) pipeline. We compute a query embedding from $(I_t,O_t)$ and run $k$-nearest-neighbor search over indexed memory embeddings (top-$k$ dense retrieval; embeddings from DRAGON+ \citep{lin2023train}.
$
\operatorname{Retrieve}(M_t, I_t, O_t)
$
returns a small, ranked set of candidate preference records $m_t$. Retrieval always precedes action selection so the policy conditions on $m_t$ rather than an empty context. We then apply a lightweight information-extraction step conditioned on $(I_t,O_t,m_t)$ to distill the retrieved notes into task-relevant personalized cues; the resulting cleaned summary is inserted into the model context for action selection.

Writing is a multi-step process triggered any time the agent receives personalized information from either pre-action feedback ($f_t^{\text{pre}}$) or post-action feedback ($f_t^{\text{post}}$). First, the feedback is passed to a ``salience detector'', implemented as an LLM-as-a-judge, which identifies and discards non-informational feedback. If the feedback is deemed salient, we use an LLM to summarize and extract the core personalized note. We then integrate this note with the memory: retrieve the most relevant existing note; if similarity exceeds a threshold $\tau$, we merge by asking the LLM to produce an updated note that replaces the old text for that entry. If similarity is below $\tau$, we add a new note. This detect–summarize–integrate pipeline keeps the database clean while enabling quick corrections (post-action) and steady accumulation of useful details (pre-action).

\section{Evaluations}
\label{sec:evaluations}

\subsection{Evaluation Domains}
\label{sec:experimental_domains}
We evaluate our framework across two distinct domains to ensure generality. We outline the core design of each domain below, while full details on dataset construction and discussion are provided in Appendix \ref{sec:add_dataset}.

\textbf{Embodied Manipulation Domain}. Our dataset models everyday indoor mobile-manipulation tasks across home and office settings, expressed as natural instructions in two modes: selecting the right item and placing an item in the right location. Each scenario is self-contained: a plain-language user request (e.g., “Could you bring me my favorite drink?”), the full inventory of available objects and locations, and the identity of the requesting persona (e.g., Alex). To emphasize personalization, many scenarios are context-dependent (e.g., “I’m drowsy, could you bring me something to drink?”), where the correct action shifts with physiological or emotional state (anxious, unwell), task-oriented goals (work energy, healthy eating), and external factors such as social events (social sharing) or environmental/temporal cues (cold weather, morning). 
To increase the difficulty of personalized preference modeling, we deliberately design personas with idiosyncratic and unconventional preferences, so that correct actions cannot be inferred from generic commonsense heuristics.

We simulate the human behavior by employing another LLMs with specific persona.  Each persona is defined by a detailed prompt that outlines a baseline preference (e.g., ``favorite drink is black coffee'') which is then consistently superseded by a rich set of context-dependent exceptions (e.g., ``when drowsy, prefers herbal tea''). The persona model is designed to provide feedback via two modalities: (1) brief, natural-language answers to the agent's clarification questions, and (2) post-action feedback after a task is completed. To capture human preference evolution, each user persona exists in two versions: an ``original'' profile and an ``evolved'' profile with updated preferences, allowing the “right” choice to change while scenes and instructions remain the same.

\textbf{Online Shopping Domain}. We model personalized online shopping via natural-language purchase requests. The agent must choose one of three product candidates (A/B/C) or abstain (“D”: no purchase) if none are suitable. Each product is described by a compact set of discrete, user-salient features (e.g., a camera's sensor type, lens mount, and viewfinder). The ground truth is governed by each persona's \emph{acceptance policy}, which is strictly conjunctive: the user specifies preferred and acceptable values for each feature, and a candidate is acceptable only if \emph{all} its features meet these criteria. To test fine-grained discrimination, options are adversarially constructed as ``near-misses''—distractors that combine highly preferred attributes with a single disqualifying trait (``poison pill''). This forces the agent to perform fine-grained reasoning to verify user preferences against each attribute. Consequently, this design makes the domain significantly more challenging—and resistant to performance gains—than embodied manipulation benchmark.

In this domain, we use a hybrid human simulator. A persona-conditioned LLM handles all pre-purchase clarification in natural language, while a deterministic, rule-based judge renders post-purchase verdicts and feedback. During clarification, the LLM mimics a user by answering only the specific feature questions the agent asks. After the agent commits (A/B/C/D), the rule-based evaluator programmatically checks the choice against the acceptance policy. If the choice is incorrect, the evaluator returns minimal, targeted feedback that pinpoints the specific offending feature. As in the embodied domain, each user has two profiles—an original and an evolved version with updated acceptance policies—to test adaptation to preference drift.

\subsection{Evaluation Protocol}
\label{sec:evaluation_protocol}

To measure continual personalization, we designed a four-phase evaluation protocol to separately test an agent's ability to learn an initial preference profile and, subsequently, its ability to adapt when that profile changes. This protocol includes two interactive \emph{learning phases} and two \emph{test phases}.

\textbf{Phase 1: Initial Learning (Training).}
The agent starts with an empty memory ($M_t = \emptyset$) and interacts with multiple users (each simulated with a unique, fixed persona) over a set of scenarios. These scenarios are ambiguous, requiring the agent to use feedback to learn each user's initial preferences. We run agents over this dataset for multiple epochs to measure the learning curve as the memory bank $M_t$ is populated.

\textbf{Phase 2: Initial Personalization (Testing).}
We evaluate the agent (using the memory learned in Phase 1) on a corresponding test set. These scenarios use the \emph{same user persona} from Phase 1 but with different instructions and observations. This phase tests the agent's ability to \emph{exploit} its learned knowledge in new situations. In this test phase, there is no human feedback involved, so the agent makes decisions solely on its memory.

\textbf{Phase 3: Adaptation to Drift (Training).}
To test adaptation to non-stationarity, we re-introduce the training scenarios from Phase 1 but with a critical change: each user's persona is \emph{swapped} for a new one (e.g., ``favorite drink: Coke'' $\to$ ``favorite drink: Sprite''). An agent relying on its memory from Phase 1 will now be ``confidently wrong'' and will fail. This phase tests the agent's ability to detect the resulting error and use new (primarily post-action) feedback to \emph{revise} and \emph{overwrite} its outdated beliefs.

\textbf{Phase 4: Adapted Personalization (Testing).}
Finally, we evaluate the agent (using the adapted memory from Phase 3) on a second test set. This set corresponds to Phase 2 but uses the \emph{newly adapted persona} from Phase 3. This phase measures whether the agent has correctly revised its memory, demonstrates proper adaptation to new scenarios, and is no longer “stuck” on the initial preferences.

This protocol allows us to separately quantify performance on initial learning (Phase 1 and 2) and adaptation to preference drift (Phase 3 and 4), providing a comprehensive view of continual personalization.

\emph{Size and splits.} In the embodied domain, we evaluate 40 users, each on 30 scenarios per phase, yielding 1,200 scenarios per phase. With two learning phases and two evaluation phases, this gives 2,400 learning and 2,400 evaluation scenarios. In the online shopping domain, we evaluate 20 users, each on 45 scenarios per phase, yielding 900 scenarios per phase and, analogously, 1,800 learning scenarios and 1,800 evaluation scenarios.

\textit {Metrics.}
We report three metrics.\; \textit{Success Rate (SR)} is the fraction of tasks completed correctly: $\mathrm{SR}=\tfrac{1}{N}\sum_{i=1}^N \mathbf{1}[\text{correct}_i]$.\; \textit{Feedback Frequency (FF)} is the proportion of tasks that used \emph{any} human feedback (pre or post-action) at least once: $\mathrm{FF}=\frac{1}{N}\sum_{i=1}^N \mathbf{1}\!\left[\text{pre}_i+\text{post}_i>0\right]$.\; 
To analyze how personalization improves over iterations within a phase, we also track an
\textit{average cumulative personalization error (ACPE)}.
For a phase with $T$ learning iterations, let $\mathrm{PE}_t \in [0,1]$ denote the personalization error rate
(e.g., fraction of incorrect tasks) at iteration $t \in \{1,\dots,T\}$.
We define the average cumulative personalization error up to iteration $t$ as
$
\mathrm{ACPE}_t
= \frac{1}{t} \sum_{s=1}^t \mathrm{PE}_s,
$
which lies in $[0,1]$ and can be interpreted as a percentage (0\% best, 100\% worst).

\section{Results}
\label{sec:results}

\begin{figure*}[h!]
  \centering
    \begin{subfigure}[t]{0.315\linewidth}
        \includegraphics[width=\linewidth]{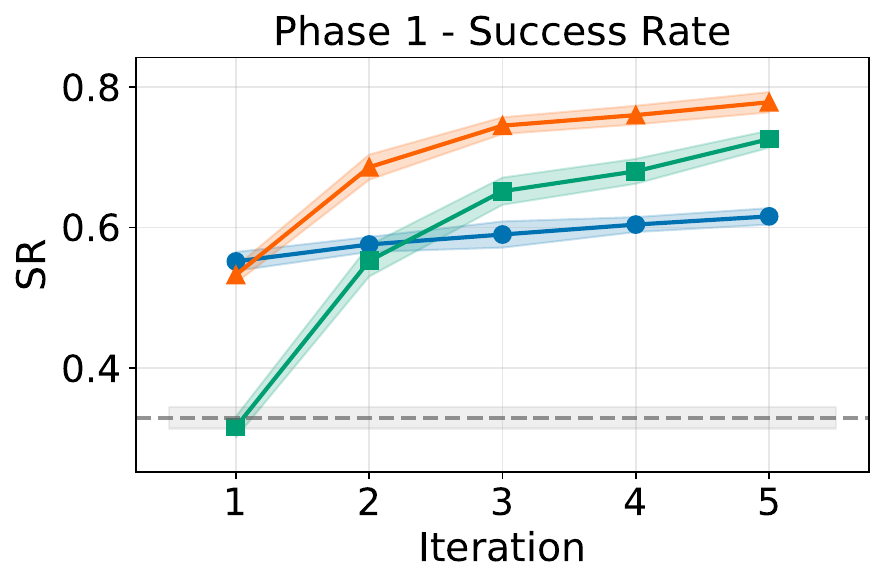}
        \label{fig:embodied_4o_4o_phase1_sr}
    \end{subfigure}    
    ~
    \begin{subfigure}[t]{0.315\linewidth}
        \includegraphics[width=\linewidth]{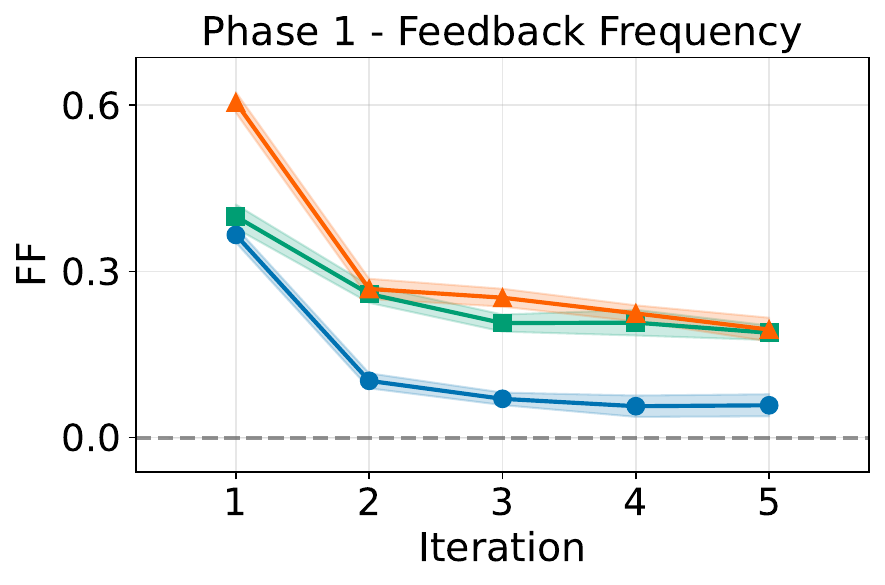}
        \label{fig:embodied_4o_4o_phase1_ff}
    \end{subfigure}
    ~
    \begin{subfigure}[t]{0.315\linewidth}
        \includegraphics[width=\linewidth]{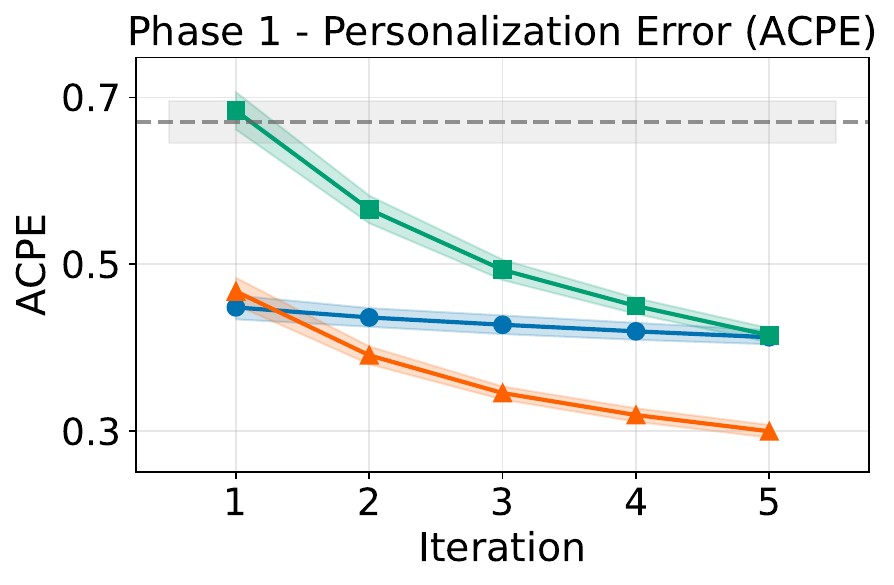}
        \label{fig:embodied_4o_4o_phase1_acpe}
    \end{subfigure}

    \vspace{-2mm}
    \begin{subfigure}[t]{0.315\linewidth}
        \includegraphics[width=\linewidth]{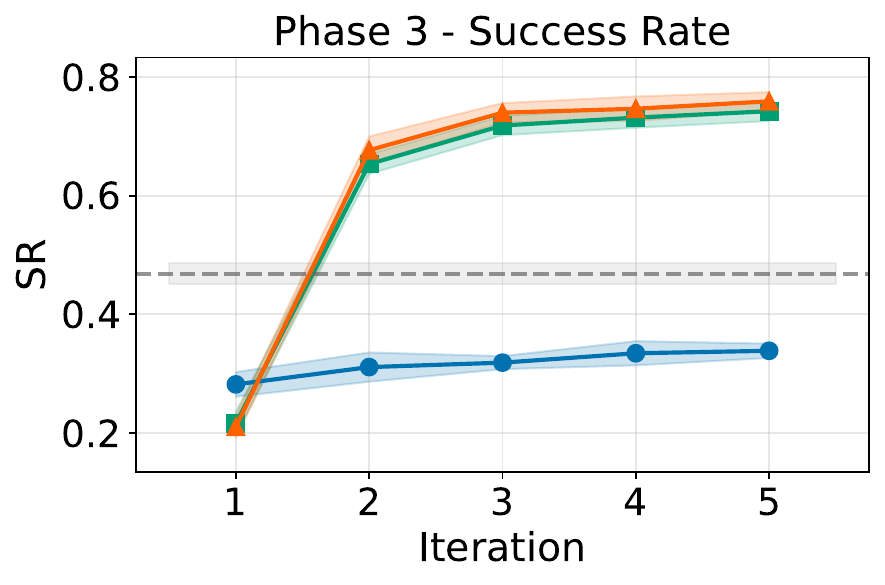}
        \label{fig:embodied_4o_4o_phase3_sr}
    \end{subfigure}    
    ~
    \begin{subfigure}[t]{0.315\linewidth}
        \includegraphics[width=\linewidth]{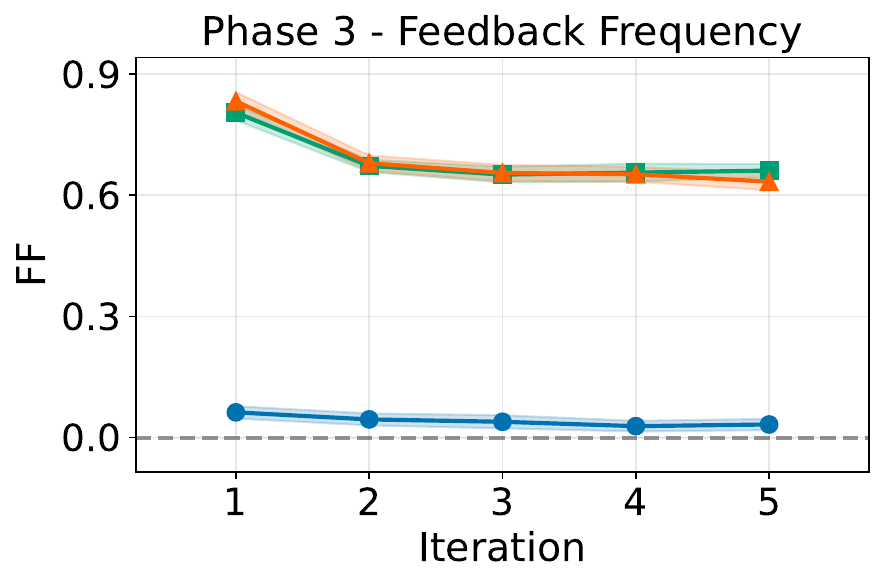}
        \label{fig:embodied_4o_4o_phase3_ff}
    \end{subfigure}
    ~
    \begin{subfigure}[t]{0.315\linewidth}
        \includegraphics[width=\linewidth]{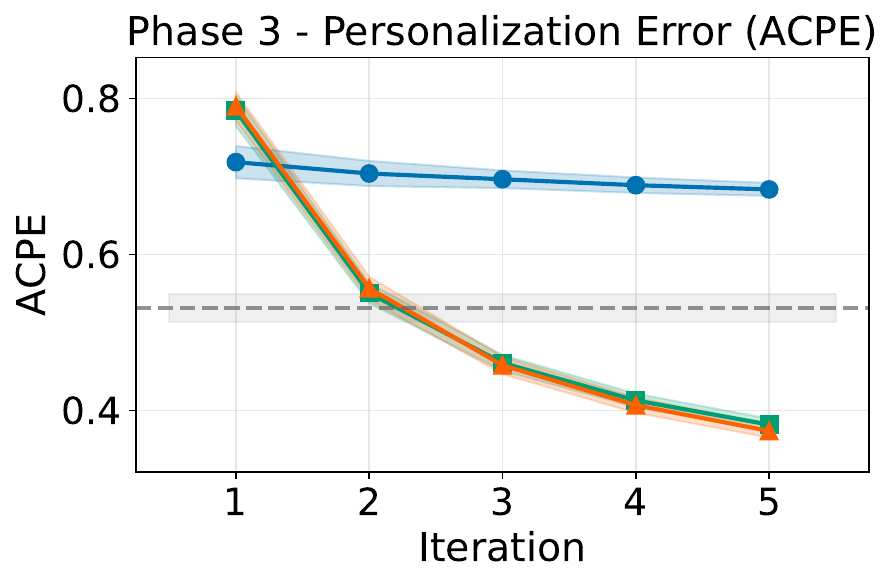}
        \label{fig:embodied_4o_4o_phase3_acpe}
    \end{subfigure}

    \par\vspace{-0.28cm}
    \begin{subfigure}[c]{\linewidth}
        \centering
        \includegraphics[width=0.6\linewidth]{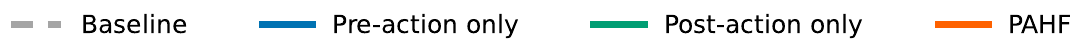}
    \end{subfigure}

    \caption{\textbf{Results on embodied manipulation}. Top row: Phase 1 learning curves for success rate (left), feedback frequency (middle), and average cumulative personalization error (ACPE) (right). Bottom row: Phase 3 learning curves for success rate (left), feedback frequency (middle), and ACPE (right). Shaded regions denote standard error.}
    \label{fig:embodied_gpt_4o_gpt_4o}

\end{figure*}

\begin{figure*}[h!]
  \centering
    \begin{subfigure}[t]{0.315\linewidth}
        \includegraphics[width=\linewidth]{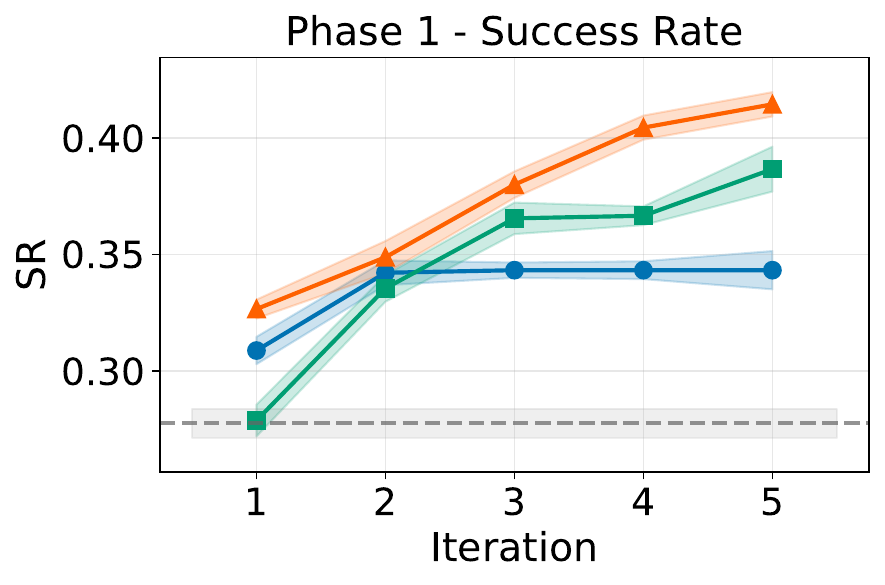}
        \label{fig:shopping_4o_4o_phase1_sr}
    \end{subfigure}    
    ~
    \begin{subfigure}[t]{0.315\linewidth}
        \includegraphics[width=\linewidth]{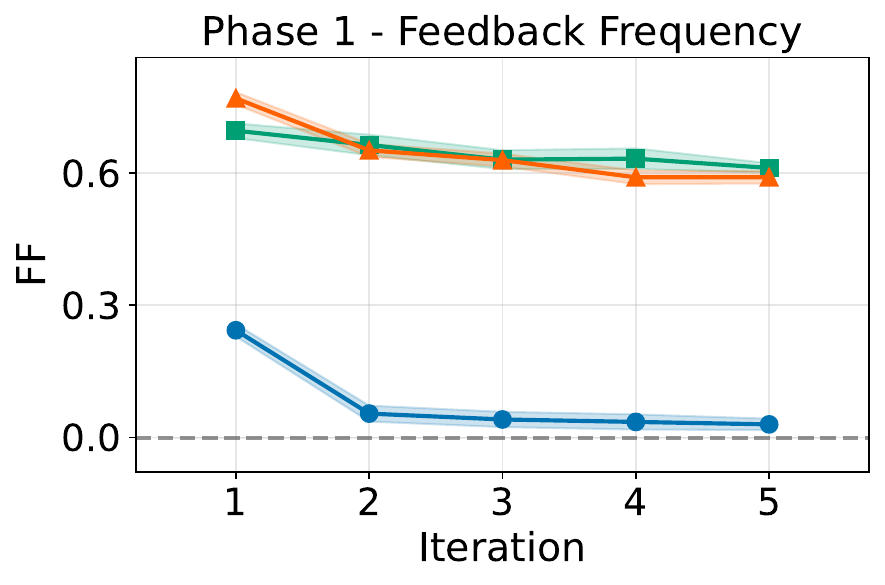}
        \label{fig:shopping_4o_4o_phase1_ff}
    \end{subfigure}
    ~
    \begin{subfigure}[t]{0.315\linewidth}
        \includegraphics[width=\linewidth]{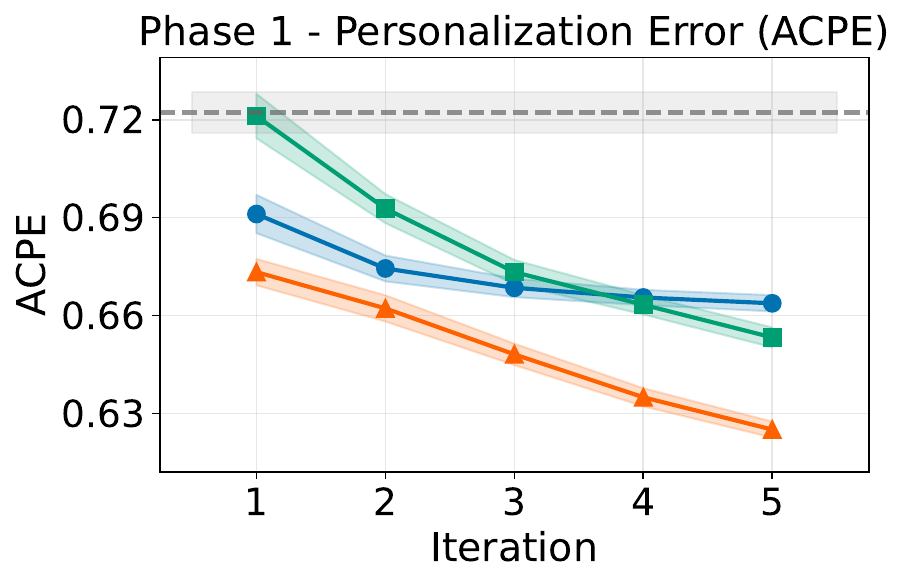}
        \label{fig:shopping_4o_4o_phase1_acpe}
    \end{subfigure}

    \vspace{-2mm}
    \begin{subfigure}[t]{0.315\linewidth}
        \includegraphics[width=\linewidth]{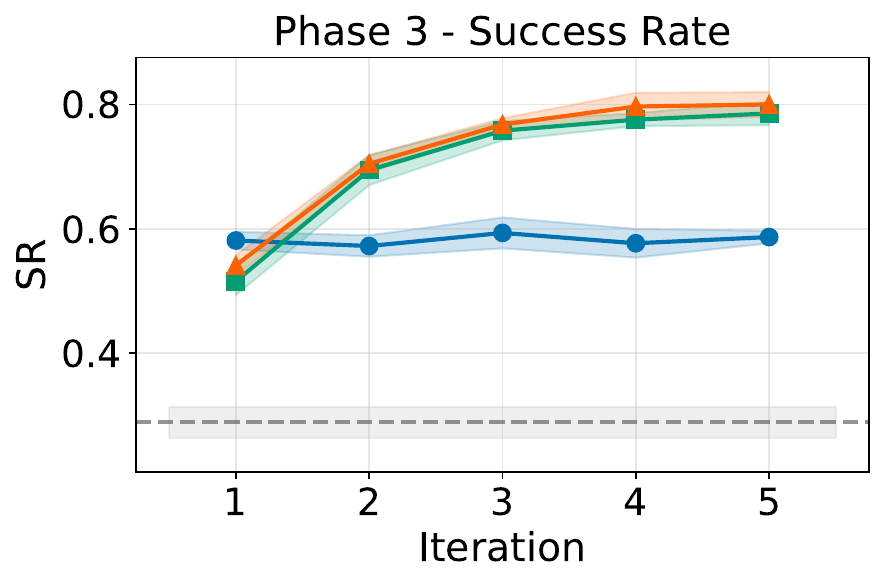}
        \label{fig:shopping_4o_4o_phase3_sr}
    \end{subfigure}    
    ~
    \begin{subfigure}[t]{0.315\linewidth}
        \includegraphics[width=\linewidth]{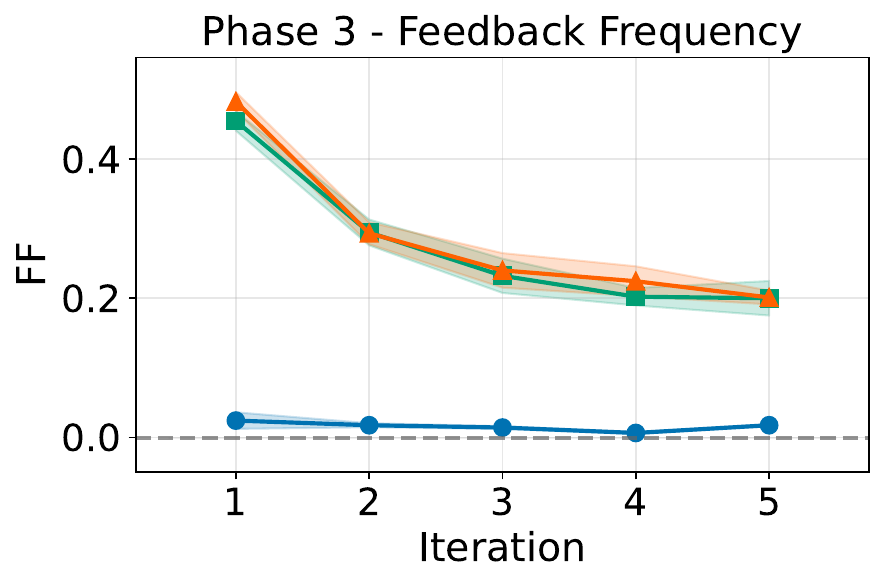}
        \label{fig:shoppping_4o_4o_phase3_ff}
    \end{subfigure}
    ~
    \begin{subfigure}[t]{0.315\linewidth}
        \includegraphics[width=\linewidth]{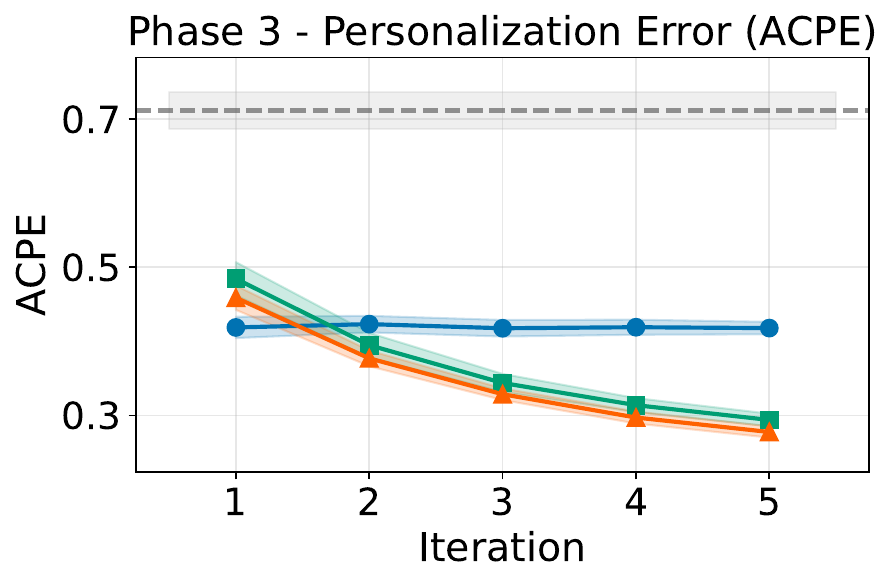}
        \label{fig:shopping_4o_4o_phase3_acpe}
    \end{subfigure}

    \par\vspace{-0.28cm}
    \begin{subfigure}[c]{\linewidth}
        \centering
        \includegraphics[width=0.6\linewidth]{figure/plots/legend.pdf}
    \end{subfigure}

    \caption{\textbf{Results on online shopping}. Top row: Phase 1 learning curves for success rate (left), feedback frequency (middle), and average cumulative personalization error (ACPE) (right). Bottom row: Phase 3 learning curves for success rate (left), feedback frequency (middle), and ACPE (right). Shaded regions denote standard error.}
    \label{fig:shopping_gpt_4o_gpt_4o}
    
\end{figure*}

\textbf{Pre-action feedback prevents initial personalization error.} In Phase~1 (initial online learning), agents equipped with pre-action clarification (Pre-action Only and PAHF) achieve substantially higher success rate on the very first interaction than both the Post-action Only agent and the no-memory baseline (Figures~\ref{fig:embodied_gpt_4o_gpt_4o},~\ref{fig:shopping_gpt_4o_gpt_4o}, top-left panels). By asking targeted questions before acting, they can align with the user’s preferences from the outset and avoid the large early mistakes that the other agents make in ambiguous situations. This one-step “warm start” is reflected in the ACPE metric, where pre-action agents incur lower cumulative personalization error during the early stages of Phase~1 (top-right panels). This behavior matches the intended role of pre-action feedback: it is especially valuable for preventing initial personalization errors under partial observability.

\textbf{Pre-action feedback alone is brittle under preference drift.}
When user preferences change in Phase~3, the Pre-action Only agent struggles to correct its previously learned but now outdated beliefs. Its Phase-3 success-rate curves show very limited improvement (Figures~\ref{fig:embodied_gpt_4o_gpt_4o} and~\ref{fig:shopping_gpt_4o_gpt_4o}, bottom-left panels), and in the embodied domain it even stays below the no-memory baseline. The corresponding ACPE curves indicate that its cumulative personalization error remains substantially higher than that of agents with post-action feedback across both domains (Figures~\ref{fig:embodied_gpt_4o_gpt_4o} and~\ref{fig:shopping_gpt_4o_gpt_4o}, bottom-right panels). We observe that once the agent has written confident notes into memory, it no longer perceives ambiguity and therefore stops asking clarification questions, failing to detect that its knowledge has become stale. This is reflected in the Phase-3 feedback-frequency curves, where the Pre-action Only agent rarely requests clarification despite being systematically wrong.

\textbf{Post-action feedback is essential for fast adaptation.}
Agents that receive post-action feedback (Post-action Only and PAHF) demonstrate effective reactive learning: in both Phase 1 and Phase 3, across both domains, their success rates improve steeply over iterations whenever feedback is available (Figures~\ref{fig:embodied_gpt_4o_gpt_4o} and~\ref{fig:shopping_gpt_4o_gpt_4o}, bottom-left panels). The ACPE curves mirror this pattern, showing sharp drops in cumulative personalization error as the agent uses corrective feedback to update its beliefs (bottom-right panels). In the dynamic-preference setting, post-action feedback is crucial for “un-sticking” overconfident agents whose memories have become stale: unlike the Pre-action Only agent, which stops asking questions because it no longer perceives ambiguity, the Post-action Only agent captures the reactive signal to overwrite outdated beliefs regardless of its confidence. Consequently, it recovers to high Phase-4 success (e.g., $67.9\%$ in the embodied domain), closely approaching PAHF ($70.5\%$ in the same domain; Table~\ref{table:gpt_4o_gpt_4o}). However, because this mechanism is strictly reactive—triggered only after a mistake—relying on it alone forces the agent to learn through trial and error. This leads to large initial personalization errors, imposing higher user costs via incorrect actions and frequent corrective feedback.

\begin{wraptable}{r}{0.52\linewidth}
\vspace{-2mm}
\caption{\textbf{Evaluation success rates.} Phase-2 and Phase-4 success rates (\%) for each method in embodied and shopping domains.}
\centering
\small
\setlength{\tabcolsep}{3.5pt}
\vspace{-1.5mm}
\begin{tabular}{lcccc}
\toprule
\multirow{2}{*}{Method} & \multicolumn{2}{c}{Embodied} & \multicolumn{2}{c}{Shopping} \\
\cmidrule(r){2-3}\cmidrule(l){4-5}
      & Phase 2 & Phase 4           & Phase 2 & Phase 4            \\
\midrule
No memory         & $32.3_{\pm 0.4}$ & $44.8_{\pm 0.5}$ & $27.8_{\pm 0.2}$ & $27.0_{\pm 0.4}$ \\
Pre-action only   & $54.1_{\pm 1.1}$ & $35.7_{\pm 1.0}$ & $34.4_{\pm 0.5}$ & $56.0_{\pm 0.7}$ \\
Post-action only  & $67.9_{\pm 1.5}$ & $68.3_{\pm 1.2}$ & $38.9_{\pm 0.5}$ & $66.9_{\pm 0.8}$ \\
\textbf{PAHF (pre+post)} & $\mathbf{70.5_{\pm 1.7}}$ & $\mathbf{68.8_{\pm 1.3}}$ & $\mathbf{41.3_{\pm 0.8}}$ & $\mathbf{70.3_{\pm 1.1}}$ \\
\bottomrule
\end{tabular}
\label{table:gpt_4o_gpt_4o}
\vspace{-2mm}
\end{wraptable}

\textbf{PAHF combines the strengths of both channels.} By enabling both clarification before acting and corrective updates after mistakes, PAHF inherits the advantages of each feedback type while mitigating their weaknesses. In Phase~1, PAHF’s success and ACPE curves generally surpass those of the other methods; it effectively minimizes initial personalization error through pre-action feedback and learns efficiently from post-action feedback. After preference drift, PAHF matches the rapid adaptation of the Post-action Only agent: in both domains, its Phase-3 learning curves nearly overlap with Post-action Only in success rate, yet PAHF consistently achieves the lowest cumulative personalization error (ACPE). Across all four evaluation phases and both domains, PAHF achieves the highest success rate (Table~\ref{table:gpt_4o_gpt_4o}), demonstrating that combining pre-action and post-action feedback with explicit memory is critical for robust continual personalization.

Taken together, these results support the complementarity of the two feedback channels. Pre-action feedback prevents large initial personalization errors, post-action feedback is crucial for correcting confidently wrong beliefs after preference drift, and PAHF’s joint use of both yields the strongest personalization performance.

\section{Conclusion}
\label{sec:conclusion}

Modern AI agents are increasingly deployed in interactive settings, yet most personalization pipelines still assume static user data and stationary preferences. In this work, we introduce Personalized Agents from Human Feedback (PAHF), a continual personalization framework that enables agents to learn online via explicit memory and dual feedback channels. Theoretically, we show that proactive pre-action queries and reactive post-action corrections address complementary failure modes: the former resolves ambiguity under partial observability, while the latter is essential for correcting miscalibrated beliefs under preference drift. Empirically, evaluations in embodied manipulation and online shopping domains demonstrate that PAHF consistently outperforms single-channel baselines. While pre-action feedback reduces early errors and post-action feedback enables recovery after persona shifts, our results show that only their combination, together with explicit memory, delivers robust performance under evolving user preferences.

\section{Acknowledgement}
\label{sec:ack}
We thank Tom Griffiths for insightful feedback on the framing and presentation, and discussions that strengthened the paper.

\clearpage
\newpage
\bibliographystyle{assets/plainnat}
\bibliography{paper}

\clearpage
\newpage
\appendix
\appendix

\section{Full Proofs for Section~\ref{sec:pahf_framework}}
\label{sec:formal_proof}

\paragraph{Notation recap.}
At round $t\in[T]$, the user’s latent preference state is $M_t^*\in\mathcal{M}$ and the agent maintains a preference memory $\hat M_t$.
Given instruction $I_t$ and observation $O_t$, the agent chooses action $a_t$, and let $a_t^*$ denote a (fixed) optimal action under $M_t^*$.
We use the $0$--$1$ loss $L_t=\mathbf{1}[a_t\neq a_t^*]$.
Preferences are piecewise stationary with at most $K$ switches over $T$ rounds.
A round is \emph{ambiguous} if, conditioned on the agent’s transcript up to $t$, the Bayes-optimal error probability (over actions) is at least a fixed constant $\varepsilon_0>0$; let $\gamma\in[0,1]$ denote the fraction of such rounds.

\paragraph{Feedback model.}
We distinguish two channels:
(i) \emph{Pre-action} feedback: on ambiguous rounds, the agent may ask up to $k$ clarification questions of arity at most $m\ge 2$. Answers are truthful.
(ii) \emph{Post-action} feedback: after acting, if $a_t\neq a_t^*$, the user provides a corrective signal which the agent may use to update its preference memory.

\paragraph{Information-theoretic assumptions.}
\begin{itemize}
\item[\textbf{A1}] (\textbf{Balanced $m$-ary pre-queries})
On an ambiguous round, the agent can select an $m$-ary question such that, for every possible answer, the posterior probability mass on hypotheses that induce a \emph{wrong} action is at most a $1/m$ fraction of its previous value. Answers are noise-free.

\item[\textbf{A2}] (\textbf{Corrective post-signal on errors})
When $a_t\neq a_t^*$ on an \emph{unambiguous} round, the post-action feedback reveals enough information to identify the correct action $a_t^*$ (equivalently, to update to a preference memory that induces $a_t^*$ on subsequent unambiguous rounds until the next switch).
If a policy “does not use post-action feedback,” it does not update $\hat M_t$ based on this signal (i.e., $\hat M_{t+1}=\hat M_t$ on such rounds).
\end{itemize}

\subsection*{Auxiliary lemmas}

\begin{lemma}[Error shrinkage with $m$-ary questions]
\label{lem:halving}
Under \textbf{A1}, after $k$ balanced $m$-ary pre-queries on a single ambiguous round, the posterior probability that the induced action is suboptimal is at most $m^{-k}$.
\end{lemma}

\begin{proof}
Let $q_0$ denote the posterior mass of hypotheses that induce a wrong action before any query on this ambiguous round.
Under \textbf{A1}, each balanced $m$-ary question ensures that, for every possible answer, the posterior mass of wrong hypotheses shrinks by at least a factor $m$; hence after one question $q_1 \le q_0/m$.
Repeating $k$ times yields $q_k \le q_0 m^{-k}\le m^{-k}$ (since $q_0\le 1$).
The Bayes-optimal decision rule has error probability at most the posterior mass on wrong-action hypotheses, so the misclassification probability is at most $m^{-k}$.
\end{proof}

\begin{lemma}[Indistinguishability without post updates]
\label{lem:no-signal}
Suppose $\gamma=0$ and a policy never uses post-action feedback (as in \textbf{A2}).
Consider two environments that differ \emph{only} in the switch time(s) but generate identical $(I_t,O_t)$ sequences.
Then the joint distribution of the agent's internal state and action sequence $(a_1,\dots,a_T)$ is identical across these environments; in particular, it is independent of the switch time(s).
\end{lemma}

\begin{proof}
With $\gamma=0$, the policy never issues pre-action queries and receives no pre-action information.
If it never uses post-action feedback, then by \textbf{A2} its preference memory update does not depend on post-action signals.
Thus the agent’s internal state evolution depends only on its internal randomness and the observed $(I_t,O_t)$ sequence.
If $(I_t,O_t)$ is identical across the two environments, then the induced distribution over internal states and actions is identical as well.
\end{proof}

\subsection*{Proof of Proposition~\ref{prop:post-necessary}}

\begin{proposition*}[Post-action feedback is necessary under preference drift]
Suppose rounds are unambiguous ($\gamma=0$) and preferences switch at most $K\ge 1$ times.
Any policy that never uses post-action feedback incurs
$\mathbb{E}\!\left[\sum_{t=1}^T L_t\right]=\Omega(T)$.
Conversely, any policy that performs a post-action update upon the first error after each switch makes at most $O(K)$ mistakes in expectation.
\end{proposition*}

\begin{proof}
\emph{Lower bound.}
Because the class ``at most $K$ switches'' includes the special case $K=1$, it suffices to prove the claim for a single switch.

Fix $K=1$. Let the switch time $\tau$ be chosen uniformly from $\{1,\dots,T\}$.
Consider an environment with two actions $\{0,1\}$ and an $(I_t,O_t)$ process that is identical for all $t$ and all $\tau$.
Define the optimal action as
\[
a_t^*(\tau)=
\begin{cases}
0, & t < \tau,\\
1, & t \ge \tau.
\end{cases}
\]
Let a policy (possibly randomized) produce an action sequence $a_{1:T}=(a_1,\dots,a_T)$.
By Lemma~\ref{lem:no-signal}, the distribution of $a_{1:T}$ is independent of $\tau$.

Now fix any \emph{deterministic} sequence $a_{1:T}\in\{0,1\}^T$ and define its total loss under switch time $\tau$:
\[
\ell(\tau; a_{1:T})
= \sum_{t=1}^T \mathbf{1}[a_t \neq a_t^*(\tau)]
= \sum_{t<\tau}\mathbf{1}[a_t=1] + \sum_{t\ge \tau}\mathbf{1}[a_t=0].
\]
Averaging over uniform $\tau$ and swapping sums yields
\begin{align*}
\mathbb{E}_{\tau}[\ell(\tau; a_{1:T})]
&= \frac{1}{T}\sum_{\tau=1}^T\left(\sum_{t<\tau}\mathbf{1}[a_t=1] + \sum_{t\ge \tau}\mathbf{1}[a_t=0]\right)\\
&= \frac{1}{T}\sum_{t=1}^T \Big(\mathbf{1}[a_t=1]\cdot |\{\tau:\, t<\tau\}| + \mathbf{1}[a_t=0]\cdot |\{\tau:\, \tau\le t\}|\Big)\\
&= \frac{1}{T}\sum_{t=1}^T \Big(\mathbf{1}[a_t=1]\cdot (T-t) + \mathbf{1}[a_t=0]\cdot t\Big)\\
&\ge \frac{1}{T}\sum_{t=1}^T \min\{t,\,T-t\}
= \frac{1}{T}\Big\lfloor \frac{T^2}{4}\Big\rfloor
= \Omega(T).
\end{align*}
(The inequality holds since for each $t$, either choosing $a_t=0$ incurs contribution $t$ or choosing $a_t=1$ incurs contribution $T-t$, so the smaller is $\min\{t,T-t\}$.)
Therefore, for any deterministic $a_{1:T}$, $\mathbb{E}_\tau[\ell(\tau;a_{1:T})]\ge \lfloor T^2/4\rfloor/T$.
For a randomized policy, taking expectation over the policy’s randomness gives
\[
\mathbb{E} \Big[\sum_{t=1}^T L_t\Big]
= \mathbb{E}_{a_{1:T}}\mathbb{E}_{\tau}[\ell(\tau; a_{1:T})]
\ge \min_{a_{1:T}}\mathbb{E}_{\tau}[\ell(\tau; a_{1:T})]
= \Omega(T).
\]

\emph{Upper bound with reactive updates.}
Now consider a policy that performs a post-action update upon the first error after each switch.
After a switch, it may err once; on that error, by \textbf{A2} the feedback identifies $a_t^*$ (on unambiguous rounds), and after updating the policy acts correctly on subsequent unambiguous rounds until the next switch.
Thus there is at most one mistake per switch, so $\sum_{t=1}^T L_t \le K$ deterministically and $\mathbb{E}[\sum_{t=1}^T L_t]=O(K)$.
\end{proof}

\subsection*{Proof of Proposition~\ref{prop:pre-necessary}}

\begin{proposition*}[Pre-action feedback is necessary under partial observability]
Assume $M_t^*$ is stationary and a fraction $\gamma>0$ of rounds are ambiguous.
Any policy that never requests pre-action clarification suffers
$\mathbb{E}\!\left[\sum_{t=1}^T L_t\right]=\Omega(\gamma T)$.
If the agent asks up to $k$ balanced $m$-ary pre-action questions on each ambiguous round, the expected number of errors on those rounds is $O(\gamma T\cdot m^{-k})$.
\end{proposition*}

\begin{proof}
\emph{Lower bound.}
On an ambiguous round, by definition the Bayes-optimal error probability (given the transcript) is at least $\varepsilon_0>0$.
Without any pre-queries, the policy chooses under this unresolved posterior, so even the Bayes-optimal decision has error probability at least $\varepsilon_0$.
Summed over $\gamma T$ ambiguous rounds, the expected number of errors is $\Omega(\gamma T)$ (absorbing $\varepsilon_0$ into constants).

\emph{Upper bound with $k$ $m$-ary pre-queries.}
By Lemma~\ref{lem:halving}, after $k$ balanced $m$-ary questions on an ambiguous round, the posterior probability that the induced action is wrong is at most $m^{-k}$.
Thus the error probability per ambiguous round is $\le m^{-k}$, and summing over at most $\gamma T$ such rounds gives $O(\gamma T m^{-k})$.
\end{proof}

\subsection*{Proof of Theorem~\ref{thm:complementarity}}

\begin{theorem*}[Complementarity of pre- and post-action feedback]
Assume the optimal action $a_t^*$ is unique (or fix a deterministic tie-break).
Let $\pi^*$ be the oracle policy that knows $M_t^*$, and define the dynamic regret
$\mathcal{R}_T=\sum_{t=1}^T\big(L_t - \mathbf{1}[a_t^{\pi^*}\neq a_t^*]\big)$.
With at most $K$ switches and ambiguity rate $\gamma$, a PAHF policy that
(i) asks up to $k$ balanced $m$-ary pre-action questions only on ambiguous rounds and
(ii) performs an immediate post-action update upon the first error after a switch
achieves
\[
\mathbb{E}[\mathcal{R}_T] \;=\; O\!\big(K \;+\; \gamma T\,m^{-k}\big).
\]
In particular, taking $k=\Theta(\log_m T)$ yields $\mathbb{E}[\mathcal{R}_T]=O(K+\gamma)$.
\end{theorem*}

\begin{proof}
Because $\pi^*$ knows $M_t^*$ each round, it can select the (tie-broken) optimal action $a_t^*$ and thus incurs zero loss:
$\mathbf{1}[a_t^{\pi^*}\neq a_t^*]=0$ for all $t$.
Hence $\mathcal{R}_T=\sum_{t=1}^T L_t$ counts the PAHF policy's mistakes.

Decompose rounds into unambiguous and ambiguous.
\begin{itemize}[leftmargin=1.2em]
\item \emph{Unambiguous rounds.}
With at most $K$ switches and an immediate post-action update upon the first error after each switch, Proposition~\ref{prop:post-necessary} yields $O(K)$ mistakes on unambiguous rounds.

\item \emph{Ambiguous rounds.}
On each ambiguous round, using up to $k$ balanced $m$-ary pre-queries reduces the error probability to at most $m^{-k}$ by Proposition~\ref{prop:pre-necessary} and Lemma~\ref{lem:halving}.
Summed over $\gamma T$ ambiguous rounds, this contributes $O(\gamma T m^{-k})$ mistakes.
\end{itemize}
Summing contributions gives $\mathbb{E}[\mathcal{R}_T]=O(K+\gamma T m^{-k})$.
Taking $k=\lceil \log_m T\rceil$ yields $\mathbb{E}[\mathcal{R}_T]=O(K+\gamma)$.
\end{proof}

\section{Limitations and Future work}

(i) Memory Architecture: While this work focuses on feedback mechanisms rather than novel memory designs, future research could integrate more sophisticated memory systems as drop-in backends to improve scalability and context retention. (ii) Noisy Feedback: Human feedback can be noisy and heterogeneous. While our framework performs basic filtering, it does not explicitly handle inconsistent or mistaken feedback. For instance, a user might be confused and provide incorrect feedback; future work could develop mechanisms to detect conflicts, query the user for clarification, and model diverse noise patterns. (iii) Reasoning Capabilities: Our benchmarks, especially the online shopping domain, remain challenging: agents still struggle to achieve high success rates. We intentionally make the tasks difficult and restrict human feedback to sparse signals, in order to stress-test fine-grained preference reasoning. However, sufficiently capable agents should be able to piece together these limited signals to recover the underlying preferences. We hope future work will close this gap and further improve performance on our benchmarks. (iv) Limited Interaction for Disambiguation: To reflect user friction in real deployments, we limit the agent to at most one clarification question per ambiguous task in both domains. This single-turn interaction budget increases task difficulty, particularly in online shopping where identifying an item that satisfies multiple preference constraints can require sequential disambiguation. Future work could explore multi-turn clarification strategies and learned turn-allocation policies, and quantify the trade-off between performance gains and user burden.

\section{Details on Dataset}
\label{sec:add_dataset}

\subsection{Embodied Manipulation}

To support the reproducibility of our experiments and provide a comprehensive understanding of the benchmark, we provide a description of the dataset construction, persona logic, and scenario generation process below.

\textbf{User Persona Construction.}
The foundation of the benchmark lies in its set of 40 distinct user personas, each governed by a rigorous, multidimensional preference model. Unlike simple key-value pairings found in traditional datasets, these personas are defined by conditional logic where a user’s preference is a function of their latent context. The preference model spans nine distinct categories: Drinks, Snacks, Storage, Location, Temperature, Environmental Approach, Health Considerations, Social Context, and Time of Day.

For every category, specific rules dictate the user's intent based on their current state. For example, the persona "Alex" generally prefers black coffee but shifts to herbal tea when the context is "Drowsy" (believing rest cures fatigue) and ice-cold water when "Dehydrated". This structure ensures that agents cannot simply memorize a static mapping (e.g., "Alex likes coffee") but must learn the underlying causal relationship between the user's state and their desired action.

\textbf{Systematic Preference Drift (The "Evolved" Dataset).}
To rigorously evaluate continual learning and adaptation, we constructed a parallel "Evolved" version of every persona. This dataset introduces non-stationarity through a systematic, 1-to-1 inversion of the user's underlying belief system. In the "Evolved" personas, the logic driving preferences is flipped, rendering the agent’s prior knowledge obsolete or detrimental.

For instance, while the original "Alex" persona preferred herbal tea when drowsy to prioritize rest, the evolved "Alex" adopts the opposite philosophy, preferring energy drinks to prioritize stimulation. Similarly, storage preferences shift from "high shelves for hygiene" to "low shelves for accessibility." This "hard" distribution shift forces the agent to unlearn outdated correlations and acquire new ones through feedback, rather than simply refining existing weights.

\textbf{Scenario Generation and Structure.}
We generate interaction scenarios for training and evaluation (Sets A and B) in both the Original and Evolved phases. Each example contains: (i) Scene objects (the target plus semantically related distractors), (ii) an intentionally ambiguous natural-language task that underspecifies the correct item/location, (iii) a context tag that explicitly identifies the specific condition embedded within the instruction (e.g., ‘Drowsy’), and (iv) the ground-truth user intent (object and location) obtained by applying the persona’s rules to the context. This structure enables precise computation of personalization error (mismatch between the agent’s choice and the persona-consistent intent) and isolates the ability to adapt when preferences shift in the Evolved phase.

\subsection{Online Shopping}

\textbf{Product Ontology and Persona Construction.} The domain is structured around 10 distinct product categories: TVs, Laptops, Smartphones, Refrigerators, Washing Machines, Microwave Ovens, Air Conditioners, Dishwashers, Cameras, and Headphones. Each category is defined by a rigid ontology of 3 specific feature dimensions. For example, a "TV" is defined by its Smart TV Operating System, Panel Technology, and Base Type, while a "Laptop" is defined by its Form Factor, Charging Adapter Type, and Webcam Placement. For every feature, the ontology provides a fixed set of possible attribute values (e.g., Panel Technology options include OLED, QD-OLED, VA LCD, IPS LCD, MicroLED, and TN LCD).

User personas are constructed by assigning hierarchical preference rankings to these attributes. Unlike binary like/dislike flags, our personas model realistic consumer behavior through a tiered preference system: "Preferred" (the specific attribute value the user actively seeks), "Acceptable" (alternative values the user will tolerate), and "Disliked" (values explicitly rejected). This hierarchy allows for nuanced ground-truth determination, where the "correct" product is the one that maximizes overlap with the user's "Preferred" and "Acceptable" sets while avoiding "Disliked" features.

\textbf{Stochastic Preference Drift.} While the Embodied dataset utilizes logical inversion to simulate drift, the Online Shopping benchmark simulates stochastic taste evolution. In the "Evolved" phases (Phases 3 and 4), each user's preference profile is re-sampled from the feature ontology. Specifically, for every product feature, the set of attributes is reshuffled: attributes that were previously "Preferred" may become "Disliked," and previously "Disliked" attributes may become the new "Preferred." This simulates a user completely changing their requirements—for example, a user who previously sought "OLED TVs for gaming" might evolve to prefer "IPS LCD panels for bright rooms." 

\textbf{Scenario Generation and Evaluation.} The dataset contains interaction scenarios divided across four phases. Phases 1 and 2 utilize the Original preferences, while Phases 3 and 4 utilize the Evolved preferences. Uniquely, the scenarios in Phase 3 are structurally identical to Phase 1 (same product options and natural language instructions) but possess different ground-truth labels. Each scenario consists of: (i) a natural language user instruction derived from varied templates (e.g., "Help me buy a laptop that suits my preferences"), (ii) a set of generated product candidates fully described by their attributes, and (iii) the ground-truth optimal product selected by evaluating the user’s current preference hierarchy against the options. This "counterfactual" design, where the inputs remain constant between Phase 1 and Phase 3, but the correct answer changes, allows us to rigorously isolate the agent's sensitivity to preference drift.

\section{Prompts}

\begin{tcolorbox}[colback=white,colframe=gray!50!black,title=Environment Prompt for the Embodied Agent,breakable]

You are simulating a robot operating in an office environment with access to multiple rooms and areas:\\
\\
COUNTER AREAS:\\
- Main counter (in front of you): Open surface for immediate access\\
- Side counter: Secondary workspace area\\
- Island counter: Central area for shared items\\
\\
STORAGE LOCATIONS:\\
- Top drawer: Upper storage compartment\\
- Middle drawer: Middle storage compartment\\
- Bottom drawer: Lower storage compartment\\
- Top shelf: Upper open shelf storage\\
- Bottom shelf: Lower open shelf storage\\
- Refrigerator top shelf: Upper refrigerator compartment\\
- Refrigerator bottom shelf: Lower refrigerator compartment\\
- Pantry cabinet: Enclosed dry storage area\\
\\
WASTE/RECYCLING:\\
- Landfill bin: General waste disposal\\
- Recycling bin: Recyclable materials\\
- Compost bin: Organic waste disposal\\
\\
SPECIALIZED AREAS:\\
- Personal desk area: Individual workspace nearby\\
- Meeting table: Shared collaborative space\\
- Coffee station: Dedicated beverage preparation area\\
- Microwave area: Heating/cooking zone\\
\\
DINING/SEATING AREAS:\\
- Couch: Comfortable seating area for relaxation\\
- Dining table: Formal eating and dining surface\\
\\
BEDROOM AREAS:\\
- Bedroom nightstand: Personal bedside storage surface\\
- Bedroom dresser: Clothing and personal item storage\\
- Bedroom closet: Hanging and folded clothing storage\\
- Bedroom closet door: Door surface for hanging items\\
\\
The robot can pick up items, place items in any of these locations, or retrieve items from these locations based on user preferences and contextual needs.

\end{tcolorbox}

\begin{tcolorbox}[colback=white,colframe=gray!50!black,title=ReAct Prompt for the Embodied Agent,breakable]

TASK INSTRUCTIONS:\\
You must determine the next action to complete the given task.\\
- If the task is clear, select the appropriate action from the options.\\
- If the task involves personalization, search memory for relevant personalized information.\\
- If the task is unclear, ask the user for clarification.\\
\\
WHEN TO ASK FOR CLARIFICATION:\\
You MUST ask for clarification when:\\
1. The task contains ambiguous references (e.g., "that drink" when multiple drinks are present)\\
2. The task involves subjective preferences (e.g., "favorite", "preferred", "best", "my", "like")\\
3. The task doesn't clearly specify which object to interact with\\
4. The task is unclear or incomplete\\
\\
FORMAT REQUIREMENTS:\\
Your response must ALWAYS follow this exact format:\\
\\
Thought: <detailed reasoning explaining your decision process>\\
Action: <EITHER a specific option letter (A, B, C, D), "Ask human">\\
\\
IMPORTANT RULES:\\
- Never skip any of the three required fields\\
- For Action, only use option letters (A, B, C, D) or the exact phrase "Ask human"\\
- Always provide detailed reasoning in your Thought field\\
- After receiving clarification from a human (marked with "Human:"), select the appropriate action option\\
- You can assume one of the option letters (A, B, C, D) is correct. Asking for clarification is ONLY necessary when you cannot determine which option is correct due to ambiguity or missing information.\\
\\
EXAMPLES:\\
\\
Example 1:\\
Scene: On the counter, there is an orange soda, a Pepsi, and an apple.\\
Alex: Put that drink in the top drawer.\\
Options:\\
A) open the top drawer and put the orange soda in it\\
B) open the bottom drawer and put the Pepsi in it\\
C) open the bottom drawer and put the orange soda in it\\
D) open the top drawer and put the Pepsi in it\\
\\
Context: No relevant personalized information found in memory.\\
\\
Thought: The task references "that drink," but two drinks are present (orange soda and Pepsi). I cannot determine which specific drink is intended, making this ambiguous.\\
Action: Ask human\\
\\
Question: Which drink do you want me to put in the top drawer? orange soda or Pepsi?\\
Human: Pepsi\\
\\
Thought: Now I know the human wants me to put the Pepsi in the top drawer. Looking at the options, option D matches this requirement exactly: "open the top drawer and put the Pepsi in it".\\
Action: D\\
\\
Example 2:\\
Scene: On the counter, there is a Coke, a bottled unsweetened tea, and a Sprite.\\
Sam: Bring me my favorite drink.\\
Options:\\
A) pick up the Coke\\
B) pick up the bottled unsweetened tea\\
C) pick up the Sprite\\
D) pick up the bottled water\\
\\
Context: No relevant personalized information found in memory.\\
\\
Thought: I need to ask human for clarification.\\
Action: Ask human\\
\\
Question: Which drink is your favorite? Coke, bottled unsweetened tea, or Sprite?\\
Human: Sprite\\
\\
Thought: The human has clarified that Sprite is their favorite drink. Option C is "pick up the Sprite", which directly matches.\\
Action: C\\
\\
Example 3:\\
Scene: On the counter, there is a bottled water, a bag of jalapeno chips, and a bag of rice chips.\\
Sam: I would like a bag of rice chips.\\
Options:\\
A) pick up the bottled water\\
B) pick up the jalapeno chips\\
C) pick up the kettle chips\\
D) pick up the rice chips\\
\\
Context: No relevant personalized information found in memory.\\
\\
Thought: It's clear that I should pick up the rice chips.\\
Action: D\\
\\
Example 4:\\
Scene: On the counter, there is a Coke, a bottled green tea, and a Sprite.\\
Sam: Bring me my favorite drink.\\
Options:\\
A) pick up the Coke\\
B) pick up the bottled green tea\\
C) pick up the Sprite\\
D) pick up the orange juice\\
\\
Context: The human's favorite drink is Coke.\\
\\
Thought: Based on the context, I know the human's favorite drink is Coke. The task asks me to bring their favorite drink, so I should pick up the Coke.\\
Action: A
\end{tcolorbox}

\begin{tcolorbox}[colback=white,colframe=gray!50!black,title=Summarization Prompt for the Embodied Agent (Same Format as Shopping),breakable]

TASK INSTRUCTIONS:\\
Summarize the information retrieved from memory that is relevant to the task. If none of the retrieved information is relevant, output `No relevant personalized information found in memory.`\\
\\
FORMAT REQUIREMENTS:\\
Your response must ONLY contain the summary in this EXACT format:\\
\\
Summary: <your summary here>\\
\\
EXAMPLES:\\
\\
Example 1:\\
Scene: On the counter, there is an orange soda, a Pepsi, and an apple.\\
Kate: Put that drink in the top drawer.\\
Options:\\
A) open the top drawer and put the orange soda in it\\
B) open the bottom drawer and put the Pepsi in it\\
C) open the bottom drawer and put the orange soda in it\\
D) open the top drawer and put the Pepsi in it\\
\\
Context:\\
None\\
\\
Summary: No relevant personalized information found in memory\\
\\
Example 2:\\
Scene: On the counter, there is an energy bar, a banana, and a microwave.\\
Kate: Put my favourite snack next to the microwave.\\
Options:\\
A) pick up the energy bar and put it next to the microwave\\
B) pick up the banana and put it next to the energy bar\\
C) pick up the banana and put it next to the microwave\\
D) pick up the energy bar and put it next to the banana\\
\\
Context:\\
- Kate's favorite snack is energy bar.\\
\\
Summary: Kate's favorite snack is energy bar.\\
\\
Example 3:\\
Scene: There is a paper cup, an empty can, and a banana peel on the counter.\\
Kate: Throw away my usual drink container in the correct bin.\\
Options:\\
A) put the paper cup in the recycling bin\\
B) put the empty can in the recycling bin\\
C) put the banana peel in the compost bin\\
D) put the empty can in the landfill bin\\
\\
Context:\\
- Kate's usual drink is canned soda.\\
- Kate prefers to recycle whenever possible.\\
- Kate never composts drink containers.\\
\\
Summary: Kate’s usual drink is canned soda, and he prefers to recycle drink containers.

\end{tcolorbox}

\begin{tcolorbox}[colback=white,colframe=gray!50!black,title=Salience Detector Prompt,breakable]
Human feedback: \{feedback\}.\\

Does the feedback contain any personalized information? Answer a single word 'Yes' or 'No'.
\end{tcolorbox}

\begin{tcolorbox}[colback=white,colframe=gray!50!black,title=Preference Drift Detector Prompt,breakable]

Does this feedback indicate a CHANGE in a previously stated preference (i.e., the user used to prefer X, but now prefers Y or no longer prefers X)?\\

Look for phrases indicating a transition:\\
- ``but now...'' / ``now I prefer...''\\
- ``used to like X, but...''\\
- ``changed my mind about...''\\
- ``instead of X, I prefer Y''\\
- ``no longer like...''\\
\\
If it's only adding new information or conditional preferences (e.g., ``I like X when tired''), answer No.\\
\\
Answer a single word: Yes or No.

\end{tcolorbox}

\begin{tcolorbox}[colback=white,colframe=gray!50!black,title=Memory Integration Prompt,breakable]

Please create a concise, integrated summary that combines the following information:\\
\\
Existing memory: \{existing\_memory\}\\
New information: \{summ\_info\}\\
\\
Provide a single, coherent summary that incorporates both pieces of information without redundancy.

\end{tcolorbox}

\begin{tcolorbox}[colback=white,colframe=gray!50!black,title=Environment Prompt for the Shopping Agent,breakable]

You are an AI shopping agent that helps users buy products in an online shopping system.\\
\\
PRODUCT CATEGORIES:\\
- Headphones (wearing style, acoustic principle, connectivity mode, control interface, audio environment feature)\\
- TVs (smart OS, panel technology, base type, tuner standard, speaker orientation)\\
- Refrigerators (door configuration, cooling architecture, storage organization, ice/water system, user interface)\\
- Cameras (image sensor, lens mount, viewfinding method, shutter mechanism, storage medium)\\
- Dishwashers (installation format, water distribution, drying technology, soil filtration, detergent delivery)\\
- Air Conditioners (cooling mechanism, installation configuration, airflow distribution, control interface, refrigerant)\\
- Washing Machines (loading mechanism, washing motion, control interface, water supply, drying method)\\
- Microwave Ovens (heating mechanism, door configuration, control interface, cavity material, cooking control)\\
- Smartphones (interaction modality, biometric authentication, camera modules, battery architecture, form factor)\\
- Laptops (form factor, charging adapter, webcam placement, keyboard technology, speaker placement)\\
\\
The agent can select products based on user preferences or choose not to buy if no option satisfies requirements.\\

\end{tcolorbox}

\begin{tcolorbox}[colback=white,colframe=gray!50!black,title=ReAct Prompt for the Shopping Agent,breakable]

TASK INSTRUCTIONS:\\
You must determine the next action to complete the given purchase task.\\
- If the task is clear and you have preference information, select the appropriate product option.\\
- If the task involves personalization, search memory for relevant preference information.\\
- If preferences are unclear or missing, ask the user for clarification.\\
\\
WHEN TO ASK FOR CLARIFICATION:\\
You MUST ask for clarification when:\\
1. The task involves subjective preferences\\
2. The task doesn't provide enough information to determine which product features the user wants\\
3. You don't have relevant preference information in memory to make the decision\\
4. The task is unclear or ambiguous\\
\\
CRITICAL PURCHASING RULE --- ALL FEATURES REQUIREMENT:\\
You MUST buy a product (select A, B, or C) ONLY if ALL features listed in the option match the user's preferred/acceptable features.\\
- Each option lists multiple features --- evaluate every feature that is explicitly stated\\
- Count how many listed features in each option match the user's stated or remembered preferences\\
- If an option has ALL listed features matching $\rightarrow$ SELECT THAT OPTION\\
- If NO option has all listed features matching $\rightarrow$ SELECT OPTION D (Do not buy anything)\\
- Note: Some examples below may show fewer features for brevity; in the benchmark, evaluate all features provided in each option (typically 3), and ALL must match\\
\\
WHEN TO SELECT ``D) Do not buy anything'':\\
You should select option D when:\\
- None of the available product options (A, B, or C) have ALL features that match the user's preferences\\
- All available options contain features that contradict the user's known preferences\\
- The user's requirements cannot be satisfied by any of the presented options\\
\\
FORMAT REQUIREMENTS:\\
Your response must ALWAYS follow this exact format:\\
\\
Thought: <detailed reasoning explaining your decision process>\\
Action: <EITHER a specific option letter (A, B, C, D), OR ``Ask human''>\\
\\
IMPORTANT RULES:\\
- For Action, only use option letters (A, B, C, D) or the exact phrase ``Ask human''\\
- Always provide detailed reasoning in your Thought field\\
- After receiving clarification from a human (marked with ``Human:''), select the appropriate action option\\
- Remember that option D means ``Do not buy anything'' --- use it when no product matches user preferences\\
- CRITICAL: You can ask the human AT MOST ONE QUESTION per task. After receiving one answer from the human, you MUST make a decision and select an action (A, B, C, or D). No further questions are allowed.\\
\\
EXAMPLES:\\
\\
Example 1:\\
Emma: I'd like to get headphones that match my style\\
Options:\\
A) headphones with bluetooth link and bone-conduction frame\\
B) headphones with wired connection and ear-hook in-ear design\\
C) headphones with 2.4 GHz low-latency dongle and in-ear design\\
D) Do not buy anything\\
\\
Context: No relevant personalized information found in memory.\\
\\
Thought: The task asks for headphones that match Emma's style, but I don't know what her preferred features are. Without this information, I cannot determine which option (A, B, C, or D) is correct.\\
Action: Ask human\\
\\
Question: What connectivity mode do you prefer for headphones? (bluetooth, wired, 2.4 GHz dongle, or RF transmitter)\\
Human: I prefer RF transmitter connectivity\\
\\
Thought: Now I know Emma prefers RF transmitter connectivity. Looking at the options, none of them have RF transmitter connectivity. Option A has bluetooth, Option B has wired, and Option C has 2.4 GHz dongle. Since none match her preferred RF transmitter connectivity, I should select option D to not buy anything.\\
Action: D\\
\\
Example 2:\\
Liam: Help me purchase a TV that I'd like\\
Options:\\
A) TV with 55-inch OLED panel and webOS platform\\
B) TV with 65-inch LCD panel and Roku TV platform\\
C) TV with 48-inch QLED panel and Fire TV platform\\
D) Do not buy anything\\
\\
Context: Liam prefers 55-inch OLED panel for TVs.\\
\\
Thought: I need to ask about Liam's TV preferences about platform to make the right choice.\\
Action: Ask human\\
\\
Question: What smart TV platform do you prefer? (webOS, Roku TV, Fire TV, or other)\\
Human: I prefer webOS platform\\
\\
Thought: Liam has clarified that he prefers webOS platform. Option A has both OLED panel and webOS platform, which matches his stated preference. This is the best match among the available options.\\
Action: A\\
\\
Example 3:\\
Sophia: I'm looking for a refrigerator that suits my needs\\
Options:\\
A) refrigerator with french-door layout and inverter compressor loop\\
B) refrigerator with side-by-side structure and dual-evaporator inverter system\\
C) refrigerator with top-freezer layout and absorption cooling system\\
D) Do not buy anything\\
\\
Context: Sophia prefers top-freezer layout and inverter compressor loop for refrigerators.\\
\\
Thought: Based on the context, Sophia prefers top-freezer layout and inverter compressor loop. Looking at the options: Option A has the inverter compressor loop but uses french-door layout instead of top-freezer. Option C has the top-freezer layout but uses absorption cooling instead of inverter compressor. Option B doesn't match either preference. None of the options provide both features Sophia is looking for, so I should select option D.\\
Action: D\\
\\
Example 4:\\
Oliver: Help me choose a washing machine that works for me\\
Options:\\
A) washing machine with front-load design and steam cycle feature\\
B) washing machine with top-load design and steam cycle feature\\
C) washing machine with front-load design and bleach dispenser\\
D) Do not buy anything\\
\\
Context: Oliver prefers front-load design and steam cycle feature for washing machines.\\
\\
Thought: Oliver is looking for front-load design and steam cycle feature. Option A has both the front-load design and steam cycle feature that Oliver wants. Option B only has the steam cycle but uses top-load instead. Option C has the front-load design but lacks the steam cycle. Option A satisfies both of Oliver's requirements.\\
Action: A\\
\\
Example 5:\\
Noah: Help me get a camera that fits my preferences\\
Options:\\
A) camera with backside-illuminated CMOS sensor and Canon EF mount\\
B) camera with Foveon X3 sensor and Sony E mount\\
C) camera with front-side illuminated CMOS sensor and Nikon F mount\\
D) Do not buy anything\\
\\
Context: Noah prefers organic photoconductive film sensor and Canon EF mount for cameras.\\
\\
Thought: Noah is looking for organic photoconductive film sensor and Canon EF mount. Option A has the Canon EF mount but uses backside-illuminated CMOS sensor instead of the organic photoconductive film sensor Noah prefers. Options B and C don't have either feature Noah wants. Since none of the options have both the sensor type and mount that Noah prefers, I should select option D.\\
Action: D
\end{tcolorbox}

\section{Additional Results}
\label{sec:add_results}

\subsection{Qualitative Results}

\begin{tcolorbox}[colback=white,colframe=gray!50!black,title=Personalization Evolution through PAHF,breakable]

\blue{Phase 1: Initial Preference Learning}\\
\\
Scene: a herbal tea, a coffee, an iced tea, rice cakes.\\
Avery: Could you bring me my favourite drink?\\
Options:\\
A) pick up the herbal tea\\
B) pick up the coffee\\
C) pick up the iced tea\\
D) pick up the rice cakes\\
\\
Context: No relevant personalized information found in memory\\
\\
Thought: The task requires identifying Avery's favorite drink, but no relevant personalized information is available in memory to determine what their favorite drink is. This makes the task ambiguous as there are multiple drinks present and no specific preference is clarified.\\
Action: Ask human\\
\\
Question: Which drink is your favorite? herbal tea, coffee, or iced tea?\\
Human: Herbal tea, please.\\
\\
Thought: The human has clarified that their favorite drink is herbal tea. Option A corresponds to picking up the herbal tea, which matches the clarified preference.\\
Action: A\\
\\
Human: Thank you, that's exactly right!\\
\\
\blue{Phase 2: Preference Evaluation}\\
\\
Scene: a herbal tea, a coffee, an iced beverage, rice cakes.\\
Avery: I'd like my favorite drink, please.\\
Options:\\
A) pick up the herbal tea\\
B) pick up the coffee\\
C) pick up the iced beverage\\
D) pick up the rice cakes\\
\\
Context: Avery prefers herbal tea as their favorite morning drink.\\
\\
Thought: Based on the context, I know that Avery's favorite drink is herbal tea, and the task is to bring their favorite drink. Option A is the valid choice.\\
Action: A\\
\\
\blue{Phase 3: Preference Drift Adaptation}\\
\\
Scene: a herbal tea, a coffee, an iced tea, rice cakes.\\
Avery: Could you bring me my favourite drink?\\
Options:\\
A) pick up the herbal tea\\
B) pick up the coffee\\
C) pick up the iced tea\\
D) pick up the rice cakes\\
\\
Context: Avery prefers herbal tea as their favorite morning drink.\\
\\
Thought: The context specifies that Avery prefers herbal tea as their favorite morning drink. The task asks to bring Avery their favorite drink, so I should pick up the herbal tea. Option A matches this requirement exactly.\\
Action: A\\
\\
Human: You might not know that I changed my previous preference for herbal tea, but now I prefer coffee. Could you bring me the coffee instead?\\
\\
\blue{Phase 4: Post-Drift Evaluation}\\
\\
Scene: a herbal tea, a coffee, an iced beverage, rice cakes.\\
Avery: I'd like my favorite drink, please.\\
Options:\\
A) pick up the herbal tea\\
B) pick up the coffee\\
C) pick up the iced beverage\\
D) pick up the rice cakes\\
\\
Context: Avery's favorite is now coffee.\\
\\
Thought: The task is to bring Avery's favorite drink. Based on the provided context, I know that Avery's favorite drink is coffee. Therefore, I should pick up the coffee to complete the task.\\
Action: B

\end{tcolorbox}

\begin{tcolorbox}[colback=white,colframe=gray!50!black,title=Pre-action Only: Pre-action feedback alone is brittle under preference drift,breakable]

\blue{Phase 1: Initial Preference Learning}\\
\\
Scene: a herbal tea, a coffee, an iced tea, rice cakes.\\
Avery: Could you bring me my favourite drink?\\
Options:\\
A) pick up the herbal tea\\
B) pick up the coffee\\
C) pick up the iced tea\\
D) pick up the rice cakes\\
\\
Context: No relevant personalized information found in memory\\
\\
Thought: The task asks me to bring the human's favorite drink, but their favorite drink is not specified in the request, and no relevant personalized information is found in memory. Therefore, I cannot determine which drink is their favorite and must seek clarification.\\
Action: Ask human\\
\\
Question: Which drink is your favorite? herbal tea, coffee, or iced tea?\\
Human: Herbal tea, please.\\
\\
Thought: The human has clarified that their favorite drink is herbal tea. The task is to bring their favorite drink. Looking at the options, option A is ``pick up the herbal tea,'' which matches the clarified preferences.\\
Action: A\\
\\
\blue{Phase 2: Preference Evaluation}\\
\\
Scene: a herbal tea, a coffee, an iced beverage, rice cakes.\\
Avery: I'd like my favorite drink, please.\\
Options:\\
A) pick up the herbal tea\\
B) pick up the coffee\\
C) pick up the iced beverage\\
D) pick up the rice cakes\\
\\
Context: Avery's preferred drink is herbal tea.\\
\\
Thought: Based on the context, Avery's favorite drink is herbal tea. The task specifically asks for their favorite drink, so I should select the herbal tea.\\
Action: A\\
\\
\blue{Phase 3: Preference Drift Adaptation}\\
\\
Scene: a herbal tea, a coffee, an iced tea, rice cakes.\\
Avery: Could you bring me my favourite drink?\\
Options:\\
A) pick up the herbal tea\\
B) pick up the coffee\\
C) pick up the iced tea\\
D) pick up the rice cakes\\
\\
Context: Avery's preferred drink is herbal tea.\\
\\
Thought: Based on the context, Avery's favorite drink is herbal tea. The task is to bring their favorite drink, so I should pick up the herbal tea. Option A matches this requirement: ``pick up the herbal tea.''\\
Action: A\\
\\
\blue{Phase 4: Post-Drift Evaluation}\\
\\
Scene: a herbal tea, a coffee, an iced beverage, rice cakes.\\
Avery: I'd like my favorite drink, please.\\
Options:\\
A) pick up the herbal tea\\
B) pick up the coffee\\
C) pick up the iced beverage\\
D) pick up the rice cakes\\
\\
Context: Avery's preferred drink is herbal tea.\\
\\
Thought: Based on the context, I know that Avery's preferred drink is herbal tea. The task asks me to bring their favorite drink, so I should pick up the herbal tea.\\
Action: \red{A}\\

\end{tcolorbox}

\begin{tcolorbox}[colback=white,colframe=gray!50!black,title=Post-action only: Post-action feedback alone can lead to large initial personalization error,breakable]

\blue{Phase 1: Initial Preference Learning}\\
\\
Scene: a herbal tea, a coffee, an iced tea, rice cakes.\\
Avery: Could you bring me my favourite drink?\\
Options:\\
A) pick up the herbal tea\\
B) pick up the coffee\\
C) pick up the iced tea\\
D) pick up the rice cakes\\
\\
Context: No relevant personalized information found in memory\\
\\
Thought: The task asks for ``my favourite drink,'' but there is no specific information in memory about Avery's preference. Since coffee is a commonly preferred drink, selecting it seems like a reasonable choice in the absence of other context.\\
Action: \red{B}\\
\\
Human: That's not my favorite drink. I prefer the herbal tea in this case---please bring that instead.\\
\\
\blue{Phase 2: Preference Evaluation}\\
\\
Scene: a herbal tea, a coffee, an iced beverage, rice cakes.\\
Avery: I'd like my favorite drink, please.\\
Options:\\
A) pick up the herbal tea\\
B) pick up the coffee\\
C) pick up the iced beverage\\
D) pick up the rice cakes\\
\\
Context: Avery's favorite drink is herbal tea, but they prefer iced tea specifically when feeling thirsty.\\
\\
Thought: The task requests Avery's favorite drink. Context reveals that herbal tea is Avery's favorite drink. Although Avery prefers iced tea when feeling thirsty, there is no mention of thirst in this task. Based on the given context, I will select the herbal tea as their favorite drink.\\
Action: A\\
\\
\blue{Phase 3: Preference Drift Adaptation}\\
\\
Scene: a herbal tea, a coffee, an iced tea, rice cakes.\\
Avery: Could you bring me my favourite drink?\\
Options:\\
A) pick up the herbal tea\\
B) pick up the coffee\\
C) pick up the iced tea\\
D) pick up the rice cakes\\
\\
Context: Avery's favorite drink is herbal tea, but they prefer iced tea specifically when feeling thirsty.\\
\\
Thought: The task asks to bring Avery's favorite drink. Based on the context, Avery's favorite drink is herbal tea, but they prefer iced tea specifically when feeling thirsty. Since no indication of thirst is provided, I'll assume it's a general request and choose herbal tea as their favorite drink.\\
Action: A\\
\\
Human: You might not know that I’ve completely flipped my preferences---herbal tea used to be my favorite, but now I prefer coffee. Could you bring me the coffee instead?\\
\\
\blue{Phase 4: Post-Drift Evaluation}\\
\\
Scene: a herbal tea, a coffee, an iced beverage, rice cakes.\\
Avery: I'd like my favorite drink, please.\\
Options:\\
A) pick up the herbal tea\\
B) pick up the coffee\\
C) pick up the iced beverage\\
D) pick up the rice cakes\\
\\
Context: Avery's favorite beverage is coffee.\\
\\
Thought: The task asks for Avery's favorite drink. Based on the context, I know Avery's favorite beverage is coffee. Therefore, I should pick up the coffee.\\
Action: B

\end{tcolorbox}

\subsection{Quantitative Results}

\begin{table}[h!]
\centering
\caption{\textbf{Evaluation success rates (\%).} Phase 2 and 4 performance across different agent-simulator pairs (Memory: SQLite).}
\small
\setlength{\tabcolsep}{3.5pt}

\begin{tabular}{lcccccccc}
\toprule
\multirow{3}{*}{Method}
& \multicolumn{4}{c}{GPT-4.1 Agent; GPT-4o Human}
& \multicolumn{4}{c}{GPT-4o Agent; GPT-4.1 Human} \\
\cmidrule(lr){2-5}\cmidrule(lr){6-9}
& \multicolumn{2}{c}{Embodied} & \multicolumn{2}{c}{Shopping}
& \multicolumn{2}{c}{Embodied} & \multicolumn{2}{c}{Shopping} \\
\cmidrule(lr){2-3}\cmidrule(lr){4-5}\cmidrule(lr){6-7}\cmidrule(lr){8-9}
& Phase 2 & Phase 4 & Phase 2 & Phase 4 & Phase 2 & Phase 4 & Phase 2 & Phase 4 \\
\midrule
No memory
& $32.7$ & $44.8$ & $24.6$ & $14.9$
& $31.4$ & $42.7$ & $28.0$ & $26.3$ \\
Pre-action only
& $52.8$ & $35.1$ & $35.8$ & $59.0$
& $57.8$ & $30.4$ & $37.2$ & $55.0$ \\
Post-action only
& $81.3$ & $77.1$ & $36.0$ & $66.6$
& $66.3$ & $67.8$ & $39.1$ & $66.2$ \\
\textbf{PAHF (pre+post)}
& $\mathbf{82.3}$ & $\mathbf{78.8}$ & $\mathbf{38.3}$ & $\mathbf{67.4}$
& $\mathbf{73.5}$ & $\mathbf{68.3}$ & $\mathbf{39.3}$ & $\mathbf{67.9}$ \\
\bottomrule
\end{tabular}

\label{table:gpt41_gpt4o_and_swap}
\end{table}

\begin{table}[h!]
\centering
\caption{\textbf{Evaluation success rates (\%).} Phase 2 and 4 performance across memory backends and domains. Agent and Simulator: GPT-4o.}
\small
\setlength{\tabcolsep}{3.5pt}

\begin{tabular}{lcccccccc}
\toprule
\multirow{3}{*}{Method}
& \multicolumn{4}{c}{SQLite Memory}
& \multicolumn{4}{c}{FAISS Memory} \\
\cmidrule(lr){2-5}\cmidrule(lr){6-9}
& \multicolumn{2}{c}{Embodied} & \multicolumn{2}{c}{Shopping}
& \multicolumn{2}{c}{Embodied} & \multicolumn{2}{c}{Shopping} \\
\cmidrule(lr){2-3}\cmidrule(lr){4-5}\cmidrule(lr){6-7}\cmidrule(lr){8-9}
& Phase 2 & Phase 4 & Phase 2 & Phase 4 & Phase 2 & Phase 4 & Phase 2 & Phase 4 \\
\midrule
No memory
& $32.3_{\pm 0.4}$ & $44.8_{\pm 0.5}$ & $27.8_{\pm 0.2}$ & $27.0_{\pm 0.4}$
& $32.3_{\pm 0.4}$ & $44.8_{\pm 0.5}$ & $27.8_{\pm 0.2}$ & $27.0_{\pm 0.4}$ \\
Pre-action only
& $54.1_{\pm 1.1}$ & $35.7_{\pm 1.0}$ & $34.4_{\pm 0.5}$ & $56.0_{\pm 0.7}$
& $51.8_{\pm 1.3}$ & $36.2_{\pm 0.9}$ & $35.9_{\pm 0.6}$ & $57.2_{\pm 0.7}$ \\
Post-action only
& $67.9_{\pm 1.5}$ & $68.3_{\pm 1.2}$ & $38.9_{\pm 0.5}$ & $66.9_{\pm 0.8}$
& $66.9_{\pm 1.7}$ & $66.6_{\pm 1.4}$ & $39.2_{\pm 0.7}$ & $65.9_{\pm 0.9}$ \\
\textbf{PAHF (pre+post)}
& $\mathbf{70.5_{\pm 1.7}}$ & $\mathbf{68.8_{\pm 1.3}}$ & $\mathbf{41.3_{\pm 0.8}}$ & $\mathbf{70.3_{\pm 1.1}}$
& $\mathbf{68.0_{\pm 1.8}}$ & $\mathbf{68.2_{\pm 1.6}}$ & $\mathbf{39.4_{\pm 1.2}}$ & $\mathbf{70.8_{\pm 1.0}}$ \\
\bottomrule
\end{tabular}

\label{table:gpt41_gpt4o_and_swap}
\end{table}

\begin{figure*}[h!]
  \centering
    \begin{subfigure}[t]{0.315\linewidth}
        \includegraphics[width=\linewidth]{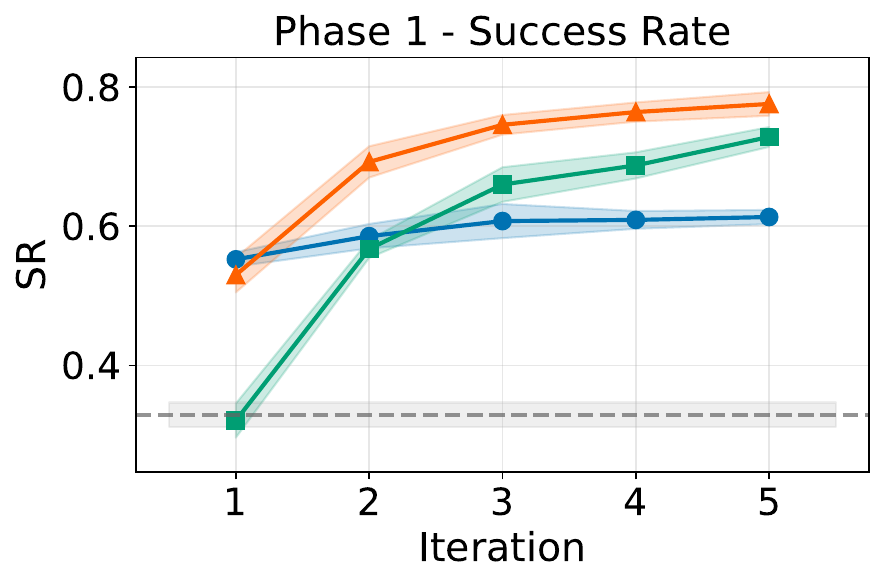}
        \label{fig:faiss_embodied_4_1_4_1_phase1_sr}
    \end{subfigure}    
    ~
    \begin{subfigure}[t]{0.315\linewidth}
        \includegraphics[width=\linewidth]{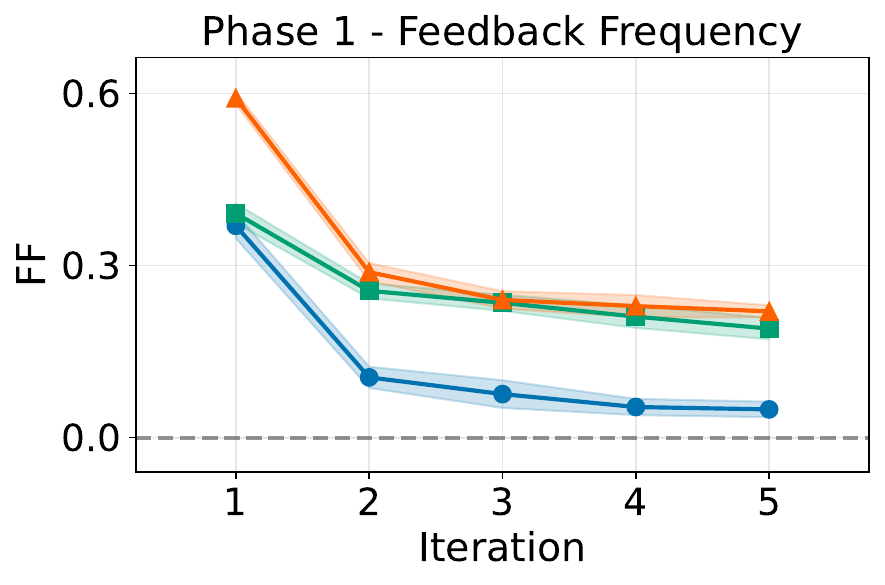}
        \label{fig:faiss_embodied_4_1_4_1_phase1_ff}
    \end{subfigure}
    ~
    \begin{subfigure}[t]{0.315\linewidth}
        \includegraphics[width=\linewidth]{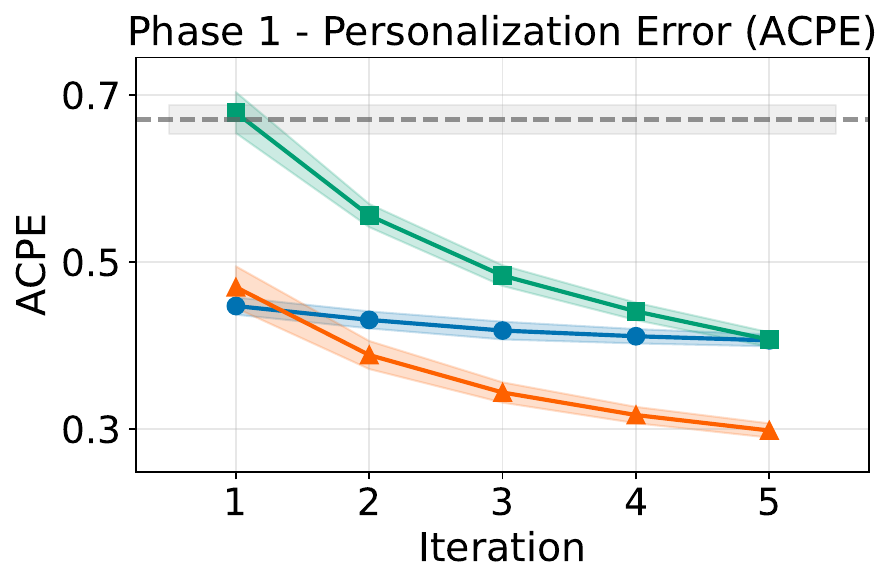}
        \label{fig:faiss_embodied_4_1_4_1_phase1_acpe}
    \end{subfigure}

    \vspace{-2mm}
    \begin{subfigure}[t]{0.315\linewidth}
        \includegraphics[width=\linewidth]{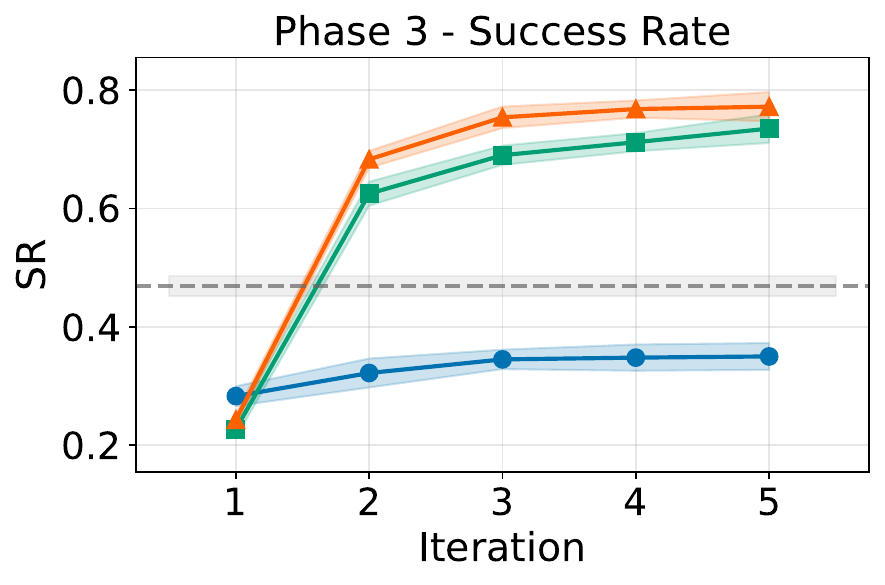}
        \label{fig:faiss_embodied_4_1_4_1_phase3_sr}
    \end{subfigure}    
    ~
    \begin{subfigure}[t]{0.315\linewidth}
        \includegraphics[width=\linewidth]{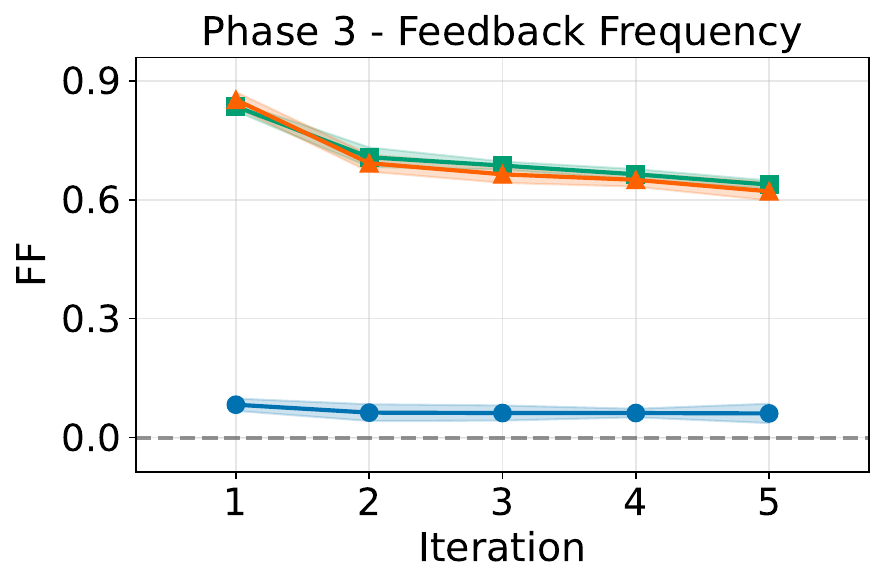}
        \label{fig:faiss_embodied_4_1_4_1_phase3_ff}
    \end{subfigure}
    ~
    \begin{subfigure}[t]{0.315\linewidth}
        \includegraphics[width=\linewidth]{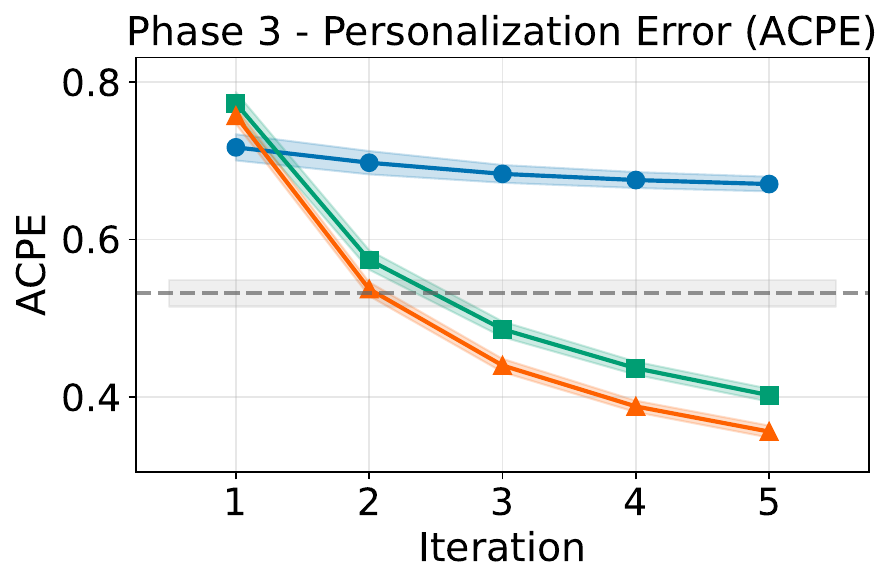}
        \label{fig:faiss_embodied_4_1_4_1_phase3_acpe}
    \end{subfigure}

    \par\vspace{-0.28cm}
    \begin{subfigure}[c]{\linewidth}
        \centering
        \includegraphics[width=0.65\linewidth]{figure/plots/legend.pdf}
    \end{subfigure}

    \caption{\textbf{Embodied manipulation results (Agent: GPT-4o, Human Simulator: GPT-4o, Memory: FAISS)}. Top row: Phase 1 learning curves for success rate (left), feedback frequency (middle), and average cumulative personalization error (ACPE) (right). Bottom row: Phase~3 learning curves for the same metrics. Shaded regions denote standard error.}
    \label{fig:faiss_embodied_gpt_4_1_gpt_4_1}

\end{figure*}

\begin{figure*}[h!]
  \centering
    \begin{subfigure}[t]{0.315\linewidth}
        \includegraphics[width=\linewidth]{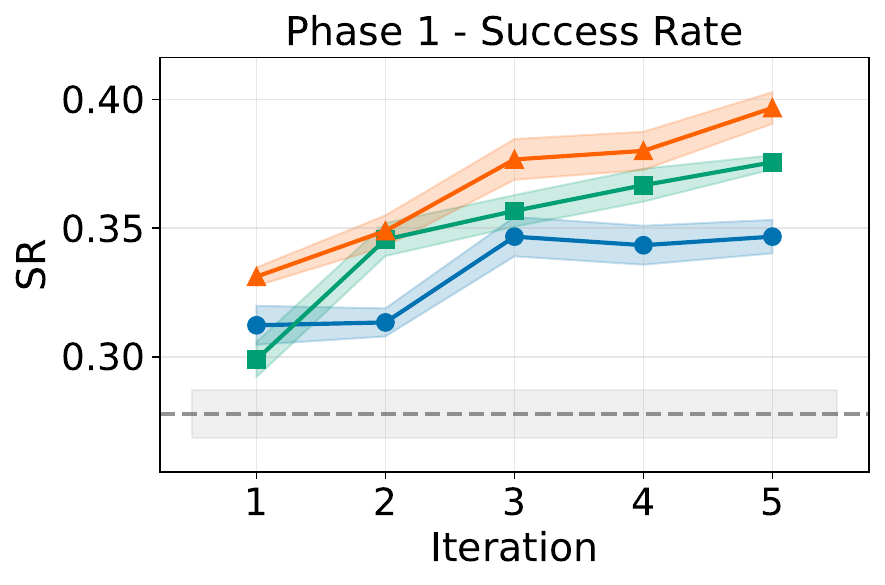}
        \label{fig:faiss_shopping_4_1_4_1_phase1_sr}
    \end{subfigure}    
    ~
    \begin{subfigure}[t]{0.315\linewidth}
        \includegraphics[width=\linewidth]{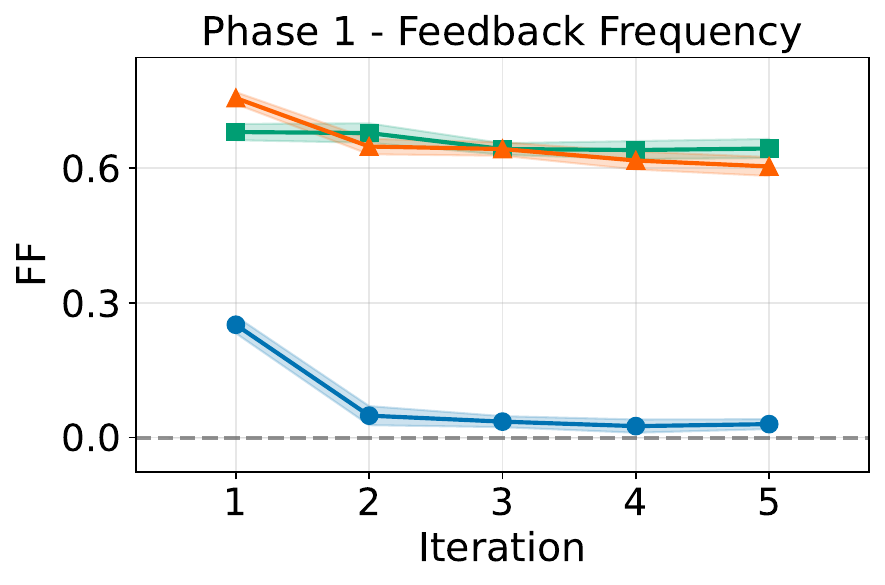}
        \label{fig:faiss_shopping_4_1_4_1_phase1_ff}
    \end{subfigure}
    ~
    \begin{subfigure}[t]{0.315\linewidth}
        \includegraphics[width=\linewidth]{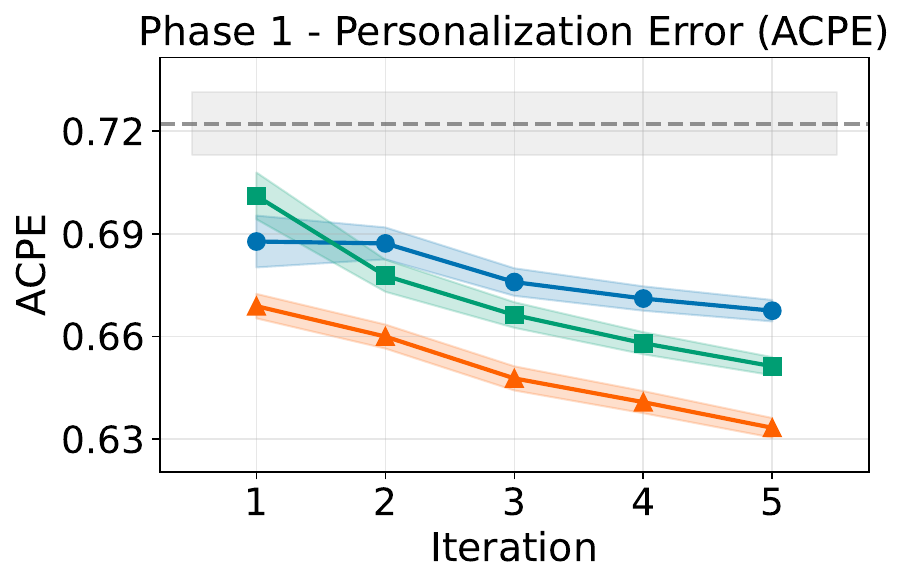}
        \label{fig:faiss_shopping_4_1_4_1_phase1_acpe}
    \end{subfigure}

    \vspace{-2mm}
    \begin{subfigure}[t]{0.315\linewidth}
        \includegraphics[width=\linewidth]{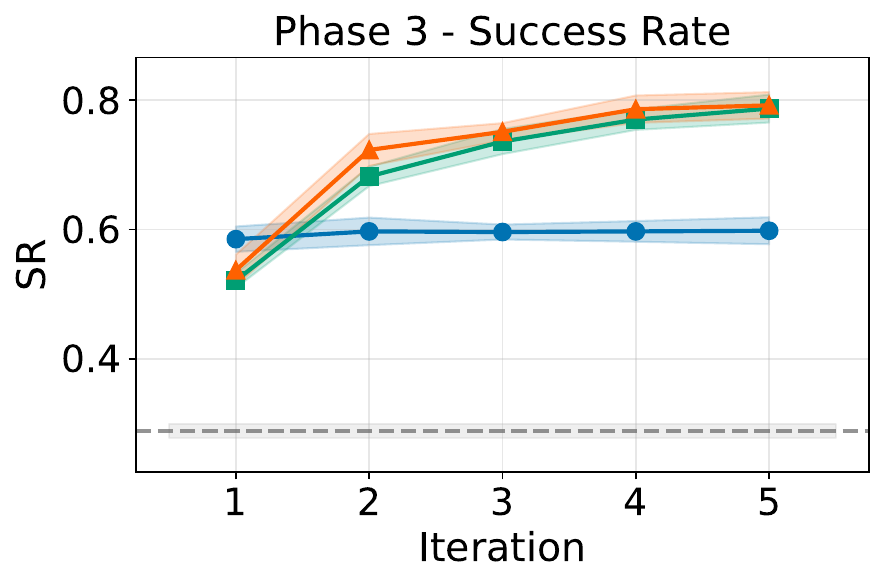}
        \label{fig:faiss_shopping_4_1_4_1_phase3_sr}
    \end{subfigure}    
    ~
    \begin{subfigure}[t]{0.315\linewidth}
        \includegraphics[width=\linewidth]{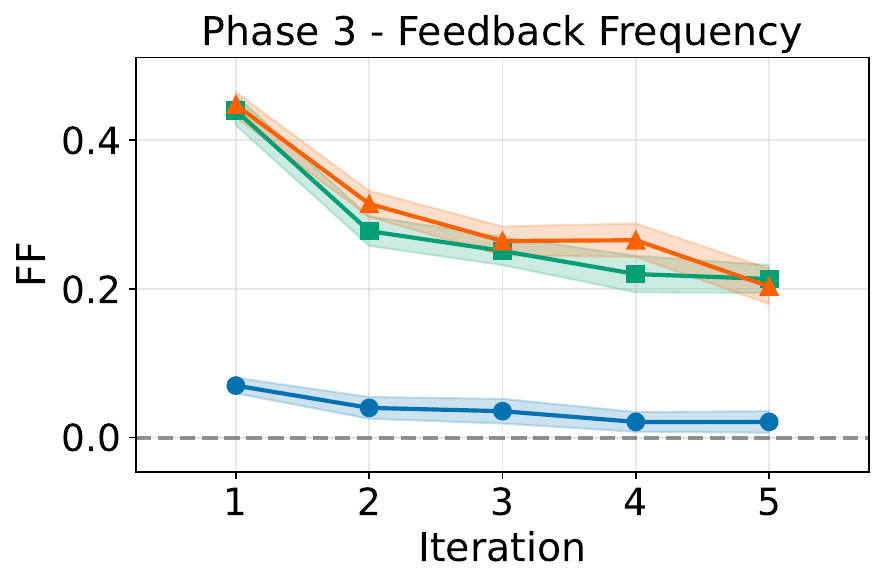}
        \label{fig:faiss_shopping_4_1_4_1_phase3_ff}
    \end{subfigure}
    ~
    \begin{subfigure}[t]{0.315\linewidth}
        \includegraphics[width=\linewidth]{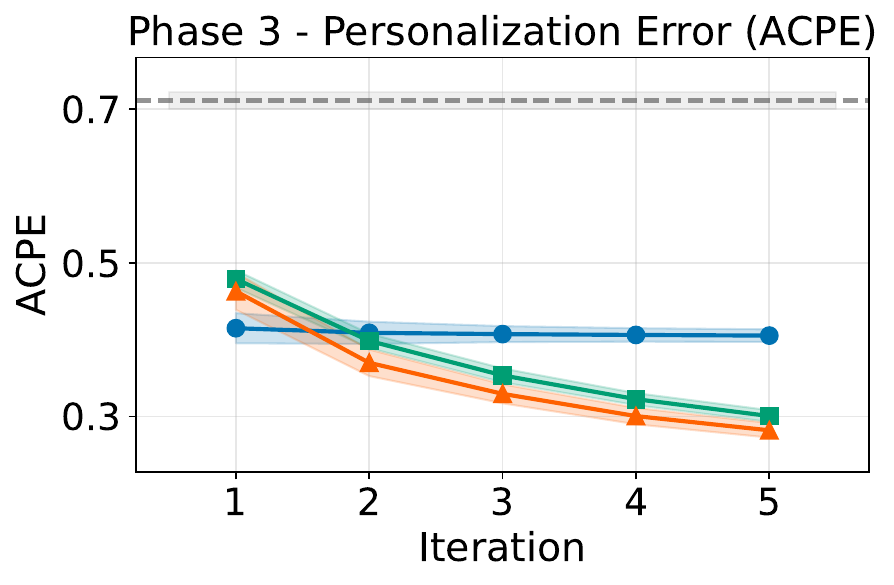}
        \label{fig:faiss_shopping_4_1_4_1_phase3_acpe}
    \end{subfigure}

    \par\vspace{-0.28cm}
    \begin{subfigure}[c]{\linewidth}
        \centering
        \includegraphics[width=0.65\linewidth]{figure/plots/legend.pdf}
    \end{subfigure}

    \caption{\textbf{Online shopping result (Agent: GPT-4o, Human Simulator: GPT-4o, Memory: FAISS)}. Top row: Phase 1 learning curves for success rate (left), feedback frequency (middle), and average cumulative personalization error (ACPE) (right). Bottom row: Phase~3 learning curves for the same metrics. Shaded regions denote standard error.}
    \label{fig:faiss_shopping_gpt_4_1_gpt_4_1}

\end{figure*}

\begin{figure*}[h!]
  \centering
    \begin{subfigure}[t]{0.315\linewidth}
        \includegraphics[width=\linewidth]{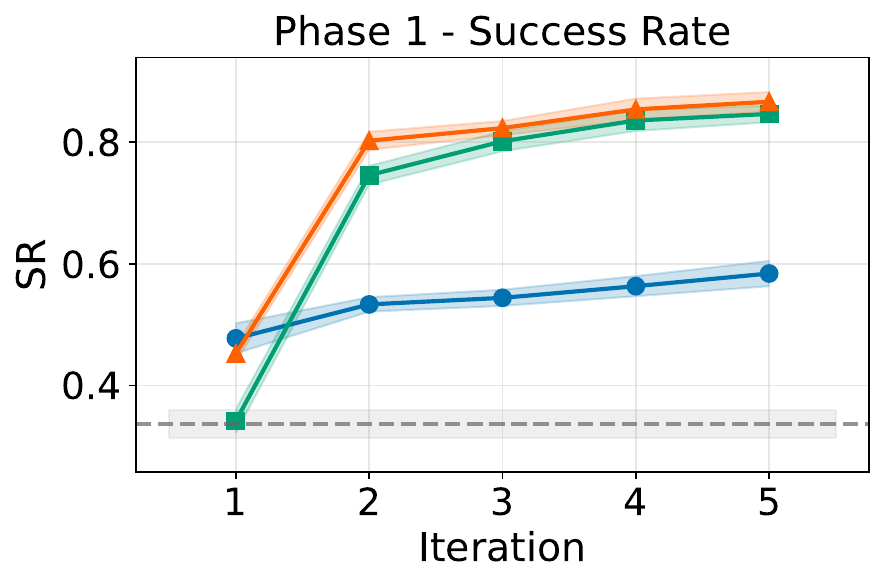}
        \label{fig:embodied_4_1_4_1_phase1_sr}
    \end{subfigure}    
    ~
    \begin{subfigure}[t]{0.315\linewidth}
        \includegraphics[width=\linewidth]{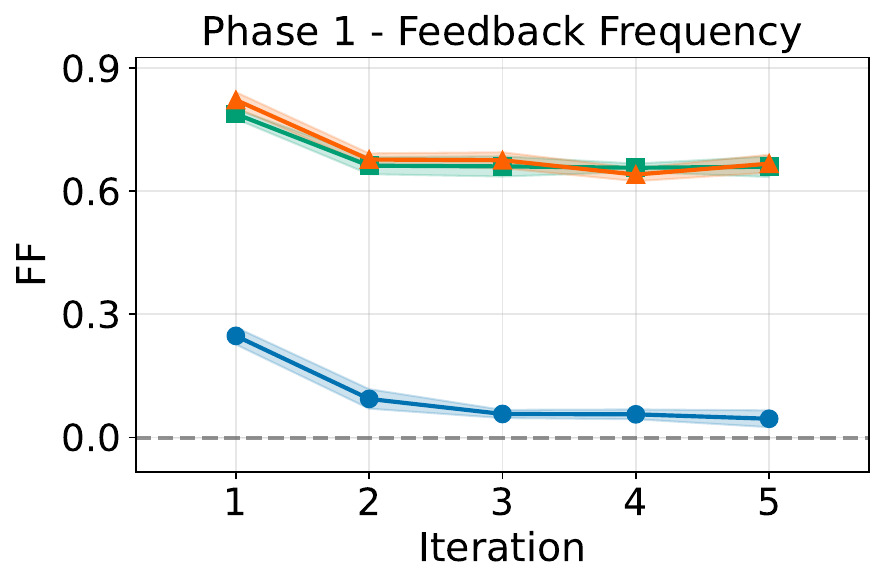}
        \label{fig:embodied_4_1_4_1_phase1_ff}
    \end{subfigure}
    ~
    \begin{subfigure}[t]{0.315\linewidth}
        \includegraphics[width=\linewidth]{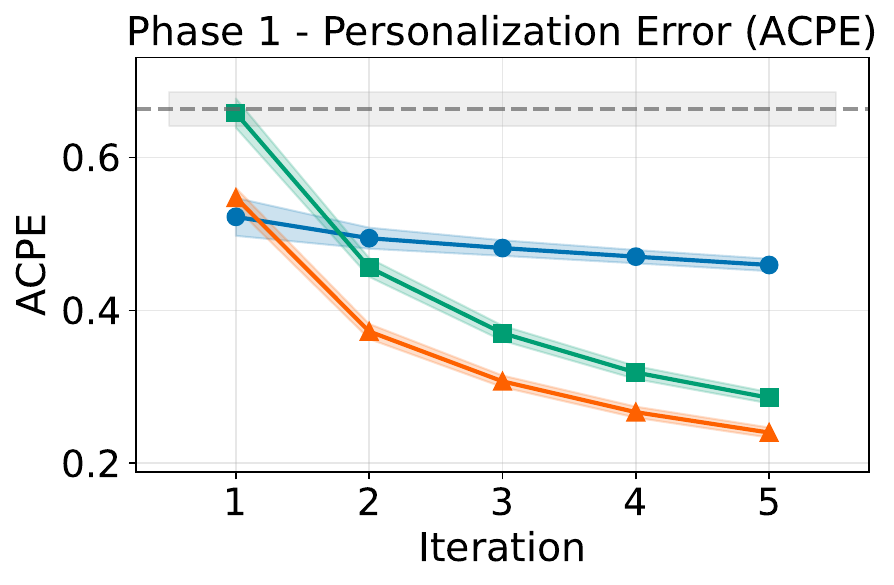}
        \label{fig:embodied_4_1_4_1_phase1_acpe}
    \end{subfigure}

    \vspace{-2mm}
    \begin{subfigure}[t]{0.315\linewidth}
        \includegraphics[width=\linewidth]{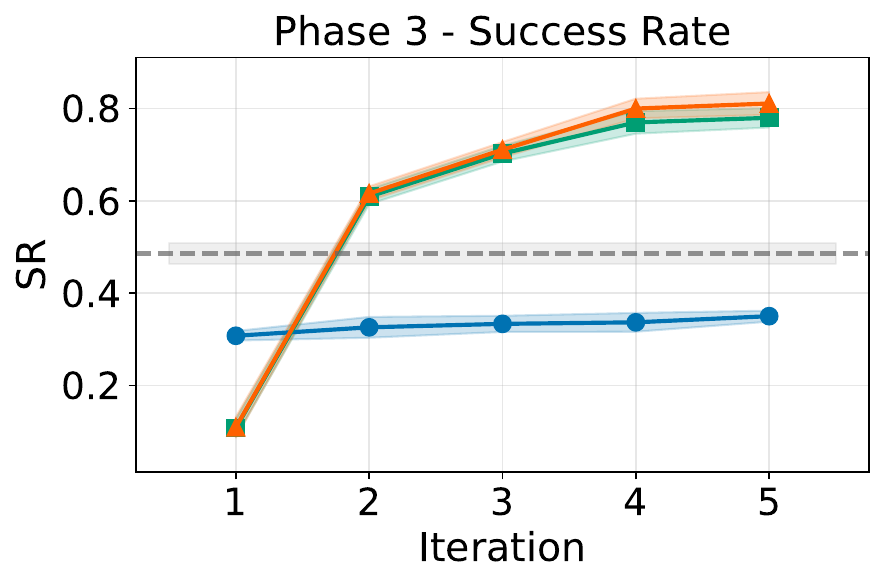}
        \label{fig:embodied_4_1_4_1_phase3_sr}
    \end{subfigure}    
    ~
    \begin{subfigure}[t]{0.315\linewidth}
        \includegraphics[width=\linewidth]{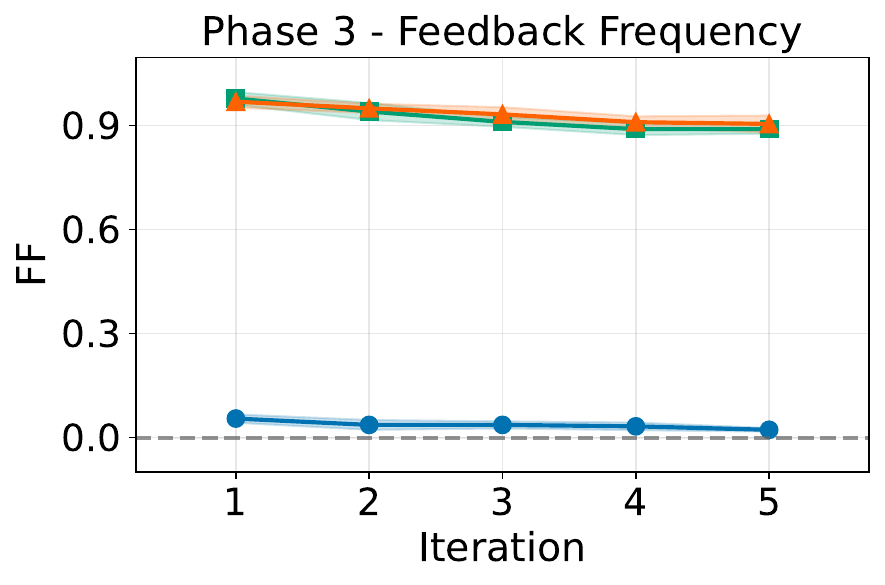}
        \label{fig:embodied_4_1_4_1_phase3_ff}
    \end{subfigure}
    ~
    \begin{subfigure}[t]{0.315\linewidth}
        \includegraphics[width=\linewidth]{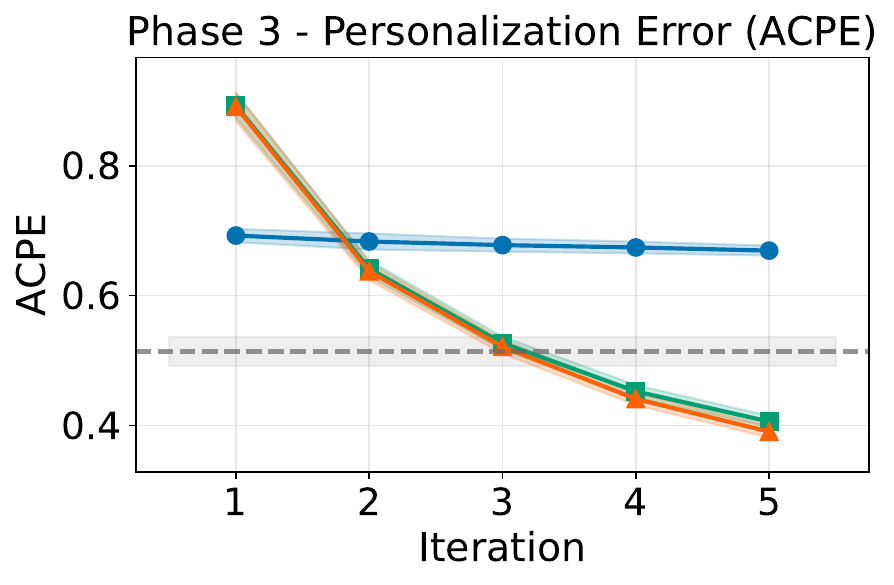}
        \label{fig:embodied_4_1_4_1_phase3_acpe}
    \end{subfigure}

    \par\vspace{-0.28cm}
    \begin{subfigure}[c]{\linewidth}
        \centering
        \includegraphics[width=0.65\linewidth]{figure/plots/legend.pdf}
    \end{subfigure}

    \caption{\textbf{Embodied manipulation results (Agent: GPT-4.1, Human Simulator: GPT-4o, Memory: SQLite)}. Top row: Phase 1 learning curves for success rate (left), feedback frequency (middle), and average cumulative personalization error (ACPE) (right). Bottom row: Phase~3 learning curves for the same metrics. Shaded regions denote standard error.}
    \label{fig:embodied_gpt_4_1_gpt_4_1}

\end{figure*}

\begin{figure*}[h!]
  \centering
    \begin{subfigure}[t]{0.315\linewidth}
        \includegraphics[width=\linewidth]{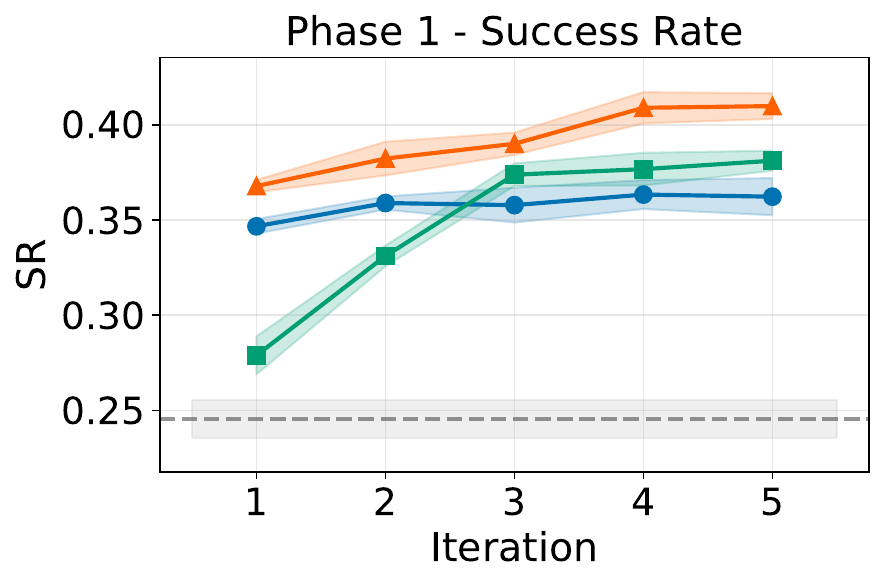}
        \label{fig:shopping_4o_4o_phase1_sr}
    \end{subfigure}    
    ~
    \begin{subfigure}[t]{0.315\linewidth}
        \includegraphics[width=\linewidth]{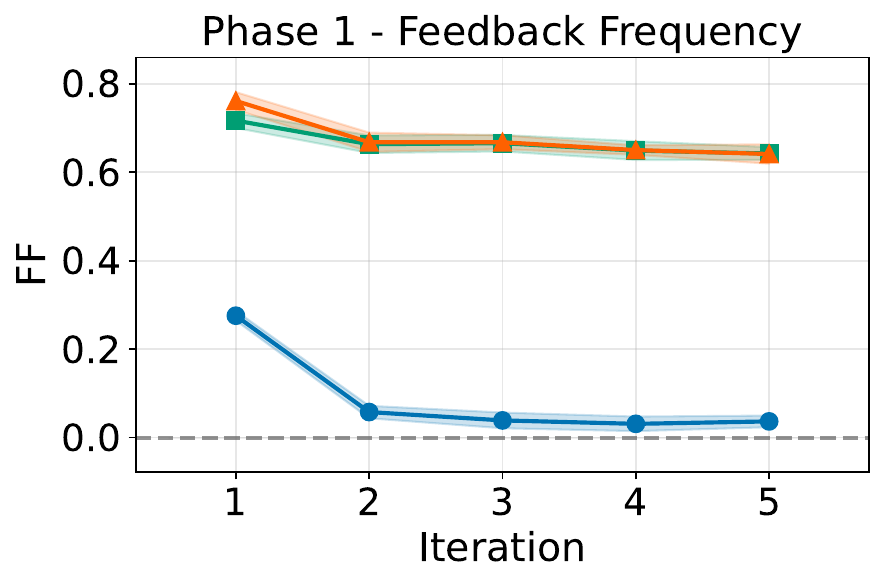}
        \label{fig:shopping_4o_4o_phase1_ff}
    \end{subfigure}
    ~
    \begin{subfigure}[t]{0.315\linewidth}
        \includegraphics[width=\linewidth]{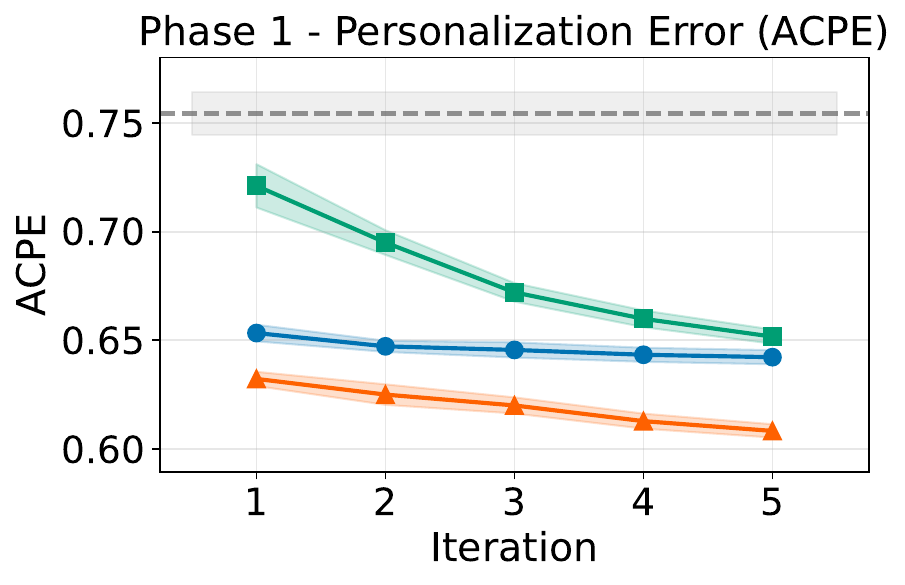}
        \label{fig:shopping_4o_4o_phase1_acpe}
    \end{subfigure}

    \vspace{-2mm}
    \begin{subfigure}[t]{0.315\linewidth}
        \includegraphics[width=\linewidth]{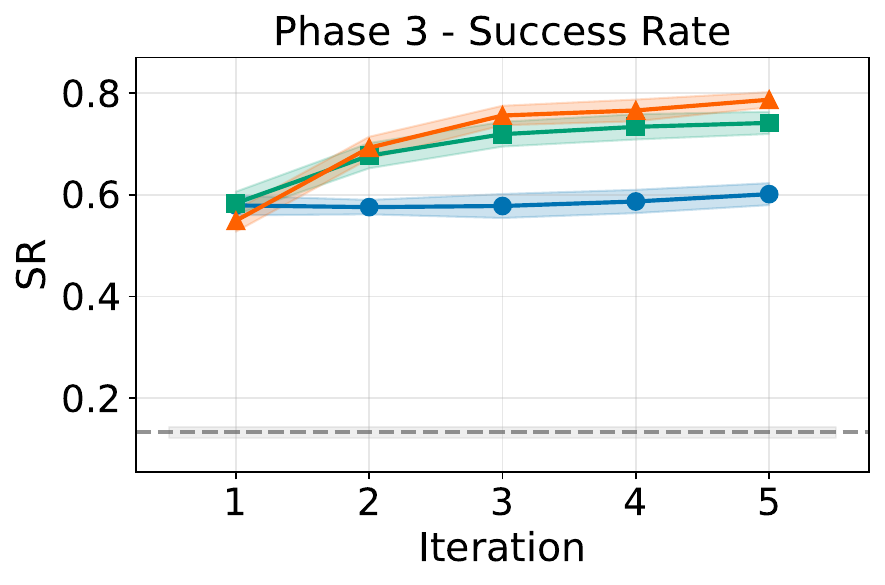}
        \label{fig:shopping_4o_4o_phase3_sr}
    \end{subfigure}    
    ~
    \begin{subfigure}[t]{0.315\linewidth}
        \includegraphics[width=\linewidth]{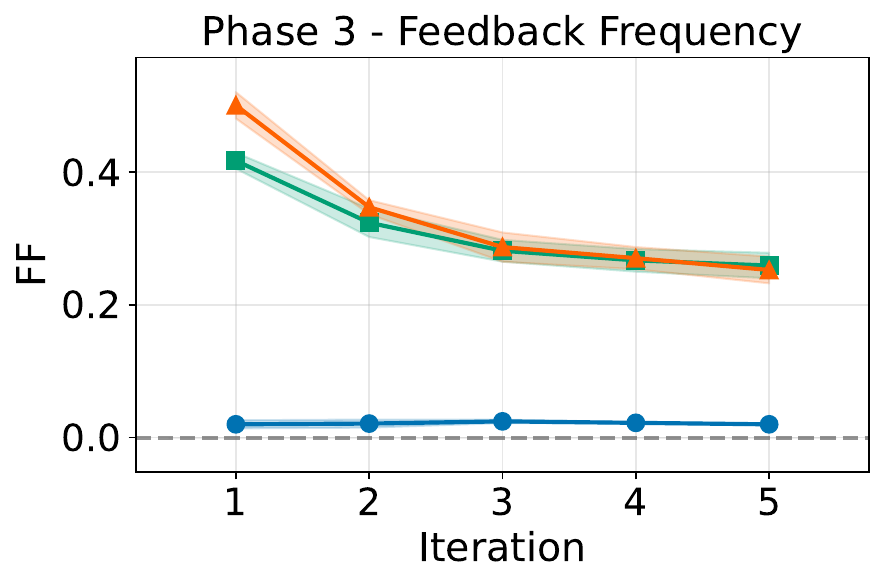}
        \label{fig:shoppping_4o_4o_phase3_ff}
    \end{subfigure}
    ~
    \begin{subfigure}[t]{0.315\linewidth}
        \includegraphics[width=\linewidth]{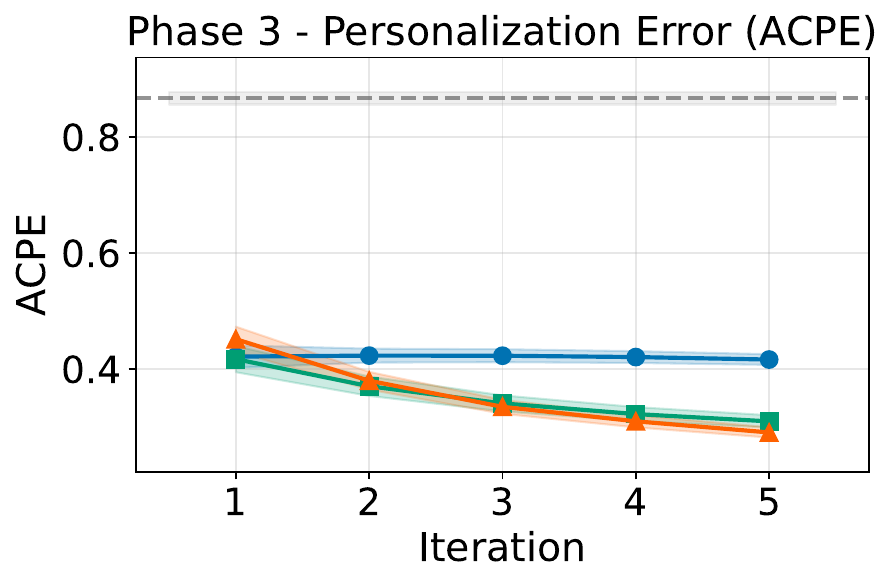}
        \label{fig:shopping_4o_4o_phase3_acpe}
    \end{subfigure}

    \par\vspace{-0.28cm}
    \begin{subfigure}[c]{\linewidth}
        \centering
        \includegraphics[width=0.65\linewidth]{figure/plots/legend.pdf}
    \end{subfigure}

    \caption{\textbf{Online shopping results (Agent: GPT-4.1, Human Simulator: GPT-4o, Memory: SQLite).} Top row: Phase 1 learning curves for success rate (left), feedback frequency (middle), and average cumulative personalization error (ACPE) (right). Bottom row: Phase~3 learning curves for the same metrics. Shaded regions denote standard error.}
    \label{fig:shopping_gpt_4_1_gpt_4_1}
\end{figure*}

\begin{figure*}[h!]
  \centering
    \begin{subfigure}[t]{0.315\linewidth}
        \includegraphics[width=\linewidth]{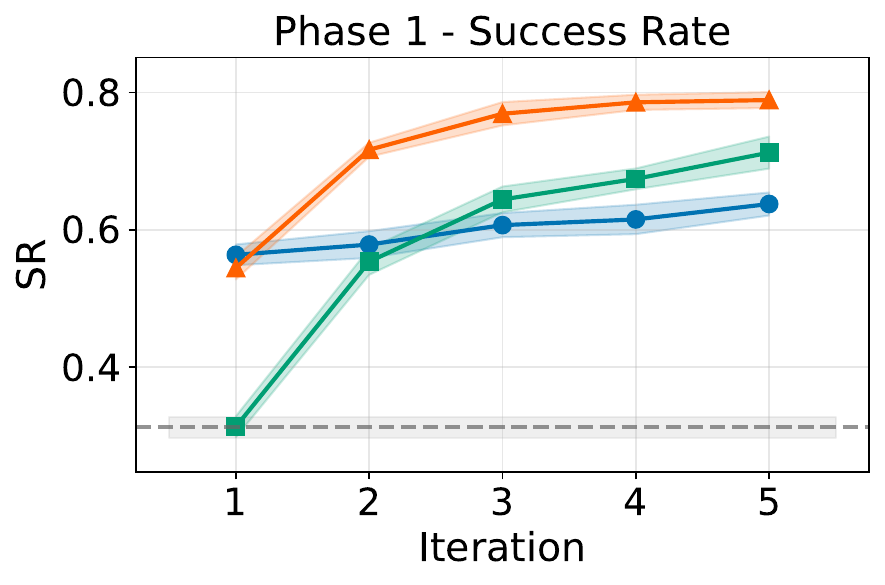}
        \label{fig:embodied_4o_4_1_phase1_sr}
    \end{subfigure}    
    ~
    \begin{subfigure}[t]{0.315\linewidth}
        \includegraphics[width=\linewidth]{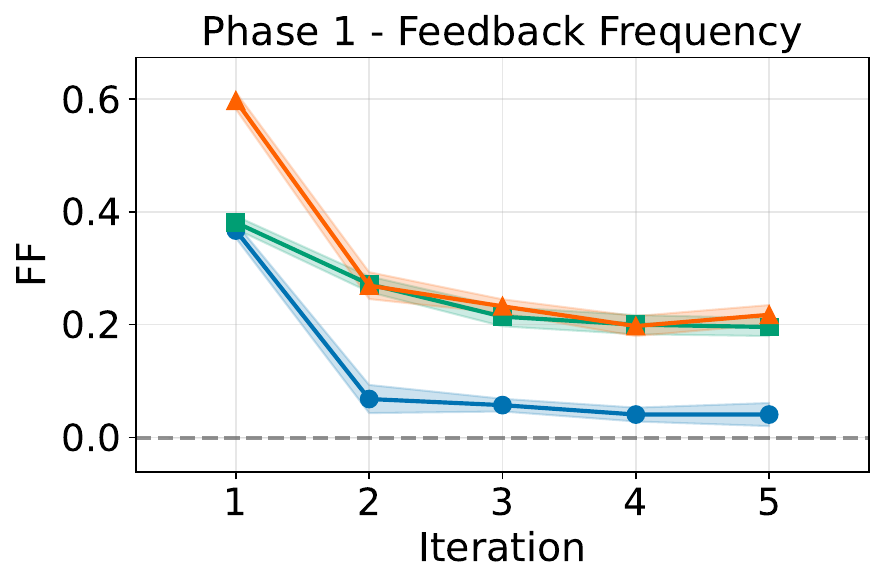}
        \label{fig:embodied_4o_4_1_phase1_ff}
    \end{subfigure}
    ~
    \begin{subfigure}[t]{0.315\linewidth}
        \includegraphics[width=\linewidth]{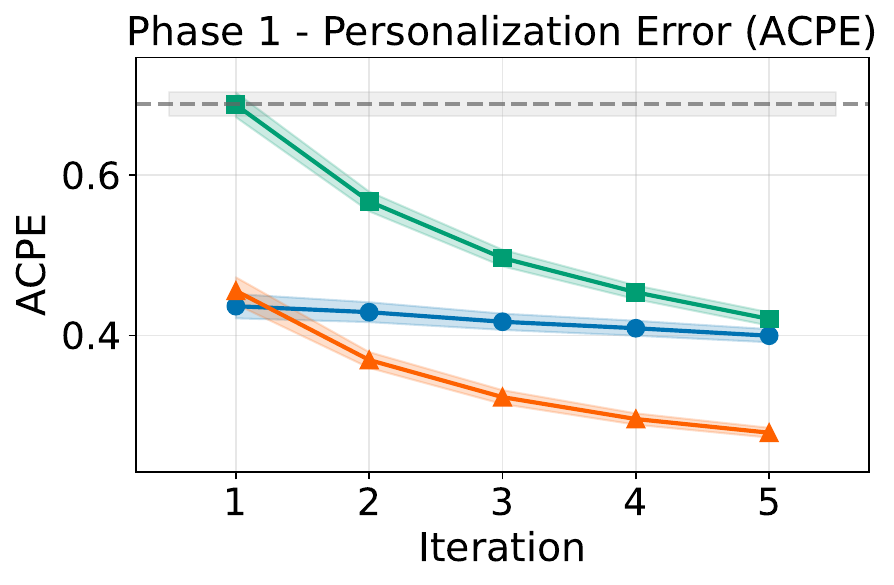}
        \label{fig:embodied_4o_4_1_phase1_acpe}
    \end{subfigure}

    \vspace{-2mm}
    \begin{subfigure}[t]{0.315\linewidth}
        \includegraphics[width=\linewidth]{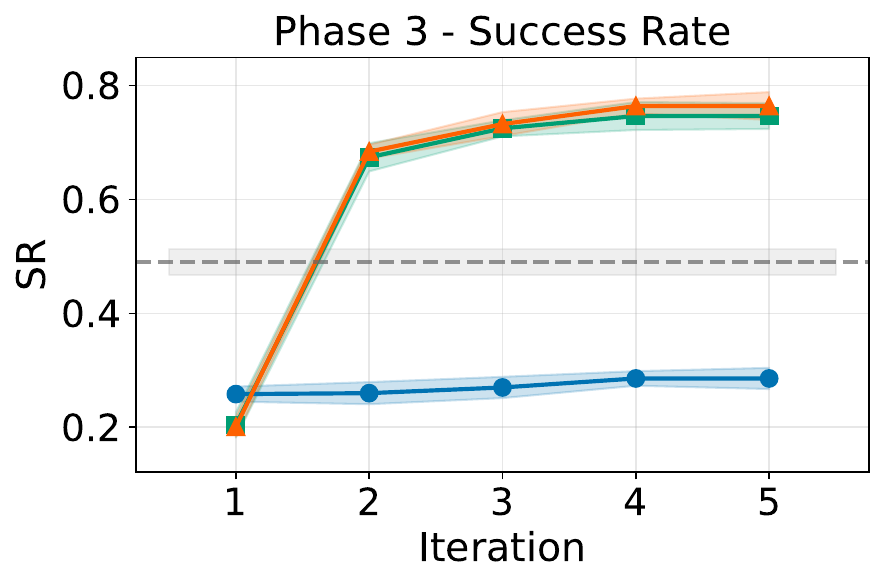}
        \label{fig:embodied_4o_4_1_phase3_sr}
    \end{subfigure}    
    ~
    \begin{subfigure}[t]{0.315\linewidth}
        \includegraphics[width=\linewidth]{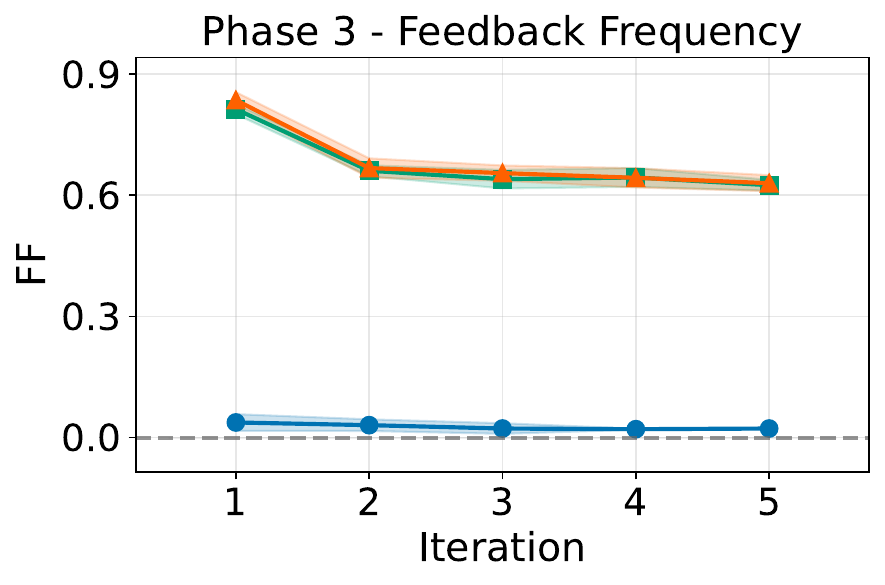}
        \label{fig:embodied_4o_4_1_phase3_ff}
    \end{subfigure}
    ~
    \begin{subfigure}[t]{0.315\linewidth}
        \includegraphics[width=\linewidth]{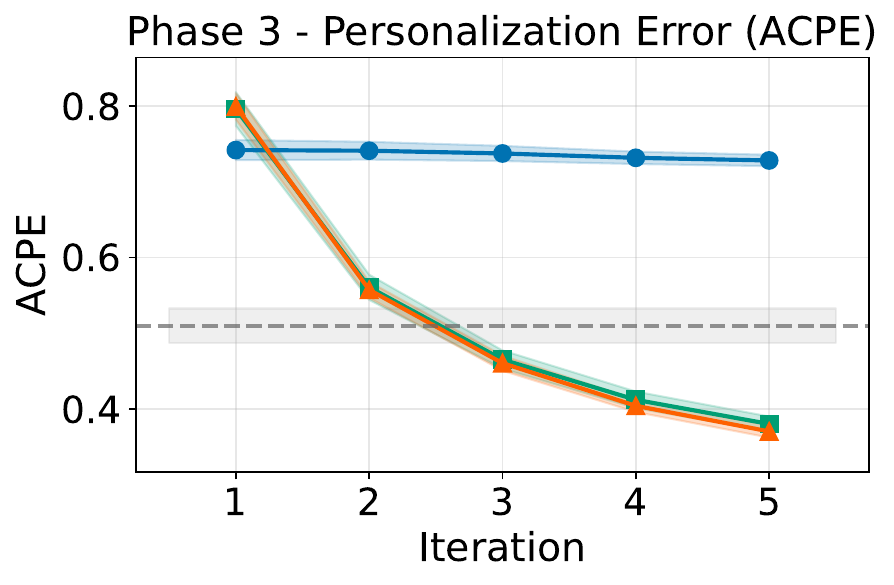}
        \label{fig:embodied_4o_4_1_phase3_acpe}
    \end{subfigure}

    \par\vspace{-0.28cm}
    \begin{subfigure}[c]{\linewidth}
        \centering
        \includegraphics[width=0.65\linewidth]{figure/plots/legend.pdf}
    \end{subfigure}

    \caption{\textbf{Embodied manipulation results (Agent: GPT-4o, Human Simulator: GPT-4.1, Memory: SQLite).} Top row: Phase 1 learning curves for success rate (left), feedback frequency (middle), and average cumulative personalization error (ACPE) (right). Bottom row: Phase~3 learning curves for the same metrics. Shaded regions denote standard error.}
    \label{fig:embodied_gpt_4o_gpt_4_1}

\end{figure*}

\begin{figure*}[h!]
  \centering
    \begin{subfigure}[t]{0.315\linewidth}
        \includegraphics[width=\linewidth]{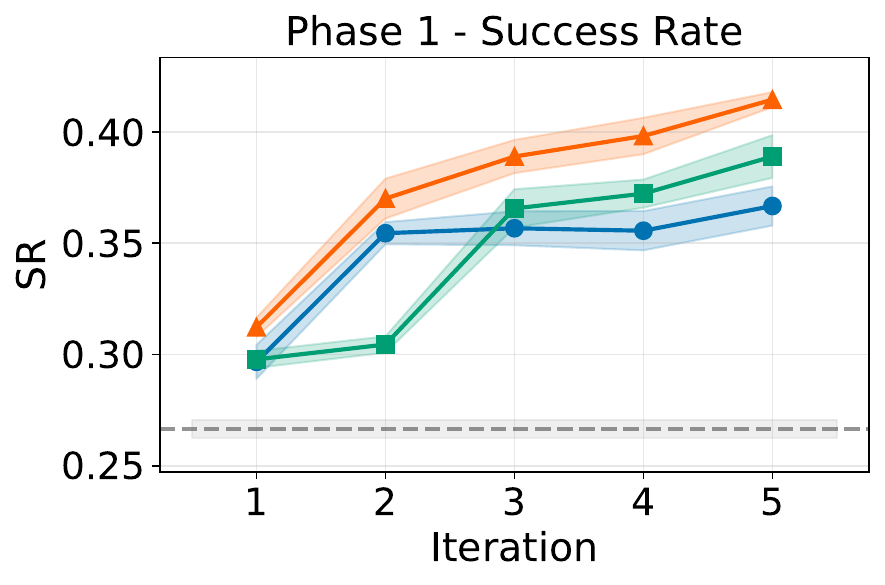}
        \label{fig:shopping_4o_4o_phase1_sr}
    \end{subfigure}    
    ~
    \begin{subfigure}[t]{0.315\linewidth}
        \includegraphics[width=\linewidth]{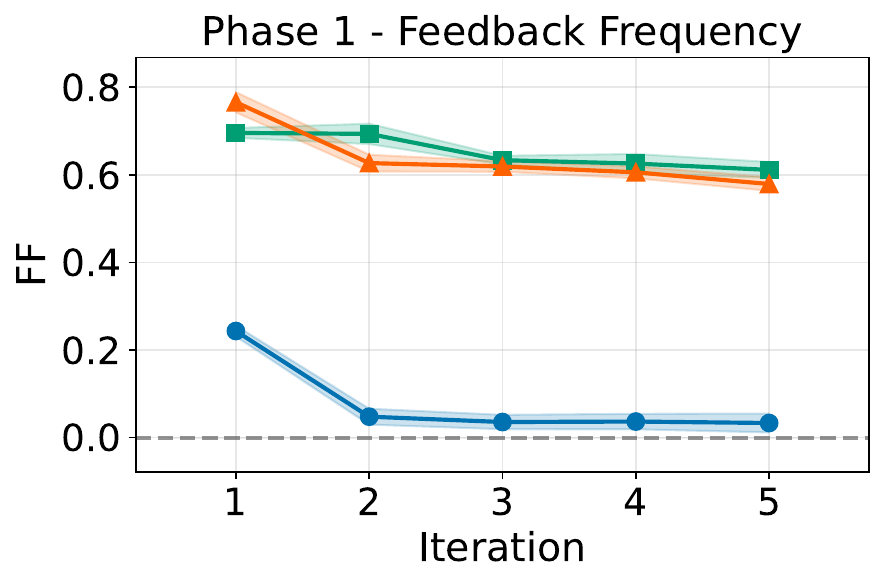}
        \label{fig:shopping_4o_4o_phase1_ff}
    \end{subfigure}
    ~
    \begin{subfigure}[t]{0.315\linewidth}
        \includegraphics[width=\linewidth]{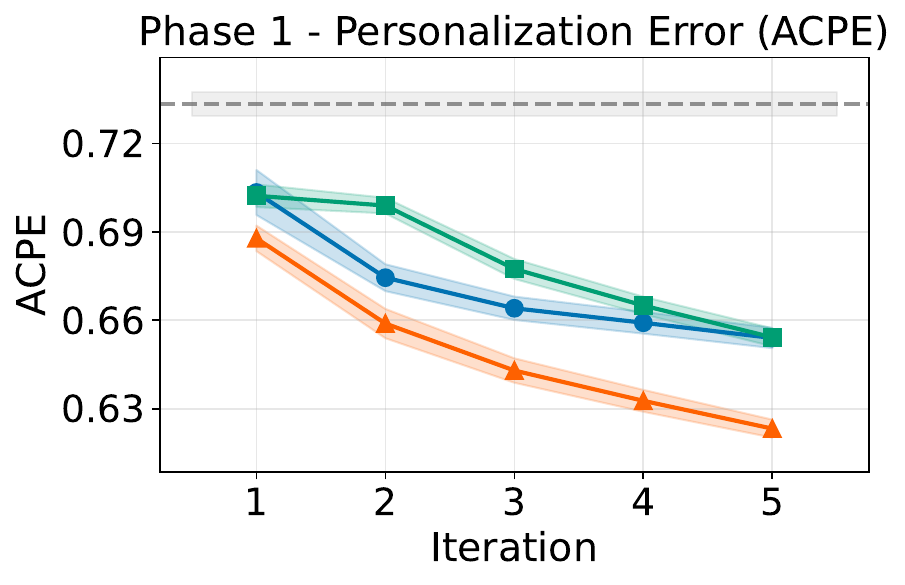}
        \label{fig:shopping_4o_4o_phase1_acpe}
    \end{subfigure}

    \vspace{-2mm}
    \begin{subfigure}[t]{0.315\linewidth}
        \includegraphics[width=\linewidth]{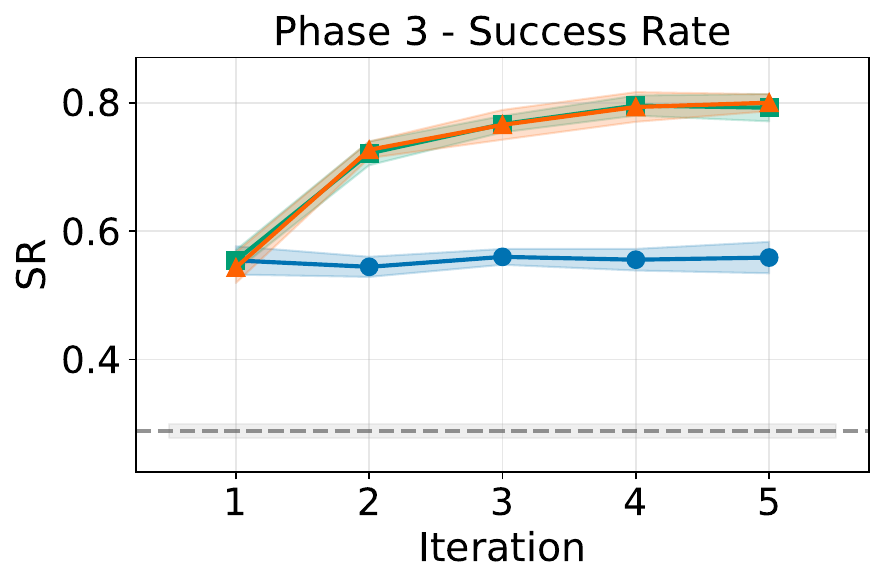}
        \label{fig:shopping_4o_4o_phase3_sr}
    \end{subfigure}    
    ~
    \begin{subfigure}[t]{0.315\linewidth}
        \includegraphics[width=\linewidth]{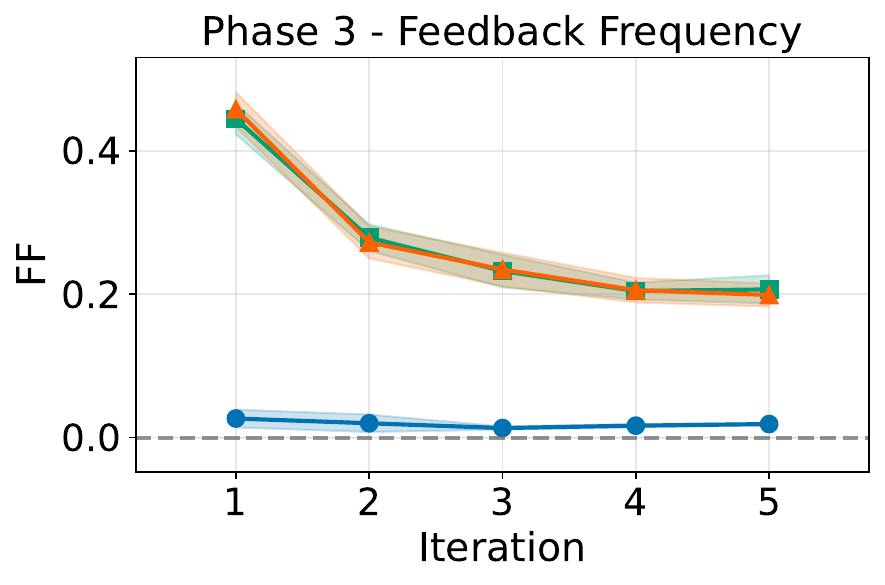}
        \label{fig:shoppping_4o_4o_phase3_ff}
    \end{subfigure}
    ~
    \begin{subfigure}[t]{0.315\linewidth}
        \includegraphics[width=\linewidth]{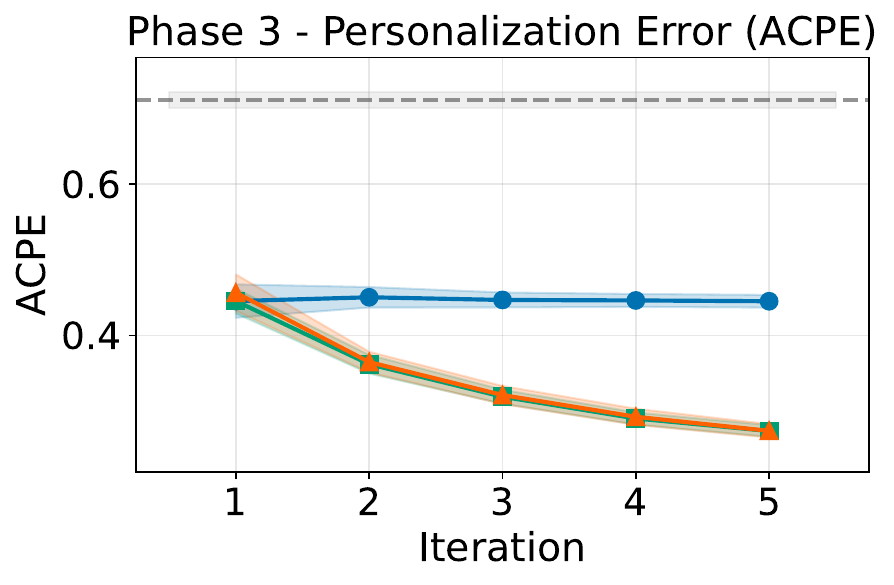}
        \label{fig:shopping_4o_4o_phase3_acpe}
    \end{subfigure}

    \par\vspace{-0.28cm}
    \begin{subfigure}[c]{\linewidth}
        \centering
        \includegraphics[width=0.65\linewidth]{figure/plots/legend.pdf}
    \end{subfigure}

    \caption{\textbf{Online shopping results (Agent: GPT-4o, Human Simulator: GPT-4.1, Memory: SQLite).} Top row: Phase 1 learning curves for success rate (left), feedback frequency (middle), and average cumulative personalization error (ACPE) (right). Bottom row: Phase~3 learning curves for the same metrics. Shaded regions denote standard error.}
    \label{fig:shopping_gpt_4o_gpt_4_1}
\end{figure*}

\begin{figure*}[h!]
  \centering
    \begin{subfigure}[t]{0.315\linewidth}
        \includegraphics[width=\linewidth]{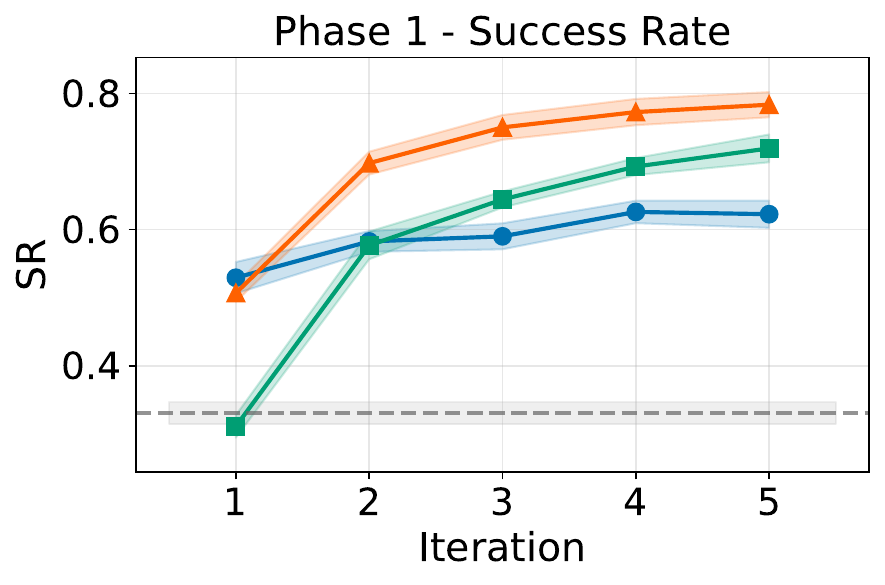}
        \label{fig:embodied_4o_llama4_phase1_sr}
    \end{subfigure}    
    ~
    \begin{subfigure}[t]{0.315\linewidth}
        \includegraphics[width=\linewidth]{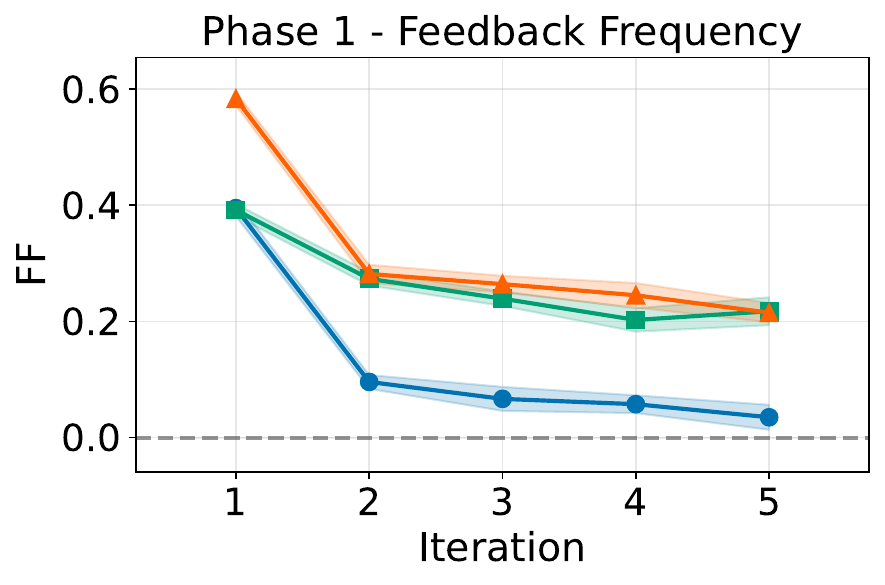}
        \label{fig:embodied_4o_llama4_phase1_ff}
    \end{subfigure}
    ~
    \begin{subfigure}[t]{0.315\linewidth}
        \includegraphics[width=\linewidth]{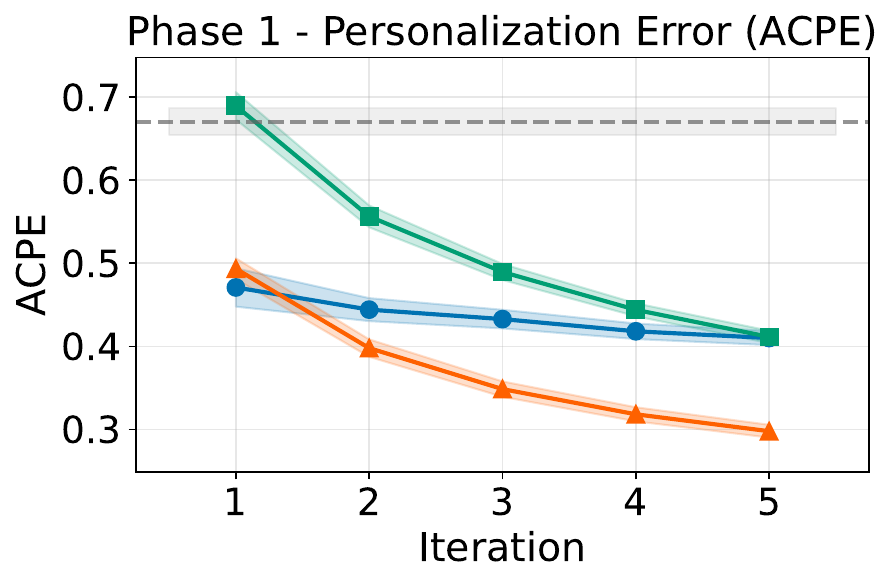}
        \label{fig:embodied_4o_llama4_phase1_acpe}
    \end{subfigure}

    \vspace{-2mm}
    \begin{subfigure}[t]{0.315\linewidth}
        \includegraphics[width=\linewidth]{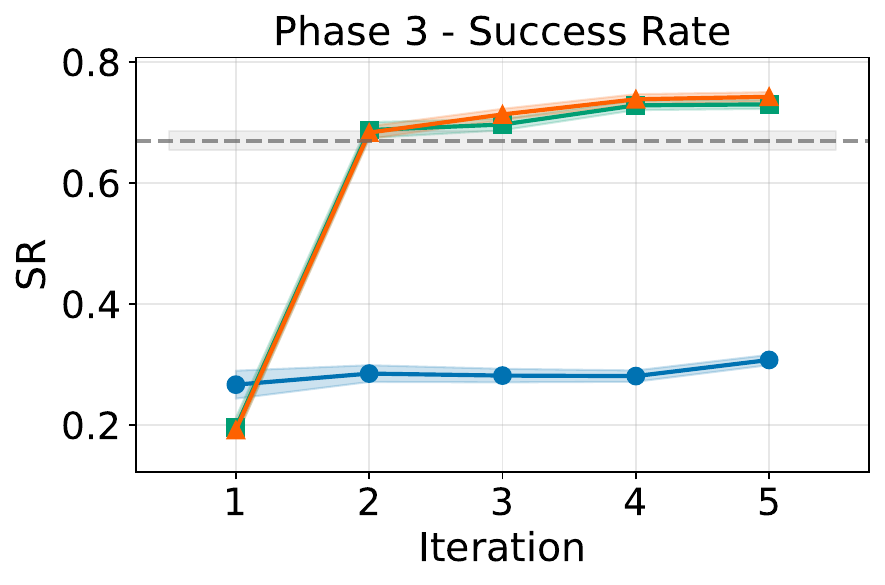}
        \label{fig:embodied_4o_llama4_phase3_sr}
    \end{subfigure}    
    ~
    \begin{subfigure}[t]{0.315\linewidth}
        \includegraphics[width=\linewidth]{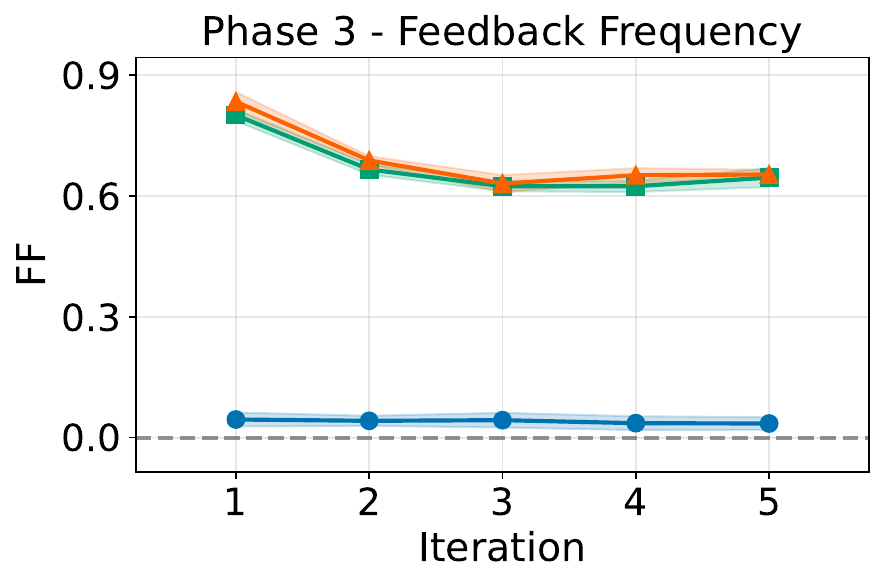}
        \label{fig:embodied_4o_llama4_phase3_ff}
    \end{subfigure}
    ~
    \begin{subfigure}[t]{0.315\linewidth}
        \includegraphics[width=\linewidth]{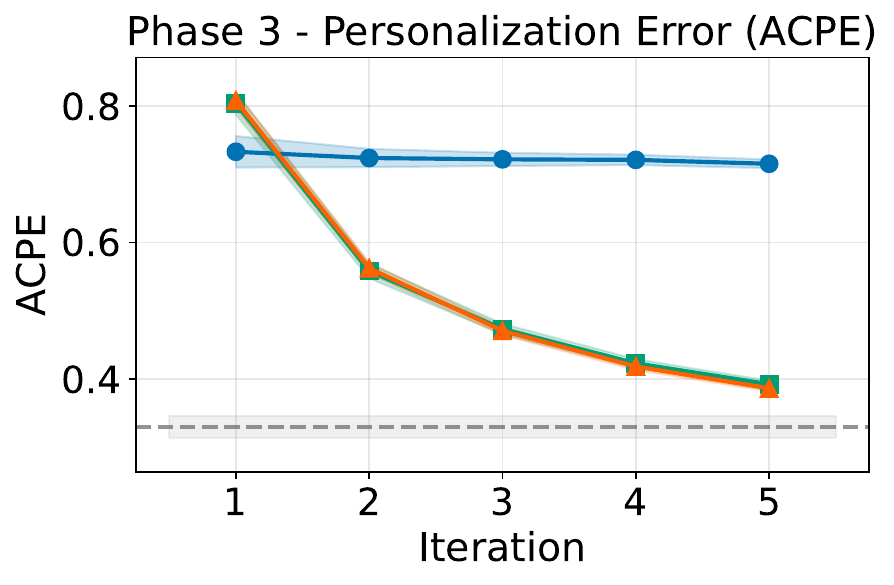}
        \label{fig:embodied_4o_llama4_phase3_acpe}
    \end{subfigure}

    \par\vspace{-0.28cm}
    \begin{subfigure}[c]{\linewidth}
        \centering
        \includegraphics[width=0.65\linewidth]{figure/plots/legend.pdf}
    \end{subfigure}

    \caption{\textbf{Embodied manipulation results (Agent: GPT-4o, Human Simulator: Llama-4-Scout, Memory: SQLite).} Top row: Phase 1 learning curves for success rate (left), feedback frequency (middle), and average cumulative personalization error (ACPE) (right). Bottom row: Phase~3 learning curves for the same metrics. Shading denote standard error.}
    \label{fig:embodied_gpt_4o_llama4}

\end{figure*}

\begin{figure*}[h!]
  \centering
    \begin{subfigure}[t]{0.315\linewidth}
        \includegraphics[width=\linewidth]{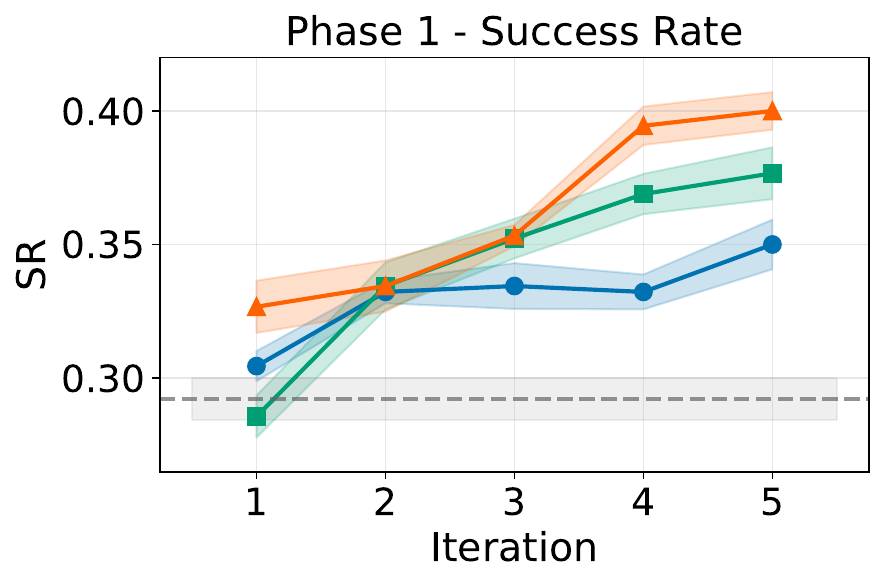}
        \label{fig:shopping_4o_llama4_phase1_sr}
    \end{subfigure}    
    ~
    \begin{subfigure}[t]{0.315\linewidth}
        \includegraphics[width=\linewidth]{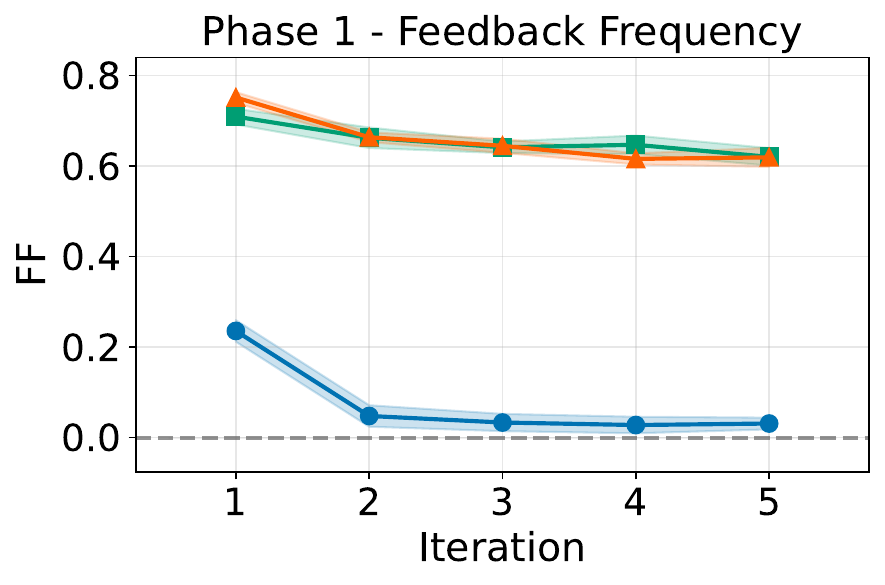}
        \label{fig:shopping_4o_llama4_phase1_ff}
    \end{subfigure}
    ~
    \begin{subfigure}[t]{0.315\linewidth}
        \includegraphics[width=\linewidth]{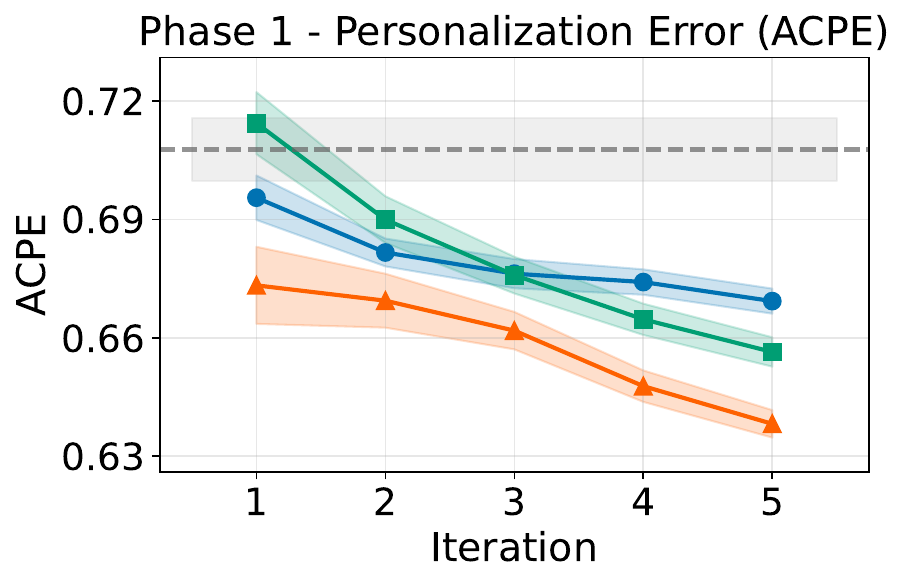}
        \label{fig:shopping_4o_llama4_phase1_acpe}
    \end{subfigure}

    \vspace{-2mm}
    \begin{subfigure}[t]{0.315\linewidth}
        \includegraphics[width=\linewidth]{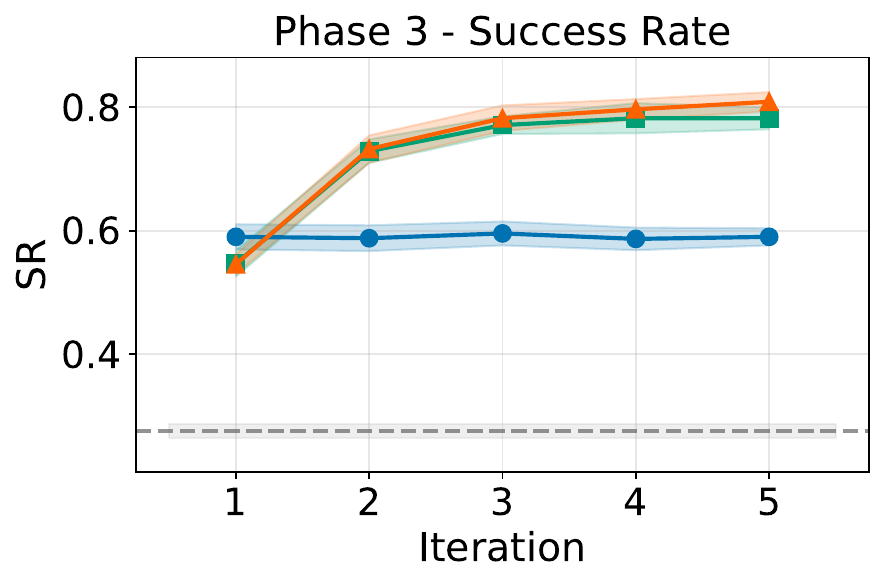}
        \label{fig:shopping_4o_llama4_phase3_sr}
    \end{subfigure}    
    ~
    \begin{subfigure}[t]{0.315\linewidth}
        \includegraphics[width=\linewidth]{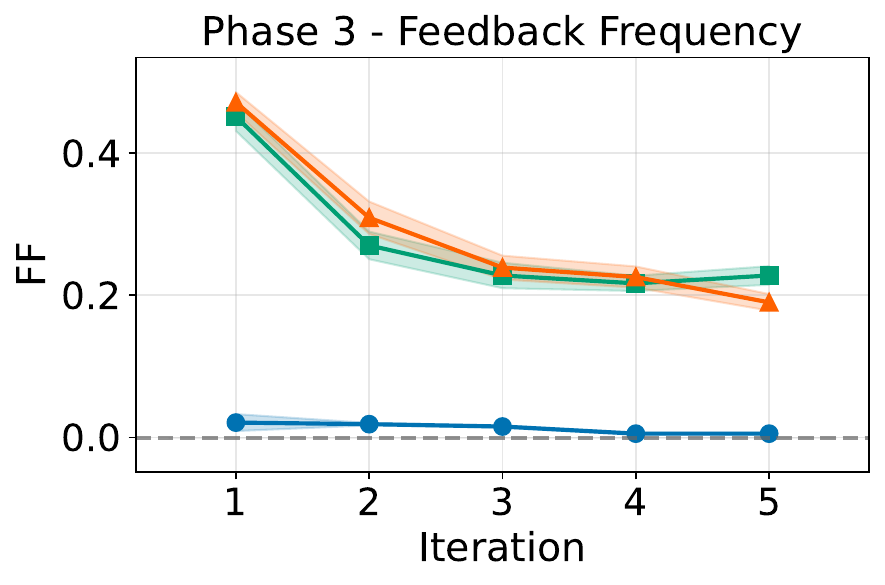}
        \label{fig:shoppping_4o_llama4_phase3_ff}
    \end{subfigure}
    ~
    \begin{subfigure}[t]{0.315\linewidth}
        \includegraphics[width=\linewidth]{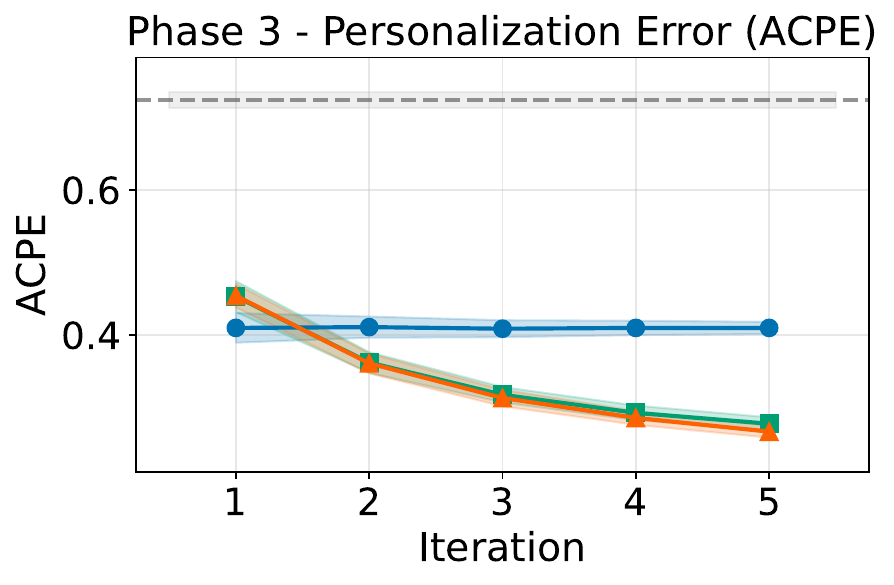}
        \label{fig:shopping_4o_llama4_phase3_acpe}
    \end{subfigure}

    \par\vspace{-0.28cm}
    \begin{subfigure}[c]{\linewidth}
        \centering
        \includegraphics[width=0.65\linewidth]{figure/plots/legend.pdf}
    \end{subfigure}

    \caption{\textbf{Online shopping results (Agent: GPT-4.1, Human Simulator: Llama-4-Scout, Memory: SQLite).} Top row: Phase 1 learning curves for success rate (left), feedback frequency (middle), and average cumulative personalization error (ACPE) (right). Bottom row: Phase~3 learning curves for the same metrics. Shaded regions denote standard error.}
    \label{fig:shopping_gpt_4o_llama4}
\end{figure*}

\end{document}